\documentclass{article}

\usepackage{xr-hyper} 
\usepackage{microtype}
\usepackage{graphicx}
\usepackage{amsmath}
\DeclareMathOperator*{\argmax}{arg\,max}

\usepackage{amssymb}
\usepackage{amsthm}
\usepackage{subcaption}
\usepackage{booktabs} 
\usepackage{makecell}
\usepackage{tikz}
\usepackage{thmtools}
\usepackage{thm-restate}
\usepackage{pgfplots}
\usepackage{caption}
\DeclareCaptionLabelFormat{alglabel}{\bfseries Pseudocode}
\definecolor{blue}{RGB}{36,105,255} 
\definecolor{green}{RGB}{41,255,27}
\definecolor{red}{RGB}{255,27,27}  
\definecolor{yellow}{RGB}{255,242,0}

\theoremstyle{plain}
\newtheorem{thm}{Theorem}
\newtheorem{lemma}[thm]{Lemma}
\newtheorem{proposition}{Proposition}
\newtheorem*{claim}{Claim}

\theoremstyle{definition}

\usepackage{hyperref}

\usepackage[accepted]{icml2019}

\icmltitlerunning{Certified Adversarial Robustness via Randomized Smoothing}

\begin{document}

\twocolumn[
\icmltitle{Certified Adversarial Robustness via Randomized Smoothing}



\icmlsetsymbol{equal}{*}

\begin{icmlauthorlist}
\icmlauthor{Jeremy Cohen}{cmu}
\icmlauthor{Elan Rosenfeld}{cmu}
\icmlauthor{J. Zico Kolter}{cmu,bosch}
\end{icmlauthorlist}

\icmlaffiliation{cmu}{Carnegie Mellon University}
\icmlaffiliation{bosch}{Bosch Center for AI}

\icmlcorrespondingauthor{Jeremy Cohen}{jeremycohen@cmu.edu}

\icmlkeywords{Machine Learning, ICML}

\vskip 0.3in
]



\printAffiliationsAndNotice{}  

\begin{abstract}
We show how to turn any classifier that classifies well under Gaussian noise into a new classifier that is certifiably robust to adversarial perturbations under the $\ell_2$ norm.
This ``randomized smoothing''  technique has been proposed recently in the literature, but existing guarantees are loose.
We prove a tight robustness guarantee in $\ell_2$ norm for smoothing with Gaussian noise.
We use randomized smoothing to obtain an ImageNet classifier with e.g. a certified top-1 accuracy of 49\% under adversarial perturbations with $\ell_2$ norm less than 0.5 (=127/255).
No certified defense has been shown feasible on ImageNet except for smoothing.
On smaller-scale datasets where competing approaches to certified $\ell_2$ robustness are viable, smoothing delivers higher certified accuracies.
Our strong empirical results suggest that randomized smoothing is a promising direction for future research into adversarially robust classification.  Code and models are available at \url{http://github.com/locuslab/smoothing}.

\end{abstract}

\section{Introduction}

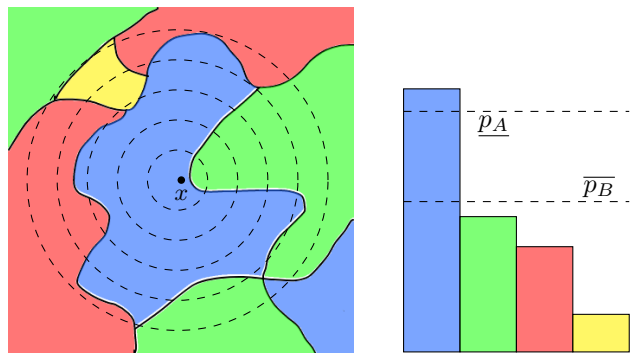
\begin{figure}[t]
\begin{center}
\begin{tikzpicture}
\node[inner sep=0pt] {\includegraphics[width=130px]{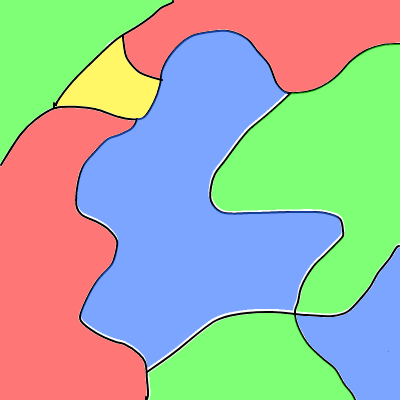}};
\draw[dashed] (-0.05, 0) circle (0.4);
\draw[dashed] (-0.05, 0) circle (0.8);
\draw[dashed] (-0.05, 0) circle (1.2);
\draw[dashed] (-0.05, 0) circle (1.6);
\draw[dashed] (-0.05, 0) circle (2.0);
\fill[black] (0, 0) circle (.05);
\node at (0, -0.2) {\small $x$};
\end{tikzpicture}
\hfill
\begin{tikzpicture}
\filldraw[fill={blue}, fill opacity = 0.6] (0,0) rectangle (0.75,3.5);
\filldraw[fill={green}, fill opacity = 0.6] (0.75,0) rectangle (1.5, 1.8);
\filldraw[fill={red}, fill opacity = 0.6] (1.5,0) rectangle (2.25, 1.4);
\filldraw[fill={yellow}, fill opacity = 0.6] (2.25, 0) rectangle (3.0, 0.5);
\draw [dashed] (0, 3.2) -- (3, 3.2);
\draw [dashed] (0, 2.0) -- (3, 2.0);
\node at (1.2, 3.0) {$\underline{p_A}$};
\node at (2.6, 2.2) {$\overline{p_B}$};
\end{tikzpicture}
\caption{Evaluating the smoothed classifier at an input $x$.  \textbf{Left}: the decision regions of the base classifier $f$ are drawn in different colors.  The dotted lines are the level sets of the distribution $\mathcal{N}(x, \sigma^2 I)$.  \textbf{Right}: the distribution $f(\mathcal{N}(x, \sigma^2 I))$.  
As discussed below, $\underline{p_A}$ is a lower bound on the probability of the top class and $\overline{p_B}$ is an upper bound on the probability of each other class. 
Here, $g(x)$ is ``blue.''}.
\label{page1illustration}
\end{center}
\vspace{-25px}
\end{figure}

Modern image classifiers achieve high accuracy on i.i.d. test sets but are not robust to small, adversarially-chosen perturbations of their inputs \citep{szegedy2014intriguing, biggio2013evasion}.
Given an image $x$ correctly classified by, say, a neural network, an adversary can usually engineer an adversarial perturbation $\delta$ so small that $x+\delta$ looks just like $x$ to the human eye, yet the network classifies $x + \delta$ as a different, incorrect class.
Many works have proposed heuristic methods for training classifiers intended to be robust to adversarial perturbations.  However, most of these heuristics have been subsequently shown to fail against suitably powerful adversaries  \citep{carlini2017adversarial, athalye2018obfuscated, uesato2018adversarial}.
In response, a line of work on \emph{certifiable robustness} studies classifiers whose prediction at any point $x$ is verifiably constant within some set around $x$ \citep[e.g.]{wong2018provable, raghunathan2018certified}.
In most of these works, the robust classifier takes the form of a neural network.
Unfortunately, all existing approaches for certifying the robustness of neural networks have trouble scaling to networks that are large and expressive enough to solve problems like ImageNet.

One workaround is to look for robust classifiers that are not neural networks.
Recently, two papers \citep{lecuyer2018certified, li2018second} showed that an operation we call \emph{randomized smoothing}\footnote{Smoothing was proposed under the name ``PixelDP'' (for differential privacy).  We use a different name since our improved analysis does not involve differential privacy.} can transform any arbitrary base classifier $f$ into a new ``smoothed classifier'' $g$ that is certifiably robust in $\ell_2$ norm.
Let $f$ be an arbitrary classifier which maps inputs $\mathbb{R}^d$ to classes $\mathcal{Y}$.
For any input $x$, the smoothed classifier's prediction $g(x)$ is defined to be the class which $f$ is most likely to classify the random variable $\mathcal{N}(x, \sigma^2 I)$ as.
That is, $g(x)$ returns the most probable prediction by $f$ of random Gaussian corruptions of $x$.

If the base classifier $f$ is most likely to classify $\mathcal{N}(x, \sigma^2 I)$ as $x$'s correct class, then the smoothed classifier $g$ will be correct at $x$.  But the smoothed classifier $g$ will also possess a desirable property that the base classifier may lack: one can verify that $g$'s prediction is constant within an $\ell_2$ ball around any input $x$, simply by estimating the probabilities with which $f$ classifies $\mathcal{N}(x, \sigma^2 I)$ as each class.
The higher the probability with which $f$ classifies $\mathcal{N}(x, \sigma^2 I)$ as the most probable class, the larger the $\ell_2$ radius around $x$ in which $g$ provably returns that class.

\citet{lecuyer2018certified} proposed randomized smoothing as a provable adversarial defense, and used it to train the first certifiably robust classifier for ImageNet.
Subsequently, \citet{li2018second} proved a stronger robustness guarantee.
However, both of these guarantees are loose, in the sense that the smoothed classifier $g$ is \emph{provably always} more robust than the guarantee indicates.
In this paper, we prove the first tight robustness guarantee for randomized smoothing.
Our analysis reveals that smoothing with Gaussian noise naturally induces certifiable robustness under the $\ell_2$ norm.
We suspect that other, as-yet-unknown noise distributions might induce robustness to other perturbation sets such as general $\ell_p$ norm balls.

\begin{table}[b]
\begin{center}
\caption{Approximate certified accuracy on ImageNet.  Each row shows a radius $r$, the best hyperparameter $\sigma$ for that radius, the approximate certified accuracy at radius $r$ of the corresponding smoothed classifier, and the standard accuracy of the corresponding smoothed classifier.  To give a sense of scale, a perturbation with $\ell_2$ radius 1.0 could change one pixel by 255, ten pixels by 80, 100 pixels by 25, or 1000 pixels by 8.  Random guessing on ImageNet would attain 0.1\% accuracy.} 
\label{table:imagenet-certified-accuracy}
\vspace{0.1in}
\begin{small}
\begin{sc}
\begin{tabular}{l l c c}
\toprule
\thead{$\ell_2$ radius} & \thead{best $\sigma$} & \thead{Cert. Acc (\%)} & \thead{Std. Acc(\%)} \\                                          
\midrule
0.5 & 0.25 & 49 & 67 \\
1.0 &  0.50 & 37 & 57 \\
2.0 & 0.50 & 19 & 57  \\
3.0 & 1.00 & 12 & 44\\
\bottomrule
\end{tabular}
\end{sc}
\end{small}
\end{center}
\end{table}

Randomized smoothing has one major drawback.
If $f$ is a neural network, it is not possible to \emph{exactly} compute the probabilities with which $f$ classifies $\mathcal{N}(x, \sigma^2 I)$ as each class.
Therefore, it is not possible to exactly evaluate $g$'s prediction at any input $x$, or to exactly compute the radius in which this prediction is certifiably robust.
Instead, we present Monte Carlo algorithms for both tasks that are guaranteed to succeed with arbitrarily high probability.

Despite this drawback, randomized smoothing enjoys several compelling advantages over other certifiably robust classifiers proposed in the literature: it makes no assumptions about the base classifier's architecture, it is simple to implement and understand, and, most importantly, it permits the use of arbitrarily large neural networks as the base classifier.  In contrast, other certified defenses do not currently scale to large networks.
Indeed, smoothing is the only certified adversarial defense which has been shown feasible on the full-resolution ImageNet classification task.

We use randomized smoothing to train state-of-the-art certifiably $\ell_2$-robust ImageNet classifiers; for example, one of them achieves 49\% provable top-1 accuracy under adversarial perturbations with $\ell_2$ norm less than 127/255 (Table \ref{table:imagenet-certified-accuracy}).
We also demonstrate that on smaller-scale datasets like CIFAR-10 and SHVN, where competing approaches to certified $\ell_2$ robustness are feasible, randomized smoothing can deliver better certified accuracies, both because it enables the use of larger networks and because it does not constrain the expressivity of the base classifier. 

\section{Related Work}

\begin{figure}[t]
\begin{center}
\includegraphics[width=110px]{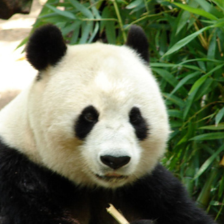}
\includegraphics[width=110px]{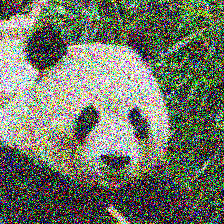}
\caption{The smoothed classifier's prediction at an input $x$ (left) is defined as the most likely prediction by the base classifier on random Gaussian corruptions of $x$ (right; $\sigma=0.5$).
Note that this Gaussian noise is much larger in magnitude than the adversarial perturbations to which $g$ is provably robust.
One interpretation of randomized smoothing is that these large random perturbations ``drown out'' small adversarial perturbations.}
\label{noisypanda}
\end{center}
\end{figure}

Many works have proposed classifiers intended to be robust to adversarial perturbations.
These approaches can be broadly divided into \emph{empirical} defenses, which empirically seem robust to known adversarial attacks, and \emph{certified} defenses, which are \emph{provably} robust to certain kinds of adversarial perturbations.

\paragraph{Empirical defenses}
The most successful empirical defense to date is \emph{adversarial training} \citep{goodfellow2015explaining,kurakin2017adversarial,madry2017towards}, in which adversarial examples are found during training (often using projected gradient descent) and added to the training set.  Unfortunately, it is typically impossible to tell whether a prediction by an empirically robust classifier is truly robust to adversarial perturbations; the most that can be said is that a specific attack was unable to find any.
In fact, many heuristic defenses proposed in the literature were later ``broken'' by stronger adversaries \citep{carlini2017adversarial, athalye2018obfuscated, uesato2018adversarial, athalye2018robustness}.
Aiming to escape this cat-and-mouse game, a growing body of work has focused on defenses with formal guarantees.

\paragraph{Certified defenses}
A classifier is said to be \emph{certifiably robust} if for any input $x$, one can easily obtain a guarantee that the classifier's prediction is constant within some set around $x$, often an $\ell_2$ or $\ell_\infty$ ball.
In most work in this area, the certifiably robust classifier is a neural network.
Some works propose algorithms for certifying the robustness of generically trained networks, while others \citep{wong2018provable, raghunathan2018certified} propose both a robust training method and a complementary certification mechanism.

Certification methods are either \emph{exact} (a.k.a ``complete'') or \emph{conservative} (a.k.a ``sound but incomplete'').
In the context of $\ell_p$ norm-bounded perturbations, exact methods take a classifier $g$, input $x$, and radius $r$, and report whether or not there exists a perturbation $\delta$ within $\|\delta\| \leq r$ for which $g(x) \neq g(x+\delta)$.  In contrast, conservative methods either certify that no such perturbation exists or decline to make a certification; they may decline even when it is true that no such perturbation exists.
Exact methods are usually based on Satisfiability Modulo Theories \citep{Katz2017, carlini2017provably, ehlers2017formal, huang2016safety} or mixed integer linear programming \citep{Cheng2017, lomuscio2017approach, dutta2017output, fischetti2017deep, bunel2017unified}.  Unfortunately, no exact methods have been shown to scale beyond moderate-sized (100,000 activations) networks \citep{tjeng2017evaluating}, and networks of that size can only be verified when they are trained in a manner that impairs their expressivity.

Conservative certification is more scalable.
Some conservative methods bound the \emph{global} Lipschitz constant of the neural network \citep{gouk2018regularisation, tsuzuku2018lipschitz, anil2019sorting, cisse2017parseval}, but these approaches tend to be very loose on expressive networks.
Others measure the \emph{local} smoothness of the network in the vicinity of a particular input $x$.
In theory, one could obtain a robustness guarantee via an upper bound on the local Lipschitz constant of the network \citep{hein2017formal}, but computing this quantity is intractable for general neural networks.
Instead, a panoply of practical solutions have been proposed in the literature \citep{wong2018provable, wang2018mixtrain, wang2018efficient, raghunathan2018certified, raghunathan2018semidefinite, wong2018scaling, dvijotham2018dual, dvijotham2018training, croce2018provable, gehr2018safety, mirman2018differentiable, singh2018fast, gowal2018effectiveness, weng2018toward, zhang2018efficient}.
Two themes stand out.  Some approaches cast verification as an optimization problem and import tools such as relaxation and duality from the optimization literature to provide conservative guarantees \citep{wong2018provable, wong2018scaling, raghunathan2018certified, raghunathan2018semidefinite, dvijotham2018dual, dvijotham2018training}.
Others step through the network layer by layer, maintaining at each layer an outer approximation of the set of activations reachable by a perturbed input \citep{mirman2018differentiable, singh2018fast, gowal2018effectiveness, weng2018toward, zhang2018efficient}.
None of these local certification methods have been shown to be feasible on networks that are large and expressive enough to solve modern machine learning problems like the ImageNet classification task.
Also, all either assume specific network architectures (e.g. ReLU activations or a layered feedforward structure) or require extensive customization for new network architectures.

\paragraph{Related work involving noise}
Prior works have proposed using a network's robustness to Gaussian noise as a proxy for its robustness to adversarial perturbations \citep{weng2018evaluating, ford2019adversarial}, and have suggested that Gaussian data augmentation could supplement or replace adversarial training \citep{zantedeschi2017efficient, kannan2018adversarial}.  \citet{smilkov2017smoothgrad} observed that averaging a classifier's input gradients over Gaussian corruptions of an image yields very interpretable saliency maps.
The robustness of neural networks to random noise has been analyzed both theoretically \citep{fawzi2016robustness, franceschi2018robustness} and empirically \citep{Dodge_2017}.
Finally, \citet{webb2018statistical} proposed a statistical technique for estimating the noise robustness of a classifier more efficiently than naive Monte Carlo simulation; we did not use this technique since it appears to lack formal high-probability guarantees.
While these works hypothesized relationships between a neural network's robustness to random noise and \emph{the same network's} robustness to adversarial perturbations, randomized smoothing instead uses a classifier's robustness to random noise \emph{to create a new classifier} robust to adversarial perturbations. 

\paragraph{Randomized smoothing}
Randomized smoothing has been studied previously for adversarial robustness.  Several works \citep{liu2018towards, cao2017mitigating} proposed similar techniques as heuristic defenses, but did not prove any guarantees.
\citet{lecuyer2018certified} used inequalities from the differential privacy literature to prove an $\ell_2$ and $\ell_1$ robustness guarantee for smoothing with Gaussian and Laplace noise, respectively.
Subsequently, \citet{li2018second} used tools from information theory to prove a stronger $\ell_2$ robustness guarantee for Gaussian noise.
However, all of these robustness guarantees are loose.
In contrast, we prove a tight robustness guarantee in $\ell_2$ norm for randomized smoothing with Gaussian noise.

\section{Randomized smoothing}
\label{section:randomizedsmoothing}

Consider a classification problem from $\mathbb{R}^d$ to classes $\mathcal{Y}$.  As discussed above, randomized smoothing is a method for constructing a new, ``smoothed'' classifier $g$ from an arbitrary base classifier $f$.
When queried at $x$, the smoothed classifier $g$ returns whichever class the base classifier $f$ is most likely to return when $x$ is perturbed by isotropic Gaussian noise:
\begin{align}
\label{g}
g(x) &= \argmax_{c \in \mathcal{Y}} \; \mathbb{P}(f(x+\varepsilon) = c) \\
\text{where } \; \varepsilon &\sim \mathcal{N}(0, \sigma^2 I) \nonumber
\end{align}
An equivalent definition is that $g(x)$ returns the class $c$ whose pre-image $\{x' \in \mathbb{R}^d: f(x') = c\}$ has the largest probability measure under the distribution $\mathcal{N}(x, \sigma^2 I)$.
The noise level $\sigma$ is a hyperparameter of the smoothed classifier $g$ which controls a robustness/accuracy tradeoff; it does not change with the input $x$.
We leave undefined the behavior of $g$ when the argmax is not unique.

We will first present our robustness guarantee for the smoothed classifier $g$.  Then, since it is not possible to exactly evaluate the prediction of $g$ at $x$ or to certify the robustness of $g$ around $x$, we will give Monte Carlo algorithms for both tasks that succeed with arbitrarily high probability.

\subsection{Robustness guarantee}

Suppose that when the base classifier $f$ classifies  $\mathcal{N}(x, \sigma^2 I)$, the most probable class $c_A$ is returned with probability  $p_A $, and the ``runner-up'' class  is returned with probability $p_B$.
Our main result is that smoothed classifier $g$ is robust around $x$ within the $\ell_2$ radius $R = \frac{\sigma}{2} (\Phi^{-1}(p_A) - \Phi^{-1}(p_B))$, where $\Phi^{-1}$ is the inverse of the standard Gaussian CDF.
This result also holds if we replace $p_A$ with a lower bound $\underline{p_A}$ and we replace $p_B$ with an upper bound $\overline{p_B}$.

\begin{thm}
\label{mainbound}
Let $f: \mathbb{R}^d \to \mathcal{Y}$ be any deterministic or random function, and let $\varepsilon \sim \mathcal{N}(0, \sigma^2 I)$.
Let $g$ be defined as in (\ref{g}).
Suppose $c_A \in \mathcal{Y}$ and $\underline{p_A}, \overline{p_B} \in [0, 1]$ satisfy:
\small
\begin{align}
\label{knownfacts}
\mathbb{P}(f(x+\varepsilon) = c_A) \ge \underline{p_A} \ge \overline{p_B }\ge \max_{c \neq c_A} \mathbb{P}(f(x+ \varepsilon) = c)
\end{align}
\normalsize
Then $g(x+\delta) = c_A$ for all $\|\delta\|_2 < R$, where
\begin{align}
R = \frac{\sigma}{2} (\Phi^{-1}(\underline{p_A}) - \Phi^{-1}(\overline{p_B})) \label{radius}
\end{align}
\end{thm}

We now make several observations about Theorem \ref{mainbound}\footnote{After the dissemination of this work, a more general result was published in \citet{alex2019certifiably, salman2019provably}: if $h: \mathbb{R}^d \to [0, 1]$ is a function and $\hat{h}$ is the ``smoothed'' version $\hat{h}(x) = \mathbb{E}_{\varepsilon \sim \mathcal{N}(0, \sigma^2 I)}[h(x+\varepsilon)]$, then the function $x \mapsto \Phi^{-1}(\hat{h}(x))$ is $1/\sigma$-Lipschitz. 
Theorem \ref{mainbound} can be proved by applying this result to the functions $f_c(x) = \mathbf{1}[f(x) = c]$ for each class $c$.
}
\begin{itemize}
\item Theorem \ref{mainbound} assumes nothing about $f$.
This is crucial since it is unclear which well-behavedness assumptions, if any, are satisfied by modern deep architectures.
\item The certified radius $R$ is large when: (1) the noise level $\sigma$ is high, (2) the probability of the top class $c_A$ is high, and (3) the probability of each other class is low.
\item The certified radius $R$ goes to $\infty$ as $\underline{p_A} \to 1$ and $\overline{p_B} \to 0$.
This should sound reasonable: the Gaussian distribution is supported on all of $\mathbb{R}^d$, so the only way that $f(x+\varepsilon) = c_A$ with probability 1 is if $f = c_A$ almost everywhere.
\end{itemize}
Both \citet{lecuyer2018certified} and \citet{li2018second} proved $\ell_2$ robustness guarantees for the same setting as Theorem \ref{mainbound}, but with different, smaller expressions for the certified radius.
However, our $\ell_2$ robustness guarantee is \emph{tight}: if (\ref{knownfacts}) is all that is known about $f$, then it is impossible to certify an $\ell_2$ ball with radius larger than $R$.
In fact, it is impossible to certify any superset of the $\ell_2$ ball with radius $R$:
\begin{thm}
\label{mainboundtight}
Assume $\underline{p_A} + \overline{p_B} \le 1$.
For any perturbation $\delta$ with $\|\delta\|_2 > R$, there exists a base classifier $f$ consistent with the class probabilities (\ref{knownfacts}) for which $g(x+\delta) \neq c_A$.
\end{thm}

Theorem \ref{mainboundtight} shows that Gaussian smoothing naturally induces $\ell_2$ robustness:  if we make no assumptions on the base classifier beyond the class probabilities (\ref{knownfacts}), then the set of perturbations to which a Gaussian-smoothed classifier is provably robust is \emph{exactly} an $\ell_2$ ball.

The complete proofs of Theorems \ref{mainbound} and \ref{mainboundtight} are in Appendix \ref{section:fullproof}.
We now sketch the proofs in the special case when there are only two classes.

\textbf{Theorem 1 (binary case).}
\textit{Suppose $\underline{p_A} \in (\frac{1}{2}, 1]$ satisfies $\mathbb{P}(f(x+\varepsilon) = c_A) \ge \underline{p_A}$.
Then $g(x+\delta) = c_A$ for all $\|\delta\|_2 < \sigma \Phi^{-1}(\underline{p_A})$.}

\begin{proof}[Proof sketch]
Fix a perturbation $\delta \in \mathbb{R}^d$.
To guarantee that $g(x+\delta) = c_A$, we need to show that $f$ classifies the translated Gaussian $\mathcal{N}(x+\delta, \sigma^2 I)$ as $c_A$ with probability $> \frac{1}{2}$.
However, all we know about $f$ is that $f$ classifies $\mathcal{N}(x, \sigma^2 I)$ as $c_A$ with probability $\ge \underline{p_A}$.
This raises the question: out of all possible base classifiers $f$ which classify $\mathcal{N}(x, \sigma^2I)$ as $c_A$ with probability $\ge \underline{p_A}$, which one $f^*$ classifies $\mathcal{N}(x+\delta, \sigma^2 I)$ as $c_A$ with the smallest probability?
One can show using an argument similar to the Neyman-Pearson lemma \citep{neyman1933}  that this ``worst-case'' $f^*$ is a linear classifier whose decision boundary is normal to the perturbation $\delta$ (Figure \ref{figure:worstcase}):
\begin{align}
f^*(x') = \begin{cases}
c_A &\mbox{ if }  \delta^T(x' - x) \le \sigma \|\delta\|_2 \Phi^{-1}(\underline{p_A}) \\
c_B &\mbox{ otherwise}
\end{cases}
\label{fstar}
\end{align}
This ``worst-case''  $f^*$ classifies $\mathcal{N}(x+\delta, \sigma^2 I)$ as $c_A$ with probability
$ \Phi \left( \Phi^{-1}(\underline{p_A}) - \frac{\|\delta\|_2}{\sigma} \right) $.
Therefore, to ensure that even the ``worst-case'' $f^*$ classifies $\mathcal{N}(x+\delta, \sigma^2 I)$ as $c_A$ with probability $>\frac{1}{2}$, we solve for those $\delta$ for which
$$  \Phi \left( \Phi^{-1}(\underline{p_A}) - \frac{\|\delta\|_2}{\sigma} \right) > \frac{1}{2} $$
which is equivalent to the condition $\|\delta\|_2 < \sigma \Phi^{-1}(\underline{p_A})$.
\end{proof}

Theorem \ref{mainboundtight} is a simple consequence: for any $\delta$ with $\|\delta\|_2 > R$, the base classifier $f^*$ defined in (\ref{fstar}) is consistent with (\ref{knownfacts}); yet if $f^*$ is the base classifier, then $g(x+\delta) = c_B$.

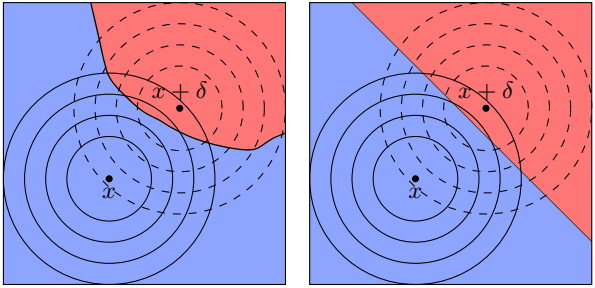
\begin{figure}[t]
\begin{center}
\begin{tikzpicture}[scale=0.9375]
\node at (2,2) {\includegraphics[width=3.75cm]{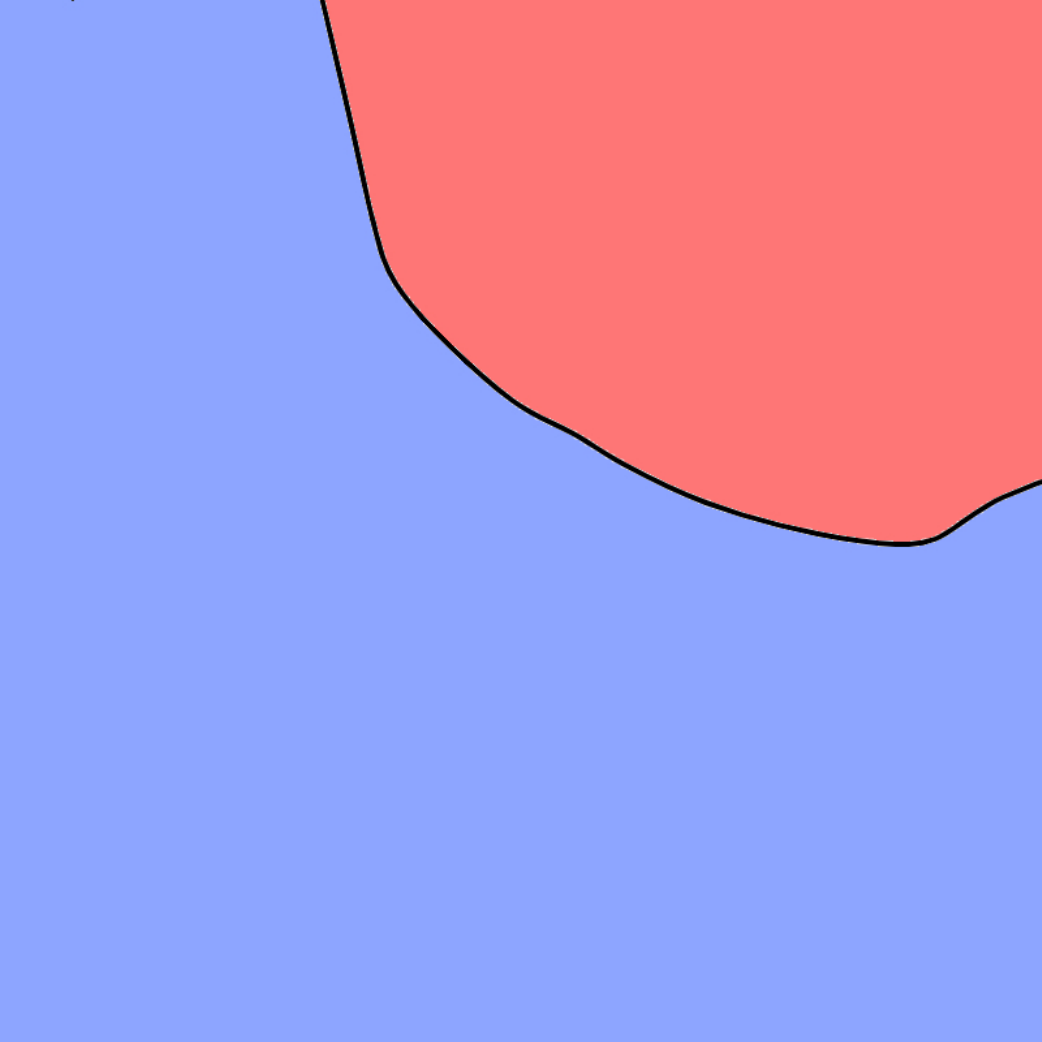}};
\draw (0,0) rectangle (4, 4);
\fill[black] (2.5,2.5) circle (.05);
\draw[dashed] (2.5,2.5) circle (0.6);
\draw[dashed] (2.5,2.5) circle (0.9);
\draw[dashed] (2.5,2.5) circle (1.2);
\draw[dashed] (2.5,2.5) circle (1.5);
\fill[black] (1.5,1.5) circle (.05);
\draw[solid] (1.5,1.5) circle (0.6);
\draw[solid] (1.5,1.5) circle (0.9);
\draw[solid] (1.5,1.5) circle (1.2);
\draw[solid] (1.5,1.5) circle (1.5);
\node at (2.5, 2.75) {\small $x+\delta$};
\node at (1.5, 1.3) {\small $x$};
\end{tikzpicture}
\begin{tikzpicture}[scale=0.9375]
\node at (2,2) {\includegraphics[width=3.75cm]{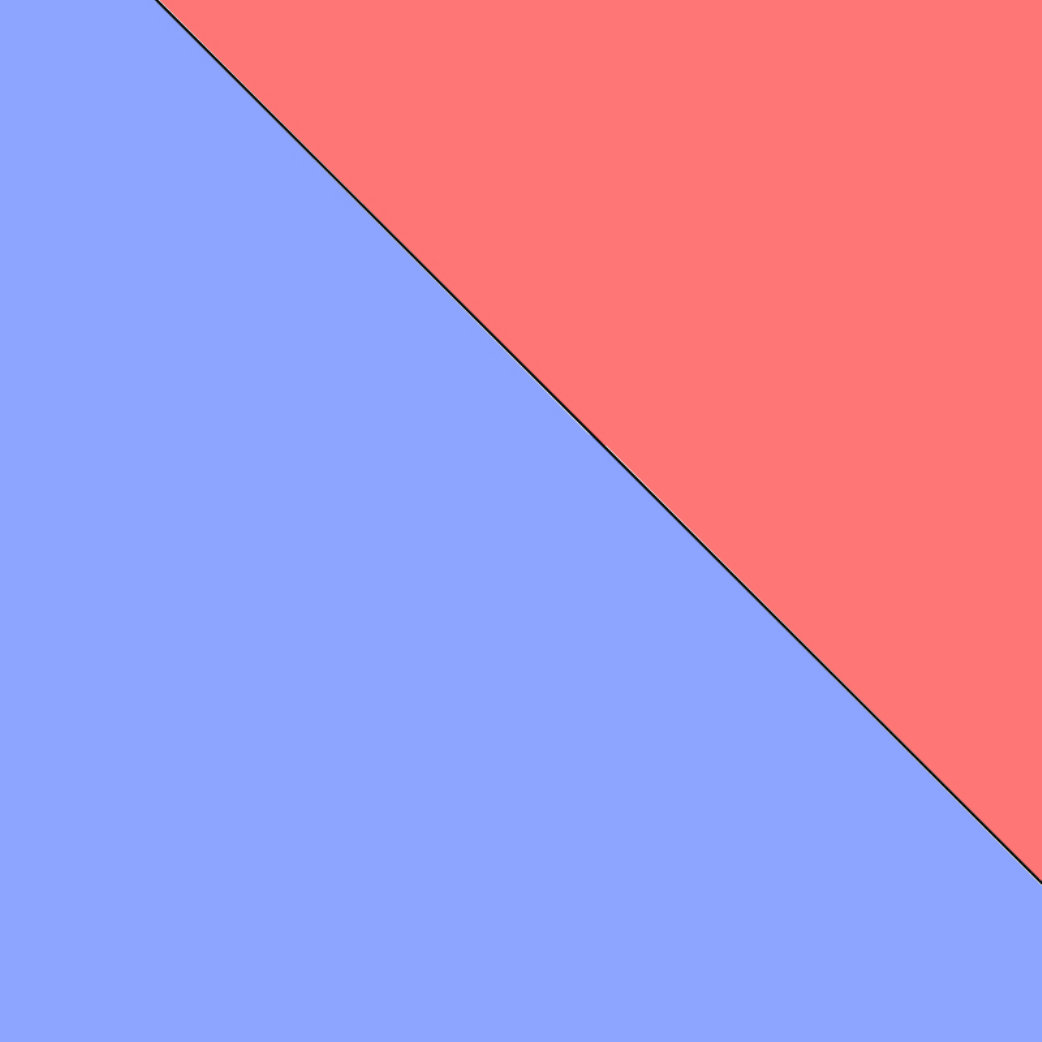}};
\draw (0,0) rectangle (4, 4);
\fill[black] (2.5,2.5) circle (.05);
\draw[dashed] (2.5,2.5) circle (0.6);
\draw[dashed] (2.5,2.5) circle (0.9);
\draw[dashed] (2.5,2.5) circle (1.2);
\draw[dashed] (2.5,2.5) circle (1.5);
\fill[black] (1.5,1.5) circle (.05);
\draw[solid] (1.5,1.5) circle (0.6);
\draw[solid] (1.5,1.5) circle (0.9);
\draw[solid] (1.5,1.5) circle (1.2);
\draw[solid] (1.5,1.5) circle (1.5);
\node at (2.5, 2.75) {\small $x+\delta$};
\node at (1.5, 1.3) {\small $x$};
\end{tikzpicture}
\end{center}
\caption{Illustration of $f^*$ in two dimensions.
The concentric circles are the density contours of $\mathcal{N}(x, \sigma^2 I)$ and $\mathcal{N}(x+\delta, \sigma^2 I)$.
Out of all base classifiers $f$ which classify $\mathcal{N}(x, \sigma^2 I)$ as $c_A$ (blue) with probability $\ge \underline{p_A}$, such as both classifiers depicted above, the ``worst-case'' $f^*$ --- the one which classifies  $\mathcal{N}(x+\delta, \sigma^2 I)$ as $c_A$ with minimal probability --- is depicted on the right: a linear classifier with decision boundary normal to the perturbation $\delta$.}
\label{figure:worstcase}
\end{figure}

Figure \ref{figure:bounds} (left) plots our $\ell_2$ robustness guarantee against the guarantees derived in prior work.
Observe that our $R$ is much larger than that of \citet{lecuyer2018certified} and moderately larger than that of \citet{li2018second}.
Appendix \ref{section:derivation} derives the other two guarantees using this paper's notation.

\paragraph{Linear base classifier}
A two-class linear classifier $f(x) = \text{sign}(w^T x + b)$ is already certifiable: the distance from any input $x$ to the decision boundary is $|w^T x+ b|/\|w\|_2$, and no perturbation $\delta$ with $\ell_2$ norm less than this distance can possibly change $f$'s prediction.
In Appendix \ref{section:otherdeferred} we show that if $f$ is linear, then the smoothed classifier $g$ is identical to the base classifier $f$.
Moreover, we show that our bound (\ref{radius}) will certify the true robust radius $|w^T x + b|/\|w\|$, rather than a smaller, overconservative radius.
Therefore, when $f$ is linear, there always exists a perturbation $\delta$ just beyond the certified radius which changes $g$'s prediction.

\paragraph{Noise level can scale with image resolution}
Since our expression (\ref{radius}) for the certified radius does not depend explicitly on the data dimension $d$, one might worry that randomized smoothing is less effective for images of higher resolution --- certifying a fixed $\ell_2$ radius is ``less impressive'' for, say, a $224 \times 224$ image than for a $56 \times 56$ image.
However, as illustrated by Figure \ref{figure:noise_dimension_paper}, images in higher resolution can tolerate higher levels $\sigma$ of isotropic Gaussian noise before their class-distinguishing content gets destroyed.
As a consequence, in high resolution, smoothing can be performed with a larger $\sigma$, leading to larger certified radii.
See Appendix \ref{section:input-resolution} for a more rigorous version of this argument.

\begin{figure}[b]
\begin{center}
	\includegraphics[width=45px]{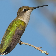}
	\includegraphics[width=45px]{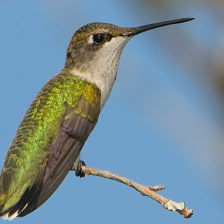}
	\includegraphics[width=45px]{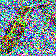}
	\includegraphics[width=45px]{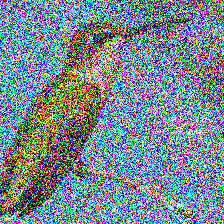}
\end{center}
\caption{Left to right: clean 56 x 56 image, clean 224 x 224 image, noisy 56 x 56 image ($\sigma=0.5)$, noisy 224 x 224 image ($\sigma=0.5$).}
\label{figure:noise_dimension_paper}
\end{figure}

\subsection{Practical algorithms}
\label{sec:practical-algorithms}

We now present practical Monte Carlo algorithms for evaluating $g(x)$ and certifying the robustness of $g$ around $x$.
More details can be found in Appendix \ref{section:practical-algorithms-appendix}.

\subsubsection{Prediction}
Evaluating the smoothed classifier's prediction $g(x)$ requires identifying the class $c_A$ with maximal weight in the categorical distribution $f(x+\varepsilon)$.
The procedure described in pseudocode as \textsc{Predict} draws $n$ samples of $f(x+\varepsilon)$ by running $n$ noise-corrupted copies of $x$ through the base classifier.
Let $\hat{c}_A$ be the class which appeared the largest number of times.
If $\hat{c}_A$ appeared much more often than any other class, then \textsc{Predict} returns $\hat{c}_A$.
Otherwise, it abstains from making a prediction.
We use the hypothesis test from \citet{hung2017rank} to calibrate the abstention threshold so as to bound by $\alpha$ the probability of returning an incorrect answer.
\textsc{Predict} satisfies the following guarantee:
\begin{proposition}
\label{thm:predict-correct}
With probability at least $1 - \alpha$ over the randomness in \textsc{Predict}, \textsc{Predict} will either abstain or return $g(x)$.
(Equivalently: the probability that \textsc{Predict} returns a class other than $g(x)$ is at most $\alpha$.)
\end{proposition}

The function \textsc{SampleUnderNoise}($f$, $x$, num, $\sigma$) in the pseudocode draws num samples of noise, $\varepsilon_1 \hdots \varepsilon_{\text{num}} \sim \mathcal{N}(0, \sigma^2 I)$, runs each $x + \varepsilon_i$ through the base classifier $f$, and returns a vector of class counts.
\textsc{BinomPValue}($n_A$, $n_A + n_B$, $p$) returns the p-value of the two-sided hypothesis test that $n_A \sim \text{Binomial}(n_A + n_B, p)$.

\begin{algorithm}[t]
   \caption{for certification and prediction}
   \label{algos}
\begin{algorithmic}
   \STATE \textit{\# evaluate $g$ at $x$}
   \STATE \textbf{function} \textsc{Predict}($f$, $\sigma$, $x$, $n$, $\alpha$)
   \STATE \quad \texttt{counts} $\leftarrow$ \textsc{SampleUnderNoise}($f$, $x$, $n$, $\sigma$)
   \STATE \quad $\hat{c}_A, \hat{c}_B \leftarrow$ top two indices in \texttt{counts}
   \STATE \quad $n_A, n_B \leftarrow $ \texttt{counts}[$\hat{c}_A$],  \texttt{counts}[$\hat{c}_B$]
   \STATE \quad \textbf{if} \textsc{BinomPValue}($n_A$, $n_A + n_B$, 0.5) $\le \alpha$ \textbf{return} $\hat{c}_A$
   \STATE \quad \textbf{else return} ABSTAIN
   \STATE
   \STATE \textit{\# certify the robustness of $g$ around $x$}
   \STATE \textbf{function} \textsc{Certify}($f$, $\sigma$, $x$, $n_0$, $n$, $\alpha$)
   \STATE \quad $\texttt{counts0} \leftarrow \textsc{SampleUnderNoise}(f, x, n_0, \sigma)$
   \STATE \quad $\hat{c}_A \leftarrow$ top index in \texttt{counts0}
   \STATE \quad $\texttt{counts} \leftarrow \textsc{SampleUnderNoise}(f, x, n, \sigma)$
   \STATE \quad $\underline{p_A} \leftarrow \textsc{LowerConfBound}$($\texttt{counts}[\hat{c}_A]$, $n$, $1 - \alpha$) 
   \STATE \quad \textbf{if} $\underline{p_A} > \frac{1}{2}$ \textbf{return} prediction $\hat{c}_A$ and radius $\sigma \, \Phi^{-1}(\underline{p_A})$
   \STATE \quad \textbf{else return} ABSTAIN
\end{algorithmic}
\end{algorithm}

Even if the true smoothed classifier $g$ is robust at radius $R$, \textsc{Predict} will be vulnerable in a certain sense to adversarial perturbations with $\ell_2$ norm slightly less than $R$.
By engineering a perturbation $\delta$ for which $f(x+\delta+\varepsilon)$ puts mass just over $\frac{1}{2}$ on class $c_A$ and mass just under $\frac{1}{2}$ on class $c_B$, an adversary can force \textsc{predict} to abstain at a high rate.
If this scenario is of concern, a variant of Theorem \ref{mainbound} could be proved to certify a radius in which $\mathbb{P}(f(x+\delta + \varepsilon) = c_A)$ is larger by some margin than $\max_{c \neq c_A} \mathbb{P}(f(x+\delta+\varepsilon) = c)$.

\subsubsection{Certification}
Evaluating \emph{and} certifying the robustness of $g$ around an input $x$ requires not only identifying the class $c_A$ with maximal weight in $f(x+\varepsilon)$, but also estimating a lower bound $\underline{p_A}$ on the probability that $f(x+\varepsilon) = c_A$ and an upper bound $\overline{p_B}$ on the probability that $f(x+\varepsilon)$ equals any other class.
Doing all three of these at the same time in a statistically correct manner requires some care.
One simple solution is presented in pseudocode as \textsc{Certify}: first, use a small number of samples from $f(x+\varepsilon)$ to take a guess at $c_A$; then use a larger number of samples to estimate $\underline{p_A}$; then simply take $\overline{p_B} = 1-\underline{p_A}$. 

\begin{proposition}
\label{thm:certify-correct}
With probability at least $1 - \alpha$ over the randomness in \textsc{Certify}, if  \textsc{Certify} returns a class $\hat{c}_A$ and a radius $R$ (i.e. does not abstain), then $g$ predicts $\hat{c}_A$ within radius $R$ around $x$:  $g(x + \delta) = \hat{c}_A \;\; \forall \; \|\delta\|_2 < R$.
\end{proposition}

The function \textsc{LowerConfBound}($k$, $n$, $1 - \alpha$) in the pseudocode returns a one-sided $(1 - \alpha)$ lower confidence interval for the Binomial parameter $p$ given a sample $k \sim \text{Binomial}(n, p)$.

\paragraph{Certifying large radii requires many samples}
Recall from Theorem \ref{mainbound} that $R$ approaches $\infty$ as $\underline{p_A}$ approaches 1.
Unfortunately, it turns out that $\underline{p_A}$ approaches 1 so slowly with $n$ that $R$ also approaches $\infty$ very slowly with $n$.
Consider the most favorable situation: $f(x) = c_A$ everywhere.
This means that $g$ is robust at radius $\infty$.
But after observing $n$ samples of $f(x+\varepsilon)$ which all equal $c_A$, the tightest (to our knowledge) lower bound would say that with probability least $1 - \alpha$,  $p_A \ge \alpha^{(1/n)}$.
Plugging $\underline{p_A} = \alpha^{(1/n)}$ and $\overline{p_B} = 1-\underline{p_A}$ into (\ref{radius})  yields an expression for the certified radius as a function of $n$:  $R = \sigma \, \Phi^{-1}(\alpha^{1/n})$.
Figure \ref{figure:bounds} (right) plots this function for $\alpha=0.001, \sigma = 1$.
Observe that certifying a radius of $4 \sigma$ with 99.9\% confidence would require $\approx 10^5$ samples.

\subsection{Training the base classifier}
\label{section:train-base-classifier}

Theorem \ref{mainbound} holds regardless of how the base classifier $f$ is trained.
However, in order for $g$ to classify the labeled example $(x, c)$ correctly and robustly, $f$ needs to consistently classify $\mathcal{N}(x, \sigma^2 I)$ as $c$.
In high dimension, the Gaussian distribution  $\mathcal{N}(x, \sigma^2 I)$ places almost no mass near its mode $x$.
As a consequence, when $\sigma$ is moderately high, the distribution of natural images has virtually disjoint support from the distribution of natural images corrupted by $\mathcal{N}(0, \sigma^2I)$; see Figure \ref{noisypanda} for a visual demonstration.
Therefore, if the base classifier $f$ is trained via standard supervised learning on the data distribution, it will see no noisy images during training, and hence will not necessarily learn to classify $\mathcal{N}(x, \sigma^2 I)$ with $x$'s true label.
Therefore, in this paper we follow \citet{lecuyer2018certified} and train the base classifier with Gaussian data augmentation at variance $\sigma^2$.
A justification for this procedure is provided in Appendix \ref{section:training-with-noise}.
However, we suspect that there may be room to improve upon this training scheme, perhaps by training the base classifier so as to maximize the smoothed classifier's certified accuracy at some tunable radius $r$.

\begin{figure}[t]
	\begin{center}
	\includegraphics[width=115px]{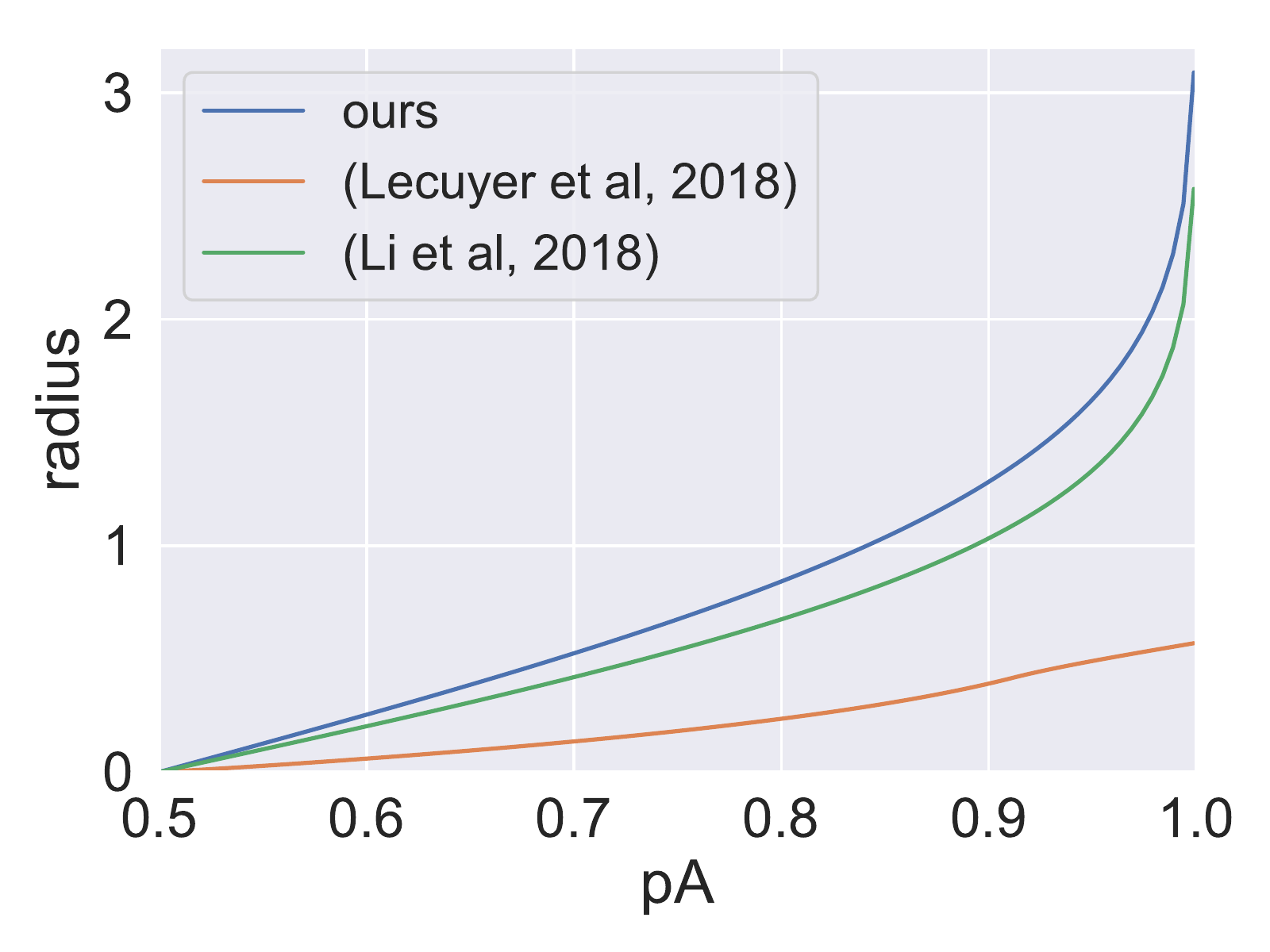}
	\includegraphics[width=115px]{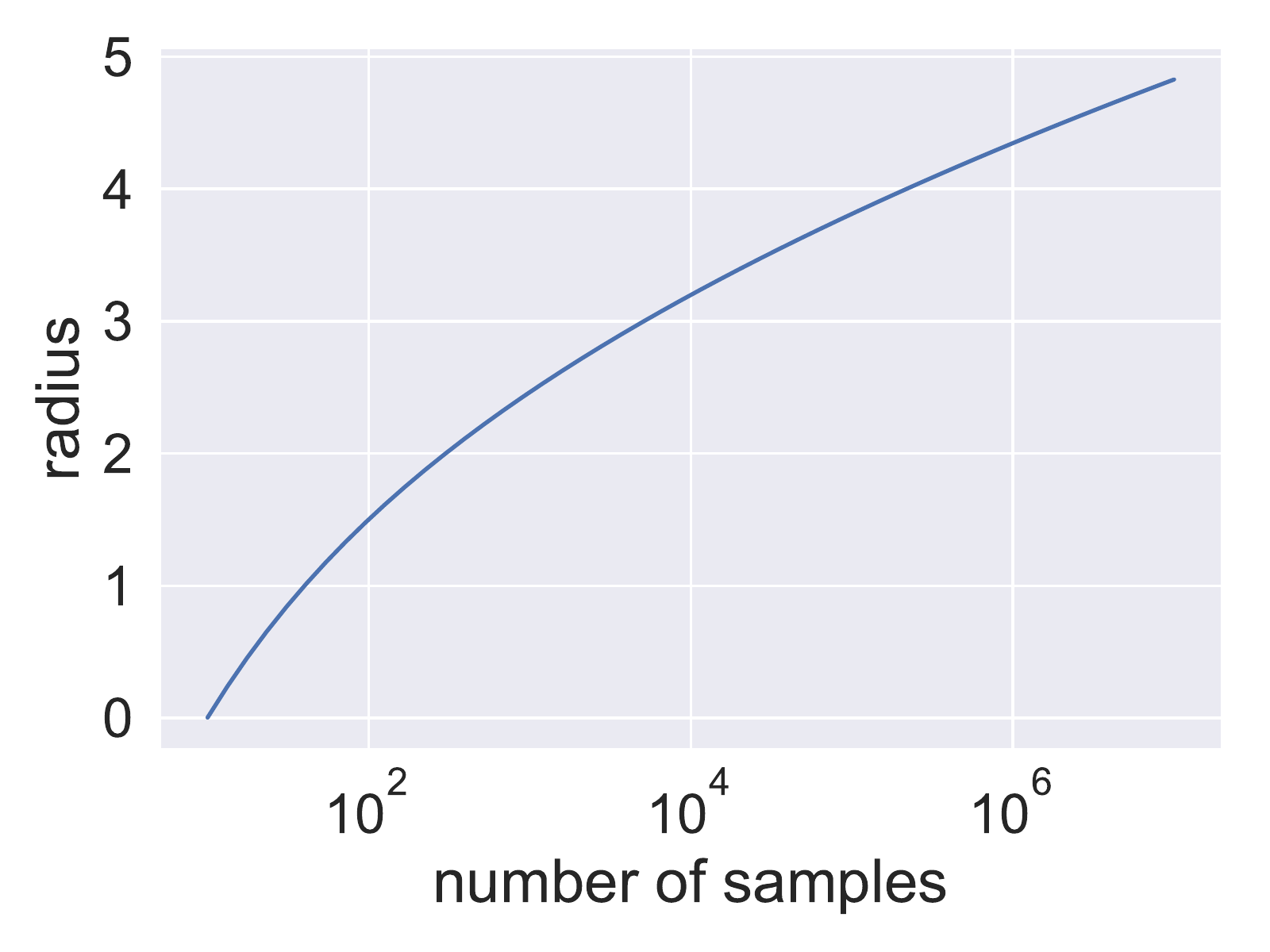}
	\caption{\textbf{Left}: Certified radius $R$ as a function of $\underline{p_A}$ (with $\overline{p_B} = 1 - \underline{p_A}$ and $\sigma = 1$) under all three randomized smoothing bounds.  \textbf{Right}: A plot of $R = \sigma \, \Phi^{-1}(\alpha^{1/n})$ for $\alpha = 0.001$ and $\sigma = 1$.  The radius we can certify with high probability grows slowly with the number of samples, even in the \textit{best} case where $f(x) = c_A$ everywhere.}
	\label{figure:bounds}
	\end{center}
\end{figure}

\section{Experiments}
\label{section:experiments}

In adversarially robust classification, one metric of interest is the \textit{certified test set accuracy} at radius $r$, defined as the fraction of the test set which $g$ classifies correctly with a prediction that is certifiably robust within an $\ell_2$ ball of radius $r$.
However, if $g$ is a randomized smoothing classifier, computing this quantity exactly is not possible, so we instead report the \textit{approximate certified test set accuracy}, defined as the fraction of the test set which \textsc{Certify} classifies correctly (without abstaining) and certifies robust with a radius $R \ge r$.
Appendix \ref{section:highprobability} shows how to convert the approximate certified accuracy into a lower bound on the true certified accuracy that holds with high probability over the randomness in \textsc{Certify}.
However Appendix \ref{section:additionalexperiments:highprobability} demonstrates that when $\alpha$ is small, the difference between these two quantities is negligible.  Therefore, in our experiments we omit the step for simplicity and report approximate certified accuracies.

In all experiments, unless otherwise stated, we ran \textsc{Certify} with $\alpha =0.001$, so there was at most a 0.1\% chance that \textsc{Certify} returned a radius in which $g$ was not truly robust.
Unless otherwise stated, when running \textsc{Certify} we used $n_0 =$ 100 Monte Carlo samples for selection and $n=$ 100,000 samples for estimation.

\begin{figure}[t]
\begin{center}
	\includegraphics[width=210px]{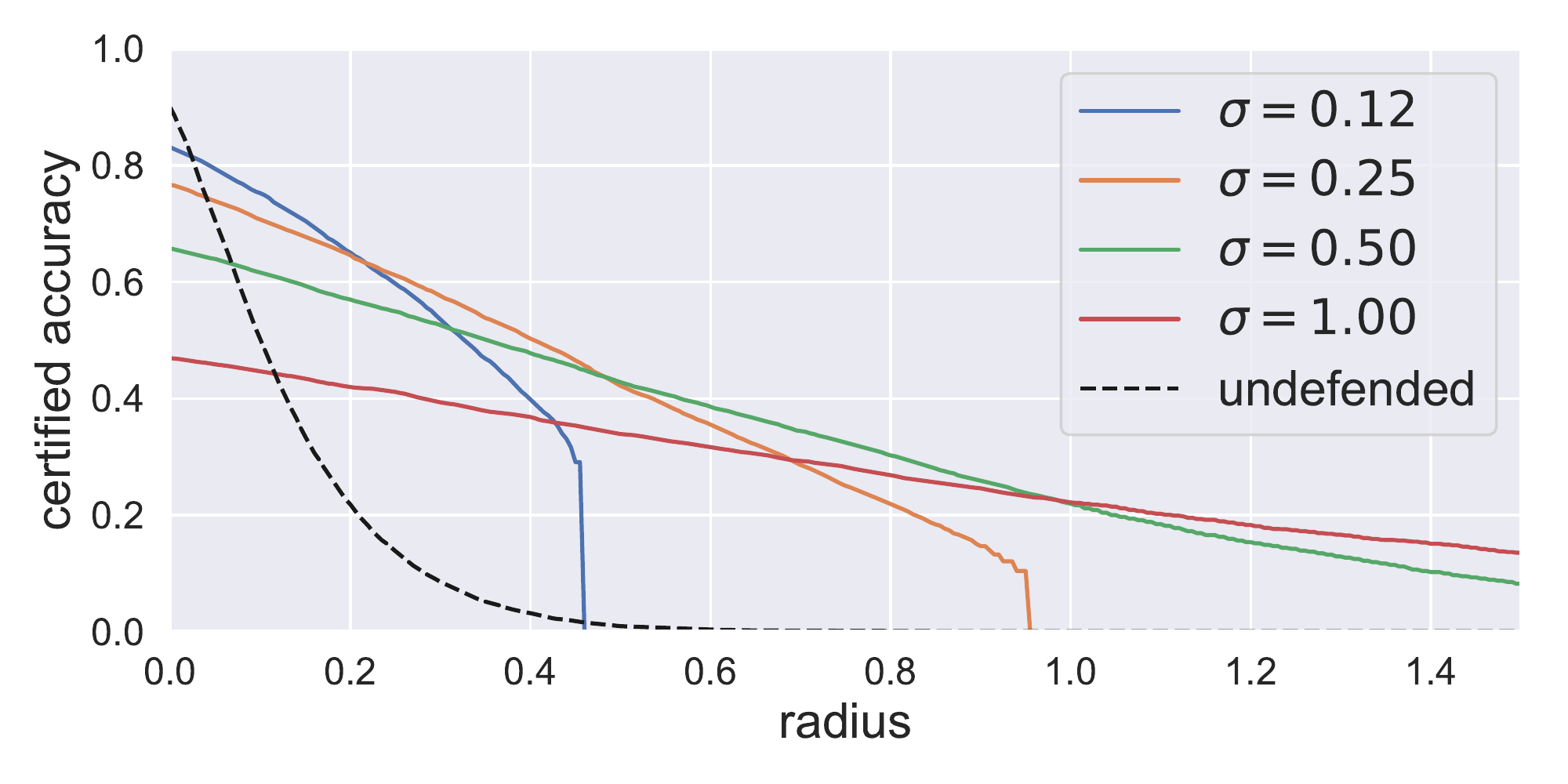}
	\includegraphics[width=210px]{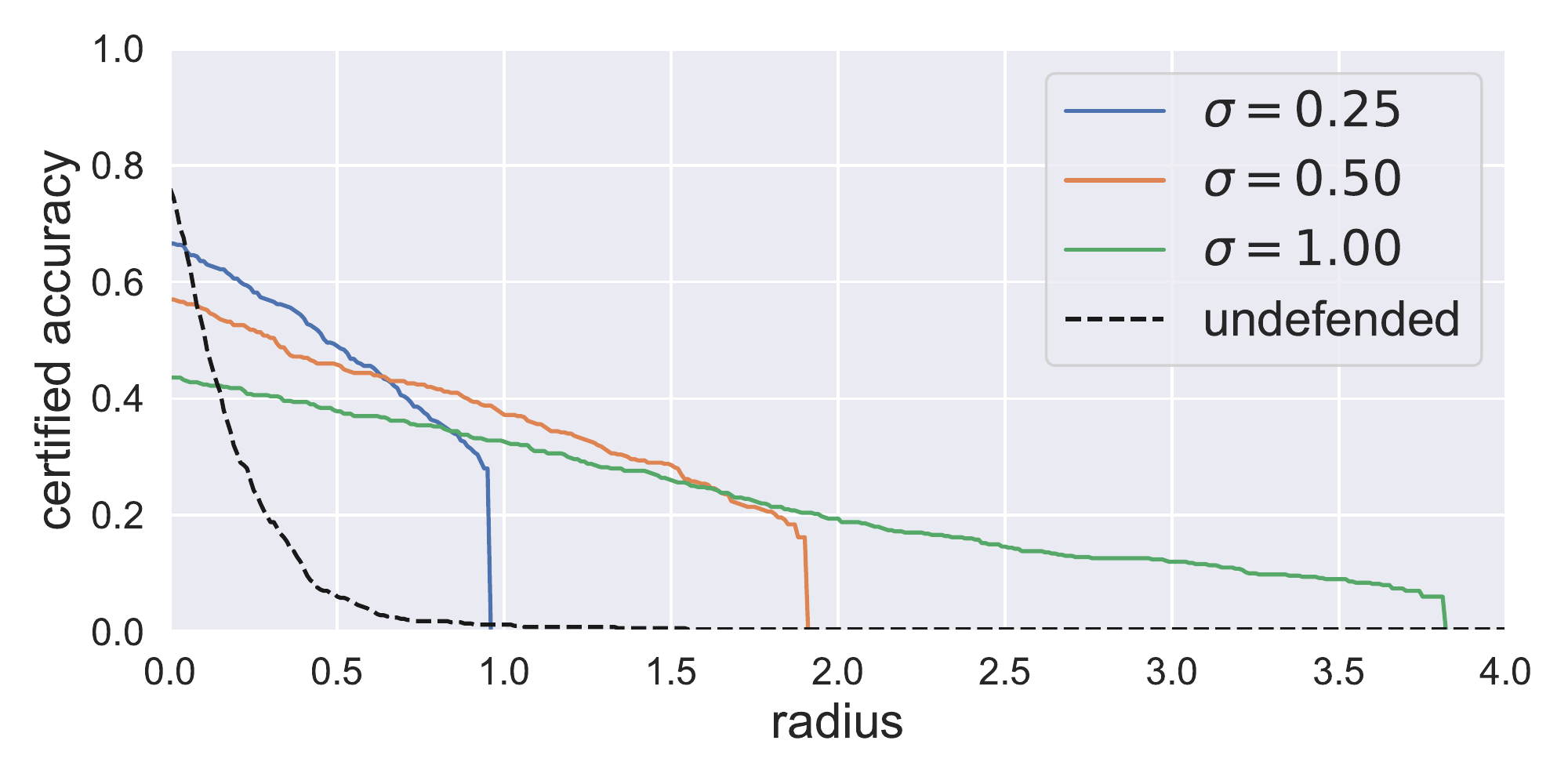}
\caption{Approximate certified accuracy attained by randomized smoothing on CIFAR-10 (\textbf{top}) and ImageNet (\textbf{bottom}).  The hyperparameter $\sigma$ controls a robustness/accuracy tradeoff.  The dashed black line is an upper bound on the empirical robust accuracy of an undefended classifier with the base classifier's architecture.}
\label{fig:certified-accuracy}
\end{center}
\end{figure}

\begin{figure}[t]
\begin{center}
\includegraphics[width=215px]{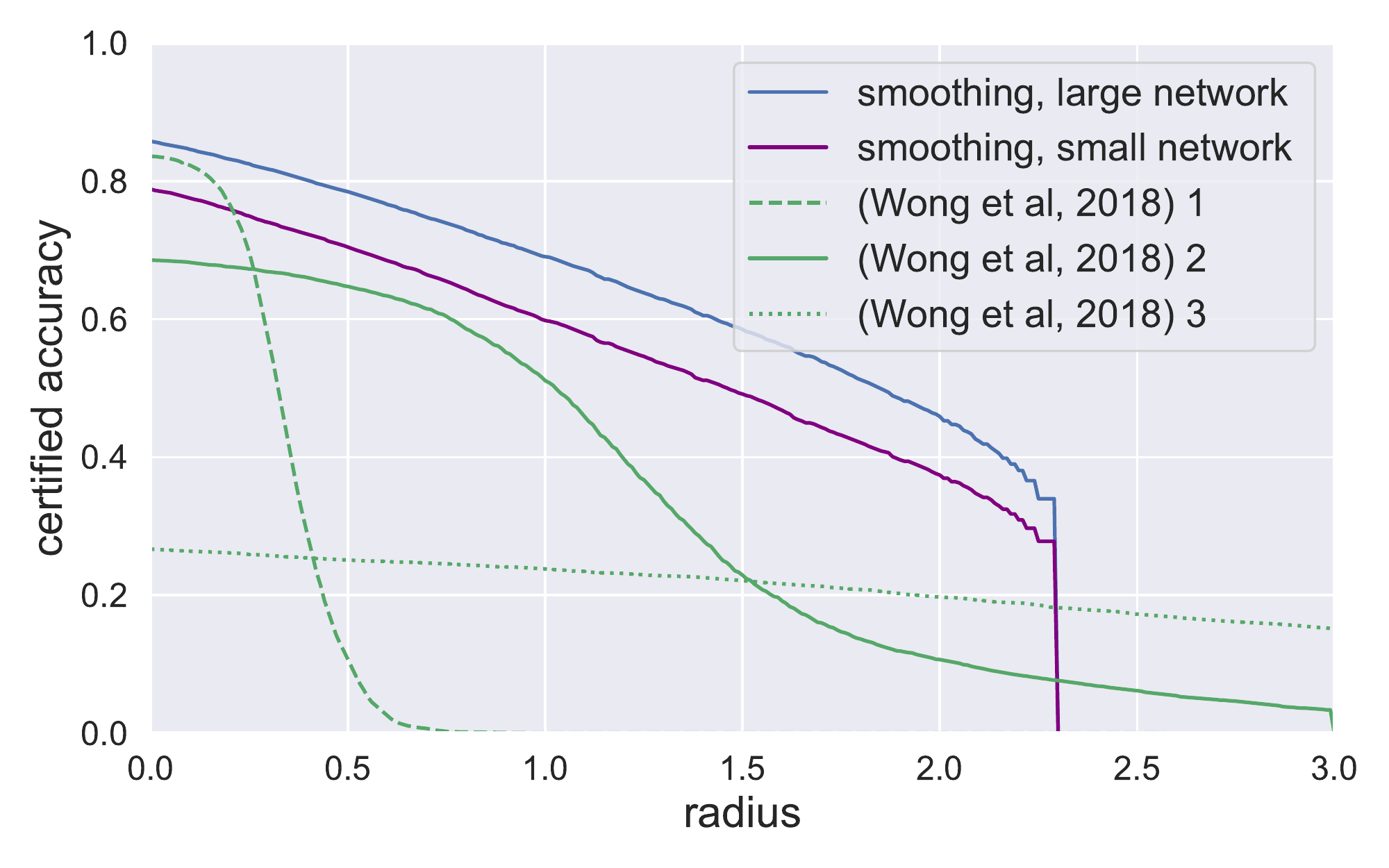}
\end{center}
\caption{Comparison betwen randomized smoothing and \citet{wong2018scaling}.  Each green line is a small resnet classifier trained and certified using the method of \citet{wong2018scaling} with a different setting of its hyperparameter $\epsilon$. 
The purple line is our method using the same small resnet architecture as the base classifier; the blue line is our method with a larger neural network as the base classifier.
\citet{wong2018scaling} gives deterministic robustness guarantees, whereas smoothing gives high-probability guaranatees; therefore, we plot here the certified accuracy of \citet{wong2018scaling} against the ``approximate'' certified accuracy of smoothing.}
\label{fig:compare-wong-fair}
\end{figure}

In the figures above that plot certified accuracy as a function of radius $r$, the certified accuracy always decreases gradually with $r$ until reaching some point where it plummets to zero.
This drop occurs because for each noise level $\sigma$ and number of samples $n$, there is a hard upper limit to the radius we can certify with high probability, achieved when all $n$ samples are classified by $f$ as the same class.

\begin{figure*}[t]
\begin{center}
\begin{subfigure}{0.33\textwidth}
	\includegraphics[width=\textwidth]{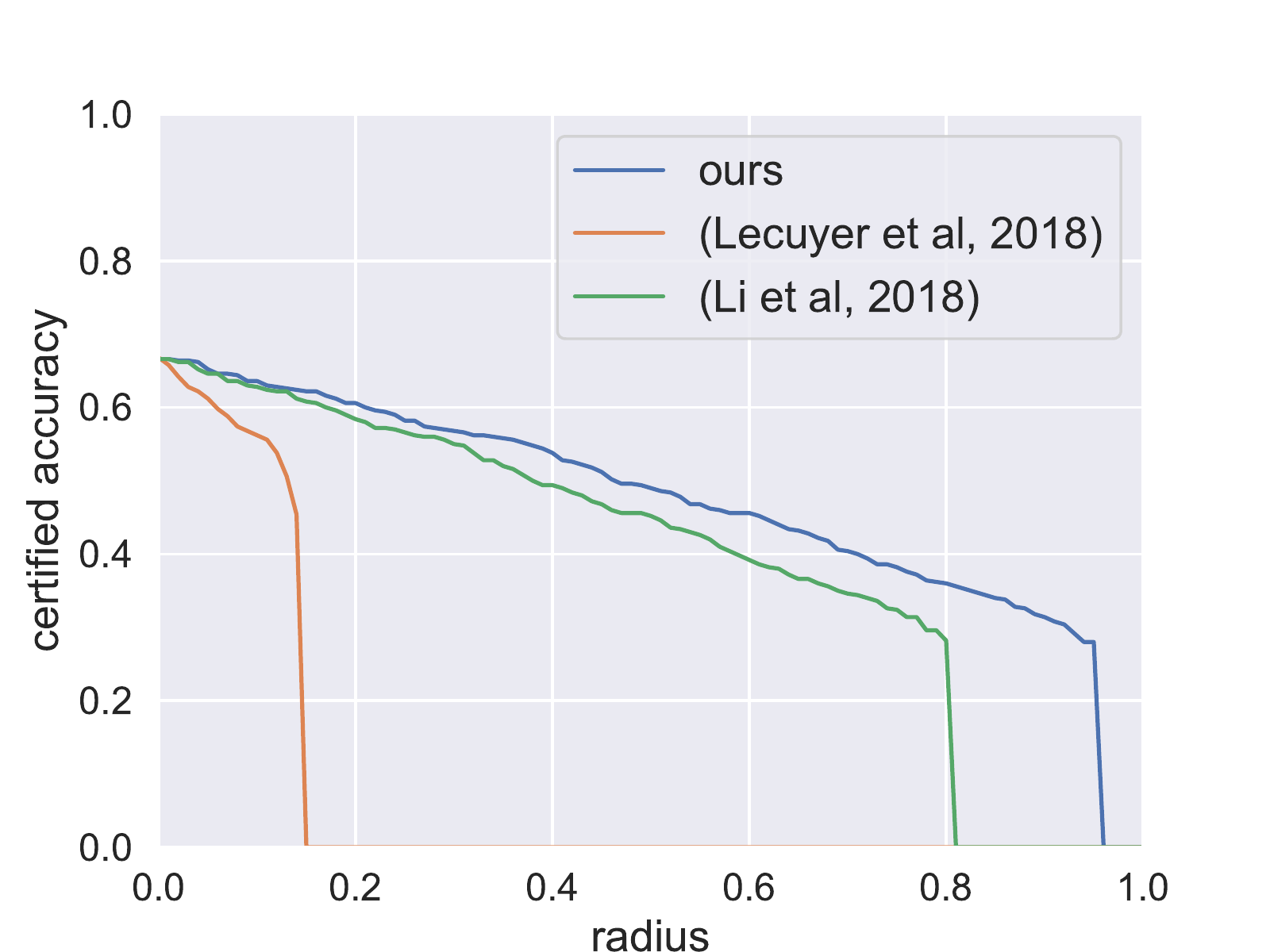}
\end{subfigure}
\begin{subfigure}{0.33\textwidth}
	\includegraphics[width=\textwidth]{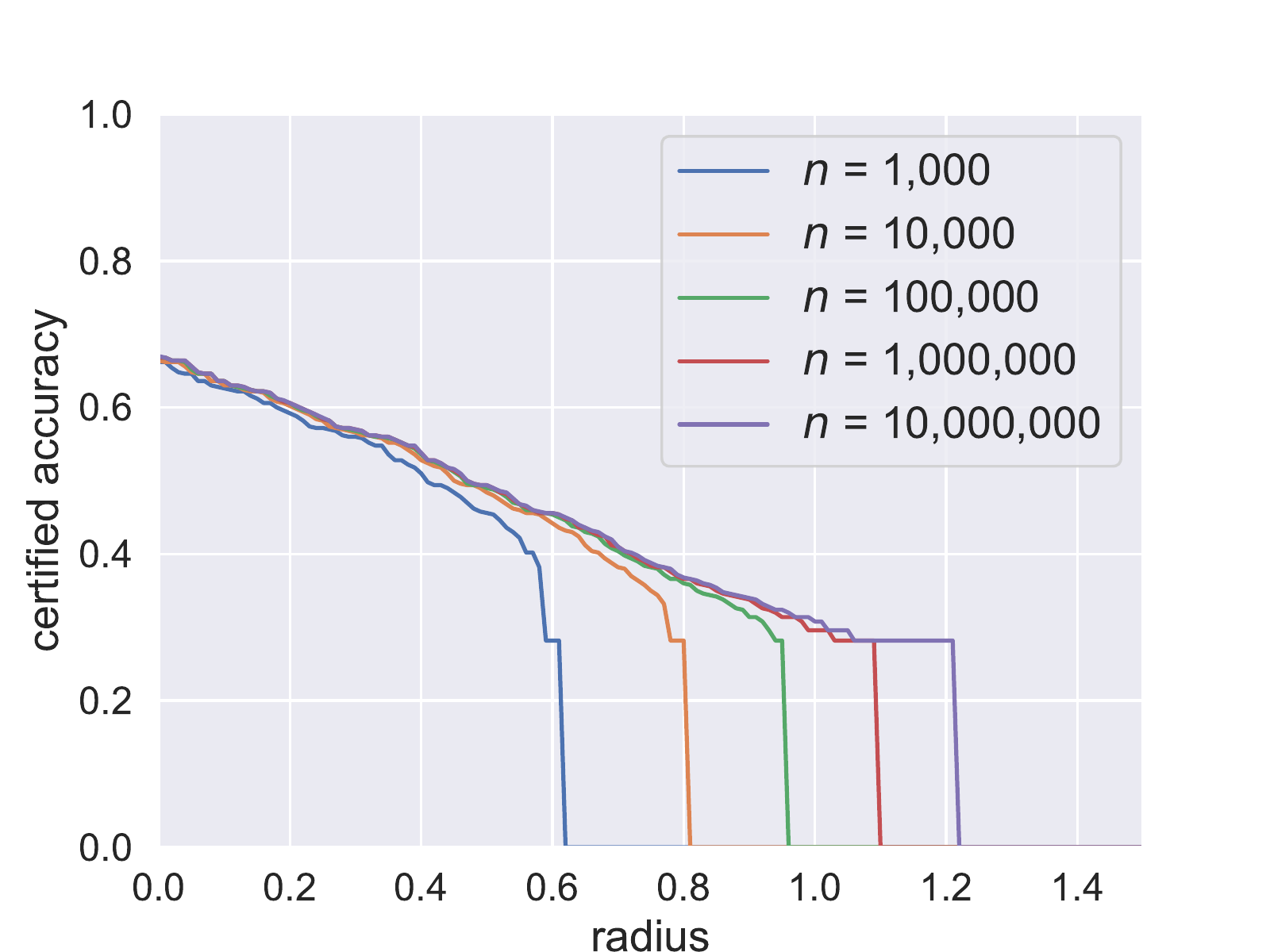}
\end{subfigure}
\begin{subfigure}{0.33\textwidth}
	\includegraphics[width=\textwidth]{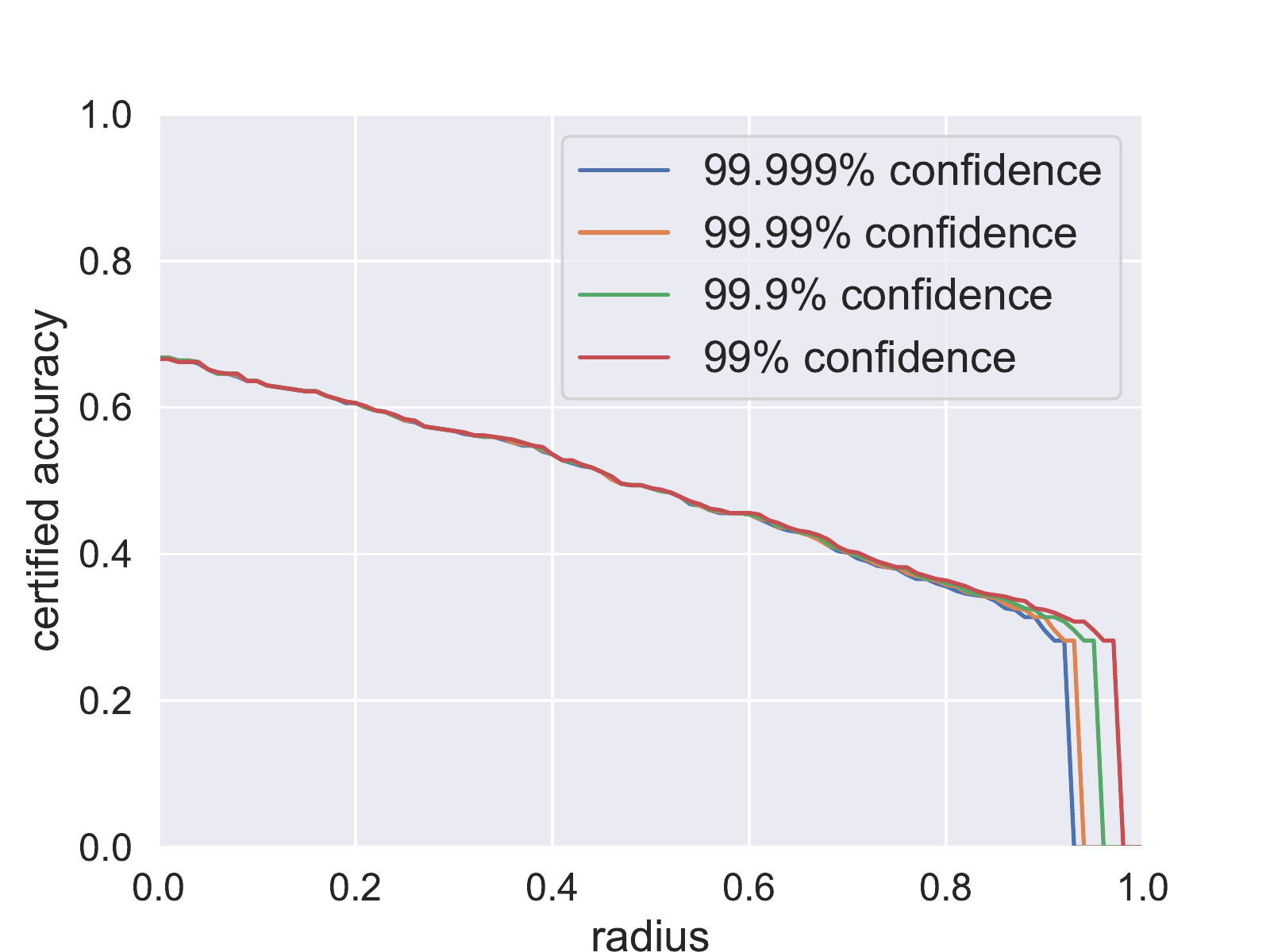}	
\end{subfigure}
\caption{Experiments with randomized smoothing on ImageNet with $\sigma=0.25$.  \textbf{Left}: certified accuracies obtained using our Theorem \ref{mainbound} versus those obtained using the robustness guarantees derived in prior work.  \textbf{Middle}: projections for the certified accuracy if the number of samples $n$ used by \textsc{Certify} had been larger or smaller.  \textbf{Right}: certified accuracy as the failure probability $\alpha$ of \textsc{Certify} is varied. }
\label{fig:ablations}
\end{center}
\end{figure*}

\paragraph{ImageNet and CIFAR-10 results}
We applied randomized smoothing to CIFAR-10 \citep{krizhevsky09learning} and ImageNet \citep{imagenetcvpr09}.
On each dataset we trained several smoothed classifiers, each with a different $\sigma$.
On CIFAR-10 our base classifier was a 110-layer residual network; certifying each example took 15 seconds on an NVIDIA RTX 2080 Ti.
On ImageNet our base classifier was a ResNet-50; certifying each example took 110 seconds.
We also trained a neural network with the base classifier's architecture on clean data, and subjected it to a DeepFool $\ell_2$ adversarial attack  \citep{moosavi2017deepfool}, in order to obtain an empirical upper bound on its robust accuracy.
We certified the full CIFAR-10 test set and a subsample of 500 examples from the ImageNet test set.

Figure \ref{fig:certified-accuracy} plots the certified accuracy attained by smoothing with each $\sigma$.
The dashed black line is the empirical upper bound on the robust accuracy of the base classifier architecture; observe that smoothing improves substantially upon the robustness of the undefended base classifier architecture.
We see that $\sigma$ controls a robustness/accuracy tradeoff.
When $\sigma$ is low, small radii can be certified with high accuracy, but large radii cannot be certified.
When $\sigma$ is high, larger radii can be certified, but smaller radii are certified at a lower accuracy.
This observation echoes the finding in  \citet{tsipras2018robustness} that adversarially trained networks with higher robust accuracy tend to have lower standard accuracy.
Tables of these results are in Appendix \ref{section:experimenttables}.

Figure \ref{fig:ablations} (\textbf{left}) plots the certified accuracy obtained using our Theorem \ref{mainbound} guarantee alongside the certified accuracy obtained using the analogous bounds of \citet{lecuyer2018certified} and \citet{li2018second}.  Since our expression for the certified radius $R$ is greater (and, in fact, tight), our bound delivers higher certified accuracies.
Figure \ref{fig:ablations} (\textbf{middle}) projects how the certified accuracy would have changed had \textsc{Certify} used more or fewer samples $n$ (under the assumption that the relative class proportions in \texttt{counts} would have remained constant).
Finally, Figure \ref{fig:ablations} (\textbf{right}) plots the certified accuracy as the confidence parameter $\alpha$ is varied.
Observe that the certified accuracy is not very sensitive to $\alpha$.

\paragraph{Comparison to baselines}
We compared randomized smoothing to three baseline approaches for certified $\ell_2$ robustness: the duality approach from \citet{wong2018scaling}, the Lipschitz approach from \citet{tsuzuku2018lipschitz}, and the approach from \citet{weng2018toward, zhang2018efficient}.
The strongest baseline was \citet{wong2018scaling}; we defer the comparison to the other two baselines to Appendix \ref{section:additionalexperiments}.

In Figure \ref{fig:compare-wong-fair}, we compare the largest publicly released model from \citet{wong2018scaling}, a small resnet, to two randomized smoothing classifiers: one which used the same small resnet architecture for its base classifier, and one which used a larger 110-layer resnet for its base classifier.
First, observe that smoothing with the large 110-layer resnet substantially outperforms the baseline (across all hyperparameter settings) at all radii.
Second, observe that smoothing with the small resnet also outperformed the method of \citet{wong2018scaling} at all but the smallest radii.
We attribute this latter result to the fact that neural networks trained using the method of \citet{wong2018scaling} are ``typically overregularized to the point that many filters/weights become identically zero,'' per that paper.
In contrast, the base classifier in randomized smoothing is a fully expressive neural network.

\paragraph{Prediction}
It is computationally expensive to certify the robustness of $g$ around a point $x$, since the value of $n$ in \textsc{Certify} must be very large.
However, it is far cheaper to evaluate $g$ at $x$ using \textsc{Predict}, since $n$ can be small.
For example, when we ran \textsc{Predict} on ImageNet ($\sigma = 0.25$) using $n=$ 100, making each prediction only took 0.15 seconds, and we attained a top-1 test accuracy of 65\% (Appendix \ref{section:experimenttables}).  

As discussed earlier, an adversary can potentially force \textsc{Predict} to abstain with high probability.
However, it is relatively rare for \textsc{Predict} to abstain on the actual data distribution.
On ImageNet  ($\sigma = 0.25$), \textsc{Predict} with failure probability $\alpha = 0.001$ abstained 12\% of the time when $n=$ 100, 4\% when $n=$ 1000, and 1\% when $n=$ 10,000.

\paragraph{Empirical tightness of bound}
When $f$ is linear, there always exists a class-changing perturbation just beyond the certified radius.
Since neural networks are not linear, we empirically assessed the tightness of our bound by subjecting an ImageNet smoothed classifier ($\sigma = 0.25$) to a projected gradient descent-style adversarial attack (Appendix  \ref{section:experimentdetails:attacks}).
For each example, we ran \textsc{Certify} with $\alpha = 0.01$, and, if the example was correctly classified and certified robust at radius $R$, we tried finding an adversarial example for $g$ within radius $1.5 R$ and within radius $2R$.
We succeeded 17\% of the time at radius $1.5R$ and 53\% of the time at radius $2R$.

\section{Conclusion}
Theorem \ref{mainboundtight} establishes that smoothing with Gaussian noise naturally confers adversarial robustness in $\ell_2$ norm: if we have no knowledge about the base classifier beyond the distribution of $f(x+\varepsilon)$, then the set of perturbations to which the smoothed classifier is provably robust is precisely an $\ell_2$ ball.
We suspect that smoothing with other noise distributions may lead to similarly natural robustness guarantees for other perturbation sets such as general $\ell_p$ norm balls.

Our strong empirical results suggest that randomized smoothing is a promising direction for future research into adversarially robust classification.
Many empirical approaches have been ``broken,'' and provable approaches based on certifying neural network classifiers have not been shown to scale to networks of modern size.
It seems to be computationally infeasible to reason in any sophisticated way about the decision boundaries of a large, expressive neural network.
Randomized smoothing circumvents this problem: the smoothed classifier is not itself a neural network, though it leverages the discriminative ability of a neural network base classifier.
To make the smoothed classifier robust, one need simply make the base classifier classify well under noise.
In this way, randomized smoothing reduces the unsolved problem of adversarially robust classification to the comparably solved domain of supervised learning.

\newpage

\section{Acknowledgements}

We thank Mateusz Kwa\'snicki for help with Lemma \ref{np-gaussian} in the appendix, Aaditya Ramdas for pointing us toward the work of \citet{hung2017rank}, and Siva Balakrishnan for helpful discussions regarding the confidence interval in Appendix \ref{section:highprobability}.
We thank Tolani Olarinre, Adarsh Prasad, Ben Cousins, Ramon Van Handel, Matthias Lecuyer, and Bai Li for useful conversations.
Finally, we are very grateful to Vaishnavh Nagarajan, Arun Sai Suggala, Shaojie Bai, Mikhail Khodak, Han Zhao, and Zachary Lipton for reviewing drafts of this work.
Jeremy Cohen is supported by a grant from the Bosch Center for AI.

\bibliography{paper}
\bibliographystyle{icml2019}

\newpage

\onecolumn

\appendix

\section{Proofs of Theorems \ref{mainbound} and \ref{mainboundtight}}
\label{section:fullproof}

Here we provide the complete proofs for Theorem \ref{mainbound} and Theorem \ref{mainboundtight}.
We fist prove the following lemma, which is essentially a restatement of the Neyman-Pearson lemma \citep{neyman1933} from statistical hypothesis testing.

\begin{lemma}[\textbf{Neyman-Pearson}]
\label{general-np}
Let $X$ and $Y$ be random variables in $\mathbb{R}^d$ with densities $\mu_X$ and $\mu_Y$.
Let $h: \mathbb{R}^d \to \{0, 1\}$ be a random or deterministic function.
Then:
\begin{enumerate}
\item If $S = \left\{z \in \mathbb{R}^d: \frac{\mu_Y(z)}{\mu_X(z)} \le t \right\}$ for some $t > 0$ and $\mathbb{P}(h(X) = 1) \ge \mathbb{P}(X \in S)$, then $\mathbb{P}(h(Y) = 1) \ge \mathbb{P}(Y \in S)$.
\item If $S = \left\{z \in \mathbb{R}^d: \frac{\mu_Y(z)}{\mu_X(z)} \ge t \right\}$ for some $t > 0$ and $\mathbb{P}(h(X) = 1) \le \mathbb{P}(X \in S)$, then  $\mathbb{P}(h(Y) = 1) \le \mathbb{P}(Y \in S)$.
\end{enumerate}
\end{lemma}

\begin{proof}
Without loss of generality, we assume that $h$ is random and write $h(1|x)$ for the probability that $h(x)=1$.

First we prove part 1.  We denote the complement of $S$ as $S^c$.
\begin{align*}
\mathbb{P}(h(Y) = 1) - \mathbb{P}(Y \in S) &= \int_{\mathbb{R}^d} h(1|z) \, \mu_Y(z) dz  - \int_{S} \mu_Y(z) dz \\
&= \left[ \int_{S^c} h(1|z) \mu_Y(z) dz  + \int_S h(1|z) \mu_Y(z) dz  \right] - \left[ \int_S h(1|z) \mu_Y(z) dz + \int_S h(0|z) \mu_Y(z) dz \right] \\
&= \int_{S^c} h(1|z) \mu_Y(z) dz  -  \int_S h(0|z) \mu_Y(z) dz \\
&\ge t \left[ \int_{S^c} h(1|z) \mu_X(z) dz - \int_S h(0|z) \mu_X(z) \right] \\
&= t \left[ \int_{S^c} h(1|z) \mu_X(z) dz + \int_{S} h(1|z) \mu_X(z) dz - \int_S h(1|z) \mu_X(z) dz  - \int_S h(0|z) \mu_X(z) \right] \\
&= t \left[ \int_{\mathbb{R}^d} h(1|z) \mu_X(z) dz - \int_S \mu_X(z) dz \right] \\
&= t \left[ \mathbb{P}(h(X) = 1) - \mathbb{P}(X \in S) \right] \\
&\ge 0
\end{align*}

The inequality in the middle is due to the fact that $\mu_Y(z) \le t \, \mu_X(z) \; \forall z \in S$ and $\mu_Y(z) > t \, \mu_X(z) \; \forall z \in S^c$. 
The inequality at the end is because both terms in the product are non-negative by assumption.

The proof for part 2 is virtually identical, except both ``$\ge$'' become ``$\le$.''
\end{proof}

\paragraph{Remark: connection to statistical hypothesis testing.}
Part 2 of Lemma \ref{general-np} is known in the field of statistical hypothesis testing as the Neyman-Pearson Lemma \citep{neyman1933}.
The hypothesis testing problem is this: we are given a sample that comes from one of two distributions over $\mathbb{R}^d$: either the null distribution $X$ or the alternative distribution $Y$.  We would like to identify which distribution the sample came from.
It is worse to say ``$Y$'' when the true answer is ``$X$'' than to say ``$X$'' when the true answer is ``$Y$.''
Therefore we seek a (potentially randomized) procedure $h: \mathbb{R}^d \to \{0, 1\}$ which returns ``$Y$'' when the sample really came from $X$ with probability no greater than some failure rate $\alpha$.
In particular, out of all such rules $h$, we would like the \textit{uniformly most powerful} one $h^*$, i.e. the rule which is most likely to correctly say ``$Y$'' when the sample really came from $Y$.
\citet{neyman1933} showed that $h^*$ is the rule which returns ``$Y$'' deterministically on the set $S^* = \{z \in \mathbb{R}^d: \frac{\mu_Y(z)}{\mu_X(z)} \ge t \}$ for whichever $t$ makes $\mathbb{P}(X \in S^*) = \alpha$.
In other words, to state this in a form that looks like Part 2 of Lemma \ref{general-np}: if $h$ is a different rule with $\mathbb{P}(h(X) = 1) \le \alpha$, then $h^*$ is more powerful than $h$, i.e. $\mathbb{P}(h(Y) = 1) \le \mathbb{P}(Y \in S^*)$.

Now we state the special case of Lemma \ref{general-np} for when $X$ and $Y$ are isotropic Gaussians.

\begin{lemma}[\textbf{Neyman-Pearson for Gaussians with different means}]
\label{np-gaussian}
Let $X \sim \mathcal{N}(x, \sigma^2 I)$ and $Y \sim \mathcal{N}(x+\delta, \sigma^2 I)$.
Let $h: \mathbb{R}^d \to \{0, 1\}$ be any deterministic or random function.
Then:
\end{lemma}
\begin{enumerate}
\item If $S = \left \{z \in \mathbb{R}^d: \delta^T z \le \beta \right \}$ for some $\beta$ and  $\mathbb{P}(h(X) = 1) \ge \mathbb{P}(X \in S)$, then $\mathbb{P}(h(Y) = 1) \ge \mathbb{P}(Y \in S)$
\item If $S = \left \{z \in \mathbb{R}^d: \delta^T z \ge \beta \right \}$ for some $\beta$ and $\mathbb{P}(h(X) = 1) \le \mathbb{P}(X \in S)$, then $\mathbb{P}(h(Y) = 1) \le \mathbb{P}(Y \in S)$
\end{enumerate}

\begin{proof}
This lemma is the special case of Lemma \ref{general-np} when $X$ and $Y$ are isotropic Gaussians with means $x$ and $x+\delta$.

By Lemma \ref{general-np} it suffices to simply show that for any $\beta$, there is some $t > 0$ for which:
\begin{align}
\{z: \delta^T z \le \beta \} =  \left\{z: \frac{\mu_Y(z)}{\mu_X(z)} \le t \right\} \quad \text{and} \quad \{z: \delta^T z \ge \beta \} =  \left\{z: \frac{\mu_Y(z)}{\mu_X(z)} \ge t \right\} 
\end{align}

The likelihood ratio for this choice of $X$ and $Y$ turns out to be:
\begin{align*}
\frac{\mu_Y(z)}{\mu_X(z)} &= \frac{\exp \left(-\frac{1}{2 \sigma^2} \sum_{i=1}^d (z_i - (x_i + \delta_i))^2) \right)}{ \exp \left(- \frac{1}{2 \sigma^2} \sum_{i=1}^d (z_i - x_i)^2 \right) } \\
&= \exp \left ( \frac{1}{2 \sigma^2} \sum_{i=1}^d 2 z_i \delta_i - \delta_i^2 - 2 x_i \delta_i \right) \\
&= \exp(a \delta^T z + b)
\end{align*}
where $a > 0$ and $b$ are constants w.r.t $z$, specifically $a = \frac{1}{\sigma^2}$ and $b = \frac{-(2\delta^T x + \|\delta\|^2)}{2 \sigma^2}$.

Therefore, given any $\beta$ we may take $t = \exp(a \beta + b)$, noticing that 
\begin{align*}
\delta^T z \le \beta &\iff \exp(a\delta^T z + b) \le t \\
\delta^T z \ge \beta &\iff \exp(a\delta^T z + b) \ge t
\end{align*}

\end{proof}

Finally, we prove Theorem \ref{mainbound} and Theorem \ref{mainboundtight}.

\textbf{Theorem \ref{mainbound} (restated).}
\textit{
Let $f: \mathbb{R}^d \to \mathcal{Y}$ be any deterministic or random function.  Let $\varepsilon \sim \mathcal{N}(0, \sigma^2 I)$.
Let $g(x) = \argmax_c \mathbb{P}(f(x+\varepsilon) = c)$.
Suppose that for a specific $x \in \mathbb{R}^d$, there exist $c_A \in \mathcal{Y}$ and $\underline{p_A}, \overline{p_B} \in [0, 1]$ such that:
\small
\begin{align}
\label{knownfacts2}
\mathbb{P}(f(x+\varepsilon) = c_A) \ge \underline{p_A} \ge \overline{p_B }\ge \max_{c \neq c_A} \mathbb{P}(f(x+ \varepsilon) = c)
\end{align}
\normalsize
Then $g(x+\delta) = c_A$ for all $\|\delta\|_2 < R$, where
\begin{align}
R = \frac{\sigma}{2} (\Phi^{-1}(\underline{p_A}) - \Phi^{-1}(\overline{p_B}))
\end{align}}

\begin{proof}
To show that $g(x+\delta) = c_A$, it follows from the definition of $g$ that we need to show that
\begin{align*}
\mathbb{P}(f(x+\delta+\varepsilon) = c_A) > \max_{c_B  \neq c_A} \mathbb{P}(f(x+\delta+\varepsilon) = c_B)
\end{align*}
We will prove that $\mathbb{P}(f(x+\delta+\varepsilon) = c_A)  > \mathbb{P}(f(x+\delta+\varepsilon) = c_B)$ for every class $c_B \neq c_A$.  Fix one such class $c_B$ without loss of generality.

For brevity, define the random variables
\begin{align*}
X &:= x + \varepsilon = \mathcal{N}(x, \sigma^2 I) \\
Y &:= x+\delta+ \varepsilon = \mathcal{N}(x + \delta, \sigma^2 I)
\end{align*}
In this notation, we know from (\ref{knownfacts2}) that
\begin{align}
\label{foo}
\mathbb{P}(f(X) = c_A) \ge \underline{p_A} \quad \text{and} \quad \mathbb{P}(f(X) = c_B) \le \overline{p_B}
\end{align}
and our goal is to show that
\begin{align}
\label{desired}
\mathbb{P}(f(Y) = c_A) > \mathbb{P}(f(Y) = c_B)
\end{align}

Define the half-spaces:
\begin{align*}
A &:= \{z:  \delta^T (z - x) \le \sigma \|\delta\| \Phi^{-1}(\underline{p_A}) \} \\
B &:= \{z: \delta^T (z - x) \ge \sigma \|\delta\| \Phi^{-1}(1 - \overline{p_B}) \}
\end{align*}
Algebra (deferred to the end) shows that $\mathbb{P}(X \in A) = \underline{p_A}$.
Therefore, by (\ref{foo}) we know that 
$ \mathbb{P}(f(X) = c_A) \ge \mathbb{P}(X \in A)$.
Hence we may apply Lemma \ref{np-gaussian} with $h(z) := \mathbf{1}[f(z) = c_A]$ to conclude:
\begin{align}
\mathbb{P}(f(Y) = c_A) &\ge \mathbb{P}(Y \in A) \label{baz1}
\end{align}
Similarly, algebra shows that $\mathbb{P}(X \in B) = \overline{p_B}$.
Therefore, by (\ref{foo}) we know that $ \mathbb{P}(f(X) = c_B) \le \mathbb{P}(X \in B)$.
Hence we may apply Lemma \ref{np-gaussian} with $h(z) := \mathbf{1}[f(z)= c_B]$ to conclude:
\begin{align}
\mathbb{P}(f(Y) = c_B) &\le \mathbb{P}(Y \in B) \label{baz2}
\end{align}

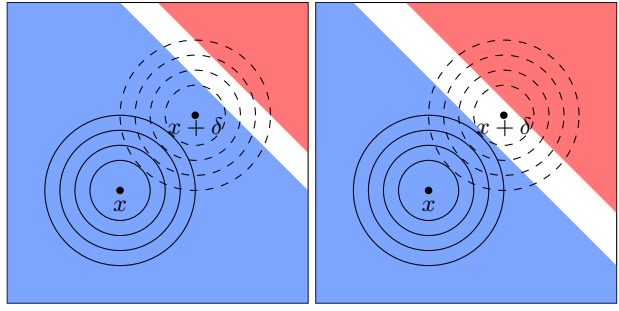
\begin{figure}[t]
\vskip 0.2in
\begin{center}
\begin{tikzpicture}[scale=1]
\draw (0,0) rectangle (4, 4);
\fill[blue, fill opacity=0.6] (0,0) --(0,4) -- (1.5, 4) -- (4, 1.5) -- (4, 0);
\fill[red, fill opacity=0.6] (2,4) --(4,4) -- (4, 2);
\fill[black] (2.5,2.5) circle (.05);
\draw[dashed] (2.5,2.5) circle (0.4);
\draw[dashed] (2.5,2.5) circle (0.6);
\draw[dashed] (2.5,2.5) circle (0.8);
\draw[dashed] (2.5,2.5) circle (1.0);
\fill[black] (1.5,1.5) circle (.05);
\draw[solid] (1.5,1.5) circle (0.4);
\draw[solid] (1.5,1.5) circle (0.6);
\draw[solid] (1.5,1.5) circle (0.8);
\draw[solid] (1.5,1.5) circle (1.0);
\node at (2.5, 2.35) {\small $x+\delta$};
\node at (1.5, 1.3) {\small $x$};
\end{tikzpicture}
\begin{tikzpicture}[scale=1]
\draw (0,0) rectangle (4, 4);
\fill[blue, fill opacity=0.6] (0,0) --(0,4) -- (0.5, 4) -- (4, 0.5) -- (4, 0);
\fill[red, fill opacity=0.6] (1.2, 4) --(4, 4) -- (4, 1.2);
\fill[black] (2.5,2.5) circle (.05);
\draw[dashed] (2.5,2.5) circle (0.4);
\draw[dashed] (2.5,2.5) circle (0.6);
\draw[dashed] (2.5,2.5) circle (0.8);
\draw[dashed] (2.5,2.5) circle (1.0);
\fill[black] (1.5,1.5) circle (.05);
\draw[solid] (1.5,1.5) circle (0.4);
\draw[solid] (1.5,1.5) circle (0.6);
\draw[solid] (1.5,1.5) circle (0.8);
\draw[solid] (1.5,1.5) circle (1.0);
\node at (2.5, 2.35) {\small $x+\delta$};
\node at (1.5, 1.3) {\small $x$};
\end{tikzpicture}

\caption{Illustration of the proof of Theorem \ref{mainbound}. 
The solid line concentric circles are the density level sets of $X:=x+\varepsilon$; the dashed line concentric circles are the level sets of  $Y:=x+\delta+\varepsilon$.
The set $A$ is in blue and the set $B$ is in red.
The figure on the left depicts a situation where $\mathbb{P}(Y \in A) > \mathbb{P}(Y \in B)$, and hence $g(x+\delta)$ may equal $c_A$. The figure on the right depicts a situation where $\mathbb{P}(Y \in A) < \mathbb{P}(Y \in B)$ and hence $g(x+\delta) \neq c_A$.}
\label{illustration}
\end{center}
\end{figure}

To guarantee (\ref{desired}), we see from (\ref{baz1},  \ref{baz2}) that it suffices to show that $\mathbb{P}(Y \in A) > \mathbb{P}(Y \in B)$, as this step completes the chain of inequalities
\begin{align}
\mathbb{P}(f(Y) = c_A) \ge  \mathbb{P}(Y \in A) > \mathbb{P}(Y \in B) \ge \mathbb{P}(f(Y) = c_B)
\end{align}

We can compute the following:
\begin{align}
\mathbb{P}(Y \in A) &= \Phi \left( \Phi^{-1}(\underline{p_A}) - \frac{\|\delta\|}{\sigma} \right) \label{eq:explicit-y-a} \\
\mathbb{P}(Y \in B) &= \Phi \left( \Phi^{-1}(\overline{p_B}) + \frac{\|\delta\|}{\sigma} \right) \label{eq:explicit-y-b}
\end{align}

Finally, algebra shows that $\mathbb{P}(Y \in A) > \mathbb{P}(Y \in B)$ if and only if:
\begin{align}
\|\delta\| < \frac{\sigma}{2}(\Phi^{-1}(\underline{p_A}) - \Phi^{-1}(\overline{p_B})) \label{yieldscondition}
\end{align}
which recovers the theorem statement.
\end{proof}

We now restate and prove Theorem \ref{mainboundtight}, which shows that the bound in Theorem \ref{mainbound} is tight.
The assumption below in Theorem \ref{mainboundtight} that $\underline{p_A} + \overline{p_B} \le 1$ is mild: given any $\underline{p_A}$ and $\overline{p_B}$ which do not satisfy this condition, one could have always redefined $\overline{p_B} \leftarrow 1 - \underline{p_A}$ to obtain a Theorem \ref{mainbound} guarantee with a larger certified radius, so there is no reason to invoke Theorem \ref{mainbound} unless $\underline{p_A} + \overline{p_B} \le 1$.

\textbf{Theorem 2 (restated).}  \textit{Assune $\underline{p_A} + \overline{p_B} \le 1$. 
For any perturbation $\delta \in \mathbb{R}^d$ with $\|\delta\|_2 > R$, there exists a base classifier $f^*$ consistent with the observed class probabilities (\ref{knownfacts2}) such that if $f^*$ is the base classifier for $g$, then $g(x+\delta) \neq c_A$.}

\begin{proof}
We re-use notation from the preceding proof.

Pick any class $c_B$ arbitrarily.  Define $A$ and $B$ as above, and consider the function
\begin{align*}
f^*(x) := \begin{cases}
c_A &\mbox{ if } x \in A \\
c_B &\mbox{ if } x \in B \\
\text{other classes} &\mbox{ otherwise}
\end{cases}
\end{align*}
This function is well-defined, since $A \cap B = \emptyset$ provided that $\underline{p_A} + \overline{p_B} \le 1$.

By construction, the function $f^*$ satisfies (\ref{knownfacts2}) with equalities, since
\begin{align*}
\mathbb{P}(f^*(x+\varepsilon) = c_A) = \mathbb{P}(X \in A) = \underline{p_A} \quad \quad \quad \mathbb{P}(f^*(x+\varepsilon) = c_B)  = \mathbb{P}(X \in B) = \overline{p_B}
\end{align*}
It follows from (\ref{eq:explicit-y-a}) and (\ref{eq:explicit-y-b}) that 
\begin{align*}
\mathbb{P}(Y \in A) < \mathbb{P}(Y \in B) \iff \|\delta\|_2 > R
\end{align*}
By assumption, $\|\delta\|_2 > R$, so $\mathbb{P}(Y \in A) < \mathbb{P}(Y \in B)$, or equivalently,
\begin{align*}
\mathbb{P}(f^*(x+\delta+\varepsilon) = c_A) < \mathbb{P}(f^*(x+\delta+\varepsilon) = c_B)
\end{align*}
Therefore, if $f^*$ is the base classifier for $g$, then $g(x+\delta) \neq c_A$.
\end{proof}

\subsubsection{Deferred Algebra}

\begin{claim} $\mathbb{P}(X \in A) = \underline{p_A}$
\end{claim}

\begin{proof}
Recall that $X \sim \mathcal{N}(x, \sigma^2 I)$ and $A = \{z:  \delta^T (z - x) \le \sigma \|\delta\| \Phi^{-1}(\underline{p_A}) \}$.
\begin{align*}
\mathbb{P}(X \in A) &= \mathbb{P}(\delta^T (X - x) \le \sigma \|\delta\| \Phi^{-1}(\underline{p_A})) \\
&= \mathbb{P}(\delta^T \mathcal{N}(0, \sigma^2 I)  \le \sigma \|\delta\| \Phi^{-1}(\underline{p_A})) \\
&= \mathbb{P}(\sigma \|\delta\| Z \le \sigma \|\delta\| \Phi^{-1}(\underline{p_A}) ) \tag{$Z \sim \mathcal{N}(0, 1)$} \\
&= \Phi(\Phi^{-1}(\underline{p_A})) \\
&= \underline{p_A}
\end{align*}
\end{proof}

\begin{claim} $\mathbb{P}(X \in B) = \overline{p_B}$
\end{claim}

\begin{proof}
Recall that $X \sim \mathcal{N}(x, \sigma^2 I)$ and $B = \{z:  \delta^T (z - x) \le \sigma \|\delta\| \Phi^{-1}(1 - \overline{p_B}) \}$.
\begin{align*}
\mathbb{P}(X \in A) &= \mathbb{P}(\delta^T (X - x) \ge \sigma \|\delta\| \Phi^{-1}(1 - \overline{p_B})) \\
&= \mathbb{P}(\delta^T \mathcal{N}(0, \sigma^2 I)  \ge \sigma \|\delta\| \Phi^{-1}(1 - \overline{p_B}))  \\
&= \mathbb{P}(\sigma \|\delta\| Z \ge \sigma \|\delta\| \Phi^{-1}(1 - \overline{p_B})) \tag{$Z \sim \mathcal{N}(0, 1)$} \\
&= \mathbb{P}(Z \ge \Phi^{-1}(1 - \overline{p_B})) \\
&= 1 -  \Phi( \Phi^{-1}(1 - \overline{p_B})) \\
&= \overline{p_B}
\end{align*}
\end{proof}

\begin{claim} $\mathbb{P}(Y \in A) = \Phi \left( \Phi^{-1}(\underline{p_A}) - \frac{\|\delta\|}{\sigma} \right) $
\end{claim}

\begin{proof}
Recall that $Y \sim \mathcal{N}(x+\delta, \sigma^2 I)$ and $A = \{z:  \delta^T (z - x) \le \sigma \|\delta\| \Phi^{-1}(\underline{p_A}) \}$.
\begin{align*}
\mathbb{P}(Y \in A) &= \mathbb{P}(\delta^T(Y-x) \le \sigma \|\delta\| \Phi^{-1}(\underline{p_A})) \\
&= \mathbb{P}(\delta^T \mathcal{N}(0, \sigma^2 I) + \|\delta\|^2 \le \sigma \|\delta\| \Phi^{-1}(\underline{p_A})) \\
&= \mathbb{P}(\sigma \|\delta\| Z \le \sigma \|\delta\| \Phi^{-1}(\underline{p_A}) - \|\delta\|^2 ) \tag{$Z \sim \mathcal{N}(0, 1)$} \\
&= \mathbb{P} \left(Z \le \Phi^{-1}(\underline{p_A}) - \frac{\|\delta\|}{\sigma} \right) \\
&= \Phi \left( \Phi^{-1}(\underline{p_A}) - \frac{\|\delta\|}{\sigma} \right)
\end{align*}
\end{proof}

\begin{claim} $\mathbb{P}(Y \in B) = \Phi \left( \Phi^{-1}(\overline{p_B}) + \frac{\|\delta\|}{\sigma} \right) $
\end{claim}

\begin{proof}
Recall that $Y \sim \mathcal{N}(x+\delta, \sigma^2 I)$ and $B = \{z:  \delta^T (z - x) \ge \sigma \|\delta\| \Phi^{-1}(1 - \overline{p_B}) \}$.
\begin{align*}
\mathbb{P}(Y \in B) &= \mathbb{P}(\delta^T(Y-x) \ge \sigma \|\delta\| \Phi^{-1}(1 - \overline{p_B})) \\
&= \mathbb{P}(\delta^T \mathcal{N}(0, \sigma^2 I) + \|\delta\|^2 \ge \sigma \|\delta\| \Phi^{-1}(1 - \overline{p_B})) \\
&= \mathbb{P}(\sigma \| \delta \| Z + \|\delta\|^2 \ge \sigma \|\delta\| \Phi^{-1}(1 - \overline{p_B})) \tag{$Z \sim \mathcal{N}(0, 1)$}  \\
&= \mathbb{P} \left(Z \ge \Phi^{-1}(1 - \overline{p_B}) - \frac{\|\delta\|}{\sigma} \right) \\
&= \mathbb{P} \left(Z \le \Phi^{-1}( \overline{p_B}) +\frac{\|\delta\|}{\sigma} \right) \\
&= \Phi \left( \Phi^{-1}(\overline{p_B}) + \frac{\|\delta\|}{\sigma} \right)
\end{align*}
\end{proof}

\newpage
\section{Smoothing a two-class linear classifier}
\label{section:otherdeferred}
In this appendix, we analyze what happens when the base classifier $f$ is a two-class linear classifier $f(x) = \text{sign}(w^T x + b)$.
To match the definition of $g$, we take $\text{sign}(\cdot)$ to be undefined when its argument is zero.

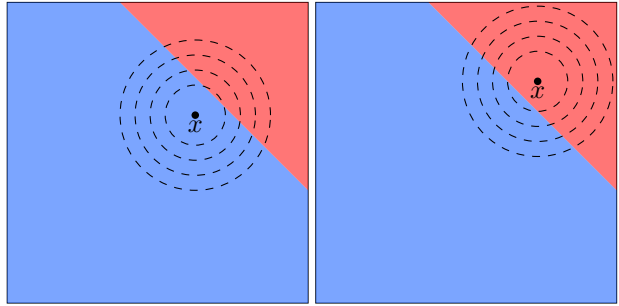
\begin{figure}[h]
\vskip 0.2in
\begin{center}
\begin{tikzpicture}[scale=1]
\draw (0,0) rectangle (4, 4);
\fill[blue, fill opacity=0.6] (0,0) --(0,4) -- (1.5, 4) -- (4, 1.5) -- (4, 0);
\fill[red, fill opacity=0.6] (1.5,4) --(4,4) -- (4, 1.5);
\fill[black] (2.5,2.5) circle (.05);
\draw[dashed] (2.5,2.5) circle (0.4);
\draw[dashed] (2.5,2.5) circle (0.6);
\draw[dashed] (2.5,2.5) circle (0.8);
\draw[dashed] (2.5,2.5) circle (1.0);
\node at (2.5, 2.35) {$x$};
\end{tikzpicture}
\begin{tikzpicture}[scale=1]
\draw (0,0) rectangle (4, 4);
\fill[blue, fill opacity=0.6] (0,0) --(0,4) -- (1.5, 4) -- (4, 1.5) -- (4, 0);
\fill[red, fill opacity=0.6] (1.5,4) --(4,4) -- (4, 1.5);
\fill[black] (2.95,2.95) circle (.05);
\draw[dashed] (2.95,2.95) circle (0.4);
\draw[dashed] (2.95,2.95) circle (0.6); d
\draw[dashed] (2.95,2.95) circle (0.8);
\draw[dashed] (2.95,2.95) circle (1.0);
\node at (2.95, 2.80) {$x$};
\end{tikzpicture}
\caption{Illustration of Proposition \ref{prop:g-is-f}. A binary linear classifier $f(x) = \text{sign}(w^T x + b)$ partitions $\mathbb{R}^d$ into two half-spaces, drawn here in blue and red.
An isotropic Gaussian $\mathcal{N}(x, \sigma^2 I)$ will put more mass on whichever half-space its center $x$ lies in: in the figure on the left, $x$ is in the blue half-space and $\mathcal{N}(x, \sigma^2 I)$ puts more mass on the blue than on red.   In the figure on the right, $x$ is in the red half-space and $\mathcal{N}(x, \sigma^2 I)$ puts more mass on red than on blue. Since the smoothed classifier's prediction $g(x)$ is defined to be whichever half-space  $\mathcal{N}(x, \sigma^2 I)$ puts more mass in, and the base classifier's prediction $f(x)$ is defined to be whichever half-space $x$ is in, we have that $g(x) = f(x)$ for all $x$.}
\label{illustration}
\end{center}
\end{figure}

Our first result is that when $f$ is a two-class linear classifier, the smoothed classifier $g$ is identical to the base classifier $f$. 
\begin{proposition}
If $f$ is a two-class linear classifier $f(x) = \text{sign}(w^T x + b)$, and $g$ is the smoothed version of $f$ with any $\sigma$, then $g(x) = f(x)$ for any $x$ (where $f$ is defined).
\label{prop:g-is-f}
\end{proposition}

\begin{proof}
From the definition of $g$,
\begin{align*}
g(x) = 1 &\iff  \mathbb{P}_{\varepsilon}(f(x+\varepsilon) = 1) > \frac{1}{2} \tag{$\varepsilon \sim \mathcal{N}(0, \sigma^2 I)$} \\
&\iff \mathbb{P}_\varepsilon \left(\text{sign}(w^T(x+\varepsilon) + b) = 1\right)  > \frac{1}{2} \\
&\iff \mathbb{P}_\varepsilon \left( w^T x + w^T \varepsilon + b \ge 0 \right) > \frac{1}{2}  \\
&\iff \mathbb{P} \left( \sigma \|w\| Z \ge - w^T x - b \right) > \frac{1}{2} \tag{$Z \sim \mathcal{N}(0, 1)$} \\
&\iff  \mathbb{P} \left(Z \le \frac{w^T x + b}{\sigma \|w\|} \right) > \frac{1}{2} \\
&\iff \frac{w^T x + b}{\sigma \|w\|} > 0 \\ 
&\iff w^T x + b > 0 \\
&\iff f(x) = 1
\end{align*}
\end{proof}
A similar calculation shows that $g(x) = -1 \iff f(x)= -1$.

A two-class linear classifier  $f(x) = \text{sign}(w^T x + b)$ is already certifiable: the distance from any point $x$ to the decision boundary is $(w^T x + b) / \|w\|_2$, and no distance with $\ell_2$ norm strictly less than this distance can possibly change $f$'s prediction.
Let $g$ be a smoothed version of $f$.
By Proposition \ref{prop:g-is-f}, $g$ is identical to $f$, so it follows that $g$ is truly robust around any input $x$ within the $\ell_2$ radius $(w^T x + b) / \|w\|_2$.
We now show that Theorem \ref{mainbound} will certify this radius, rather than a smaller, over-conservative radius.
\begin{proposition}
If $f$ is a two-class linear classifier  $f(x) = \text{sign}(w^T x + b)$, and $g$ is the smoothed version of $f$ with any $\sigma$, then invoking Theorem \ref{mainbound} at any $x$ (where $f$ is defined) with $\underline{p_A} = p_A$ and $\overline{p_B} = p_B$ will yield the certified radius $R = \frac{|w^Tx + b|}{\|w\|}$.
\label{prop:binary-linear-radius}
\end{proposition}

\begin{proof}
In binary classification, $p_A = 1 - p_B$, so Theorem \ref{mainbound} returns $R = \sigma \Phi^{-1}(\underline{p_A})$.

We have:
\begin{align*}
p_A &= \mathbb{P}_{\varepsilon}(f(x+\varepsilon) = g(x)) \\
&= \mathbb{P}_{\varepsilon}(\text{sign}(w^T(x+\varepsilon) + b) = \text{sign}(w^Tx + b)) \tag{By Proposition \ref{prop:g-is-f}, $g(x) = f(x)$} \\
&= \mathbb{P}_{\varepsilon}(\text{sign}(w^T x + \sigma \|w\| Z + b) = \text{sign}(w^T x + b) )
\end{align*}
There are two cases: if $w^T x + b > 0$, then
\begin{align*}
p_A &= \mathbb{P}_{\varepsilon} ( w^T x + \sigma \|w\| Z + b  > 0) \\
&= \mathbb{P}_{\varepsilon}  \left (Z > \frac{- w^T x - b}{\sigma \|w\|} \right) \\
&=  \mathbb{P}_{\varepsilon}  \left (Z < \frac{w^T x+ b}{\sigma \|w\|} \right) \\
&= \Phi \left (\frac{w^T x+ b}{\sigma \|w\|} \right)
\end{align*}
On the other hand, if $w^T x + b < 0$, then
\begin{align*}
p_A &= \mathbb{P}_\varepsilon( w^T x + \sigma \|w\| Z + b < 0 ) \\
&= \mathbb{P}_\varepsilon \left( Z < \frac{-w^T x - b}{\sigma \|w\|}  \right) \\
&= \Phi \left( \frac{-w^T x - b}{\sigma \|w\|}  \right)
\end{align*}
In either case, we have:
\begin{align*}
p_A = \Phi \left(\frac{|w^T x + b|}{\sigma \|w\|}  \right)
\end{align*}

Therefore, the bound in Theorem 1 returns a radius of
\begin{align*}
R &= \sigma \Phi^{-1}(p_A) \\
&= \frac{|w^Tx + b|}{\|w\|}
\end{align*}
\end{proof}

The previous two propositions imply that when $f$ is a two-class linear classifier, the Theorem \ref{mainbound} bound is ``tight'' in the sense that there always exists a class-changing perturbation just beyond the certified radius.\footnote{Note that this is a different sense of ``tight'' than the sense in which Theorem \ref{mainboundtight} proves that Theorem \ref{mainbound} is tight.  
Theorem \ref{mainboundtight} proves that for any fixed perturbation $\delta$ outside the radius certified by Theorem \ref{mainbound}, there exists a base classifier $f$ for which $g(x+\delta) \neq g(x)$.
In contrast, Proposition \ref{prop:binary-linear-tight} proves that for any fixed binary linear base classifier $f$, there exists a perturbation $\delta$ just outside the radius certified by Theorem \ref{mainbound} for which $g(x+\delta) \neq g(x)$.}

\begin{proposition}
Let $f$ be a two-class linear classifier  $f(x) = \text{sign}(w^T x + b)$, let $g$ be the smoothed version of $f$ for some $\sigma$, let $x$ be any point (where $f$ is defined), and let $R$ be the radius certified around $x$ by Theorem \ref{mainbound}.  Then for any radius $r >  R$, there exists a perturbation $\delta$ with $\|\delta\|_2 = r$ for which $g(x+\delta) \neq g(x)$. 
\label{prop:binary-linear-tight}
\end{proposition}
\begin{proof}
By Proposition \ref{prop:g-is-f} it suffices to show that there exists some perturbation $\delta$ with $\|\delta\|_2 = r$ for which $f(x+\delta) \neq f(x)$.

By Proposition \ref{prop:binary-linear-radius}, we know that $R  = \frac{|w^T x + b|}{\|w\|_2}$.

If $w^T x + b > 0$, consider the perturbation $\delta = -\frac{w}{ \| w \|_2 } r$.
This perturbation satisfies $ \|\delta\|_2 = r$ and
\begin{align*}
w^T (x+\delta) + b &= w^T x + b + w^T \delta \\
&= w^T x + b - \|w\|_2 r \\
&< w^T x + b - \|w\|_2 R \\
&= w^T x + b - |w^T x + b| \\
&= w^T x + b - (w^T x + b) \\
&= 0
\end{align*}
implying that $f(x+\delta) = -1$.

Likewise, if $w^T x + b < 0$, then consider the perturbation $\delta = \frac{w}{ \| w \|_2 } r$.
This perturbation satisfies $\|\delta\|_2 = r$ and $f(x+\delta) =-1 $.

\end{proof}

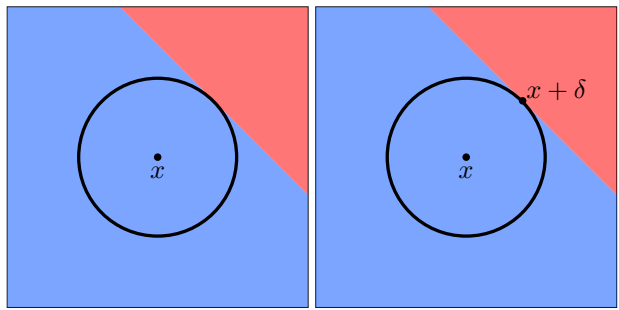
\begin{figure}[h]
\vskip 0.2in
\begin{center}
\begin{tikzpicture}[scale=1]
\draw (0,0) rectangle (4, 4);
\fill[blue, fill opacity=0.6] (0,0) --(0,4) -- (1.5, 4) -- (4, 1.5) -- (4, 0);
\fill[red, fill opacity=0.6] (1.5,4) --(4,4) -- (4, 1.5);
\fill[black] (2.0, 2.0) circle (.05);
\draw[solid, line width=1.25] (2.0, 2.0) circle (1.05);
\node at (2.0, 1.80) {$x$};
\end{tikzpicture}
\begin{tikzpicture}[scale=1]
\draw (0,0) rectangle (4, 4);
\fill[blue, fill opacity=0.6] (0,0) --(0,4) -- (1.5, 4) -- (4, 1.5) -- (4, 0);
\fill[red, fill opacity=0.6] (1.5,4) --(4,4) -- (4, 1.5);
\fill[black] (2.0, 2.0) circle (.05);
\fill[black] (2.75, 2.75) circle (.05);
\draw[solid, line width=1.25] (2.0, 2.0) circle (1.05);
\node at (2.0, 1.80) {$x$};
\node at (3.2, 2.9) {$x+\delta$};
\end{tikzpicture}
\caption{\textbf{Left}: Illustration of of Proposition \ref{prop:binary-linear-radius}. The red/blue half-spaces are the decision regions of both the base classifier $f$ and the smoothed classifier $g$.  (Since the base classifier is binary linear, $g = f$ everywhere.)  The black circle is the robustness radius $R$ certified by Theorem \ref{mainbound}. \textbf{Right}: Illustration of Proposition \ref{prop:binary-linear-tight}.  For any $r > R$, there exists a perturbation $\delta$ with $\|\delta\|_2 = r$ for which $g(x+\delta) \neq g(x)$. }
\label{illustration}
\end{center}
\end{figure}

This special property of two-class linear classifiers is not true in general.
In fact, it is possible to construct situations where $g$'s prediction around some point $x_0$ is robust at radius $\infty$, yet Theorem \ref{mainbound} only certifies a radius of $\tau$, where $\tau$ is arbitrarily close to zero.

\begin{proposition}
For any $\tau > 0$, there exists a base classifier $f$ and an input $x_0$ for which the corresponding smoothed classifier $g$ is robust around $x_0$ at radius $\infty$, yet Theorem \ref{mainbound} only certifies a radius of  $\tau$ around $x_0$.
\label{prop:not-tight}
\end{proposition}

\begin{proof}
Let $t = -\Phi^{-1}(\frac{1}{2} \Phi(\tau))$ and consider the following base classifier:
\begin{align*}
f(x) = \begin{cases}
1 &\mbox{ if } x < -t \\
-1 &\mbox{ if } -t \le x \le t  \\
1 &\mbox{ if } x > t
\end{cases}
\end{align*}
Let $g$ be the smoothed version of $f$ with $\sigma = 1$.
We will show that $g(x) = 1$ everywhere, implying that $g$'s prediction is robust around $x_0 = 0 $ with radius $\infty$.
Yet Theorem \ref{mainbound} only certifies a radius of $\tau$ around $x_0$.

Let $Z \sim \mathcal{N}(0, 1)$.
For any $x$, we have:
\begin{align*}
\mathbb{P}(f(x+\varepsilon) = -1) &= \mathbb{P}(-t \le x+\varepsilon \le t) \\
 &= \mathbb{P}[-t -x \le Z \le t-x] \\
&\le \mathbb{P}[-t \le Z \le t] \tag{apply Lemma \ref{lemma:gaussian-interval} below with $\ell = -t-x$} \\
&=  1 - 2\Phi(-t) \\
&= 1 - \Phi(\tau) \\
&< \frac{1}{2}.
\end{align*}

Therefore, $g(x) = 1$ for all $x$.

Meanwhile, at $x_0 =0$, we have:
\begin{align*}
\mathbb{P}(f(x_0+\varepsilon) = 1) &= \mathbb{P}(f(\varepsilon) = 1) \\
&= \mathbb{P}(Z < -t \text{ or } Z > t)\\
&= 2\Phi(-t)  \\
&= \Phi(\tau),
\end{align*}
so by Theorem \ref{mainbound}, the certified radius around $x_0$ is $R= \tau$.

\end{proof}

The proof of Proposition \ref{prop:not-tight} employed the following lemma, which formalizes the visually obvious fact that out of all intervals of some fixed width $2t$, the interval with maximal mass under the standard normal distribution $Z$ is the interval $[-t, t]$.
\begin{lemma}
Let $Z \sim \mathcal{N}(0, 1)$.  For any $\ell \in \mathbb{R}$, $t > 0$, we have $\mathbb{P}(\ell \le Z \le \ell + 2t) \le \mathbb{P}(-t \le Z \le t)$.
\label{lemma:gaussian-interval}
\end{lemma}

\begin{proof}
Let $\phi$ be the PDF of the standard normal distribution.
Since $\phi$ is symmetric about the origin (i.e. $\phi(x) = \phi(-x) \; \forall x$), 
\begin{align*}
\mathbb{P}(-t \le Z \le t) = 2 \int_0^{t} \phi(x) dx.
\end{align*}

There are two cases to consider:

\textbf{Case 1:} The interval $[\ell, \ell + 2t]$ is entirely positive, i.e. $\ell \ge 0$, or $[\ell, \ell + 2t]$ is entirely negative, i.e. $\ell + 2t \le 0$.

First, we use the fact that $\phi$ is symmetric about the origin to rewrite $\mathbb{P}(\ell \le Z \le \ell + 2t)$ as the probability that $Z$ falls in a non-negative interval $[a, a+2t]$ for some $a$.

Specifically, if $\ell \ge 0$, then let $a = \ell$.  Else, if $\ell + 2t \le 0$, then let $a = -(\ell + 2t)$. 
We therefore have:
\begin{align*}
\mathbb{P}(\ell \le Z \le \ell + 2t) = \mathbb{P}(a \le Z \le a + 2t).
\end{align*}

Therefore:
\begin{align*}
\mathbb{P}(-t \le Z \le t) - \mathbb{P}(\ell \le Z \le \ell + 2t) &= \int_0^t \phi(x) dx - \int_a^{a+t} \phi(x) dx + \int_0^t \phi(x) dx - \int_{a+t}^{a+2t} \phi(x) dx \\
&= \int_a^{a+t} \phi(x-a) dx - \int_a^{a+t} \phi(x) dx + \int_{a+t}^{a+2t} \phi(x-a-t) dx - \int_{a+t}^{a+2t} \phi(x) dx \\
&= \int_a^{a+t} \left[ \phi(x-a)  - \phi(x) \right] dx + \int_{a+t}^{a+2t} \left[ \phi(x-a-t) - \phi(x) \right] dx \\
&\ge \int_a^{a+t} 0 \; dx + \int_{a+t}^{a+2t} 0 \; dx \\
&= 0
\end{align*}
where the inequality is because $\phi$ is monotonically decreasing on $[0, \infty)$.

\textbf{Case 2:} $I$ is partly positive, partly negative, i.e. $\ell < 0 < \ell + 2t$.

First, we use the fact that $\phi$ is symmetric about the origin to rewrite $\mathbb{P}(\ell \le Z \le \ell + 2t)$ as the sum of the probabilities that $Z$ falls in two non-negative intervals $[0, a]$ and $[0, b]$ for some $a, b$.

Specifically, let $a = \min(-\ell, \ell + 2t)$ and $b = \max(-\ell, \ell + 2t)$.
We therefore have:
\begin{align*}
\mathbb{P}(\ell \le Z \le \ell + 2t) = \mathbb{P}(0 \le Z \le  a) + \mathbb{P}(0 \le Z \le b).
\end{align*}

Note that by construction, $a+b = 2t$, and $0 \le a \le t$ and $t \le b \le 2t$.

We have:
\begin{align*}
\mathbb{P}(-t \le Z \le t) - \mathbb{P}(\ell \le Z \le \ell + 2t) &= \left( \int_0^t \phi(x) dx - \int_0^{a} \phi(x) dx \right] - \left[ \int_{0}^{b} \phi(x) dx - \int_0^t \phi(x) dx  \right) \\
&= \int_a^t \phi(x) dx - \int_t^b \phi(x) dx \\
&= \int_a^t \phi(x) dx - \int_t^{2t - a} \phi(x) dx \\
&= \int_a^t \phi(x) dx - \int_a^t \phi(x - a + t) dx \\
&= \int_a^t (\phi(x) - \phi(x-a+t)) dx  \\
& \ge \int_a^t 0 \, dx \\
&= 0
\end{align*}

where the inequality is again because $\phi$ is monotonically decreasing on $[0, \infty)$.
\end{proof}

\newpage
\section{Practical algorithms}
\label{section:practical-algorithms-appendix}

In this appendix, we elaborate on the prediction and certification algorithms described in Section \ref{sec:practical-algorithms}.
The pseudocode in Section \ref{sec:practical-algorithms} makes use of several helper functions: 
\begin{itemize}
\item 
 \textsc{SampleUnderNoise}($f$, $x$, num, $\sigma$) works as follows:
\begin{enumerate}
\item Draw num samples of noise, $\varepsilon_1 \hdots \varepsilon_{\text{num}} \sim \mathcal{N}(0, \sigma^2 I)$.
\item Run the noisy images through the base classifier $f$ to obtain the predictions $f(x+\varepsilon_1), \hdots, f(x+\varepsilon_{\text{num}})$.
\item Return the counts for each class, where the count for class $c$ is defined as $\sum_{i=1}^{\text{num}} \mathbf{1}[f(x+\varepsilon_i) = c]$.
\end{enumerate}
\item  \textsc{BinomPValue}($n_A$, $n_A + n_B$, $p$) returns the p-value of the two-sided hypothesis test that $n_A \sim \text{Binomial}(n_A + n_B, p)$.
Using \texttt{scipy.stats.binom\_test}, this can be implemented as: \texttt{binom\_test(nA, nA + nB, p)}.

\item
\textsc{LowerConfBound}($k$, $n$, $1 - \alpha$) returns a one-sided $(1 - \alpha)$ lower confidence interval for the Binomial parameter $p$ given that $k \sim \text{Binomial}(n, p)$.
In other words, it returns some number $\underline{p}$ for which $\underline{p} \le p$ with probability at least $1 - \alpha$ over the sampling of $k \sim \text{Binomial}(n, p)$.
Following \citet{lecuyer2018certified}, we chose to use the Clopper-Pearson confidence interval, which inverts the Binomial CDF \citep{clopper1934}.
Using \texttt{statsmodels.stats.proportion.proportion\_confint}, this can be implemented as 
\begin{verbatim}
proportion_confint(k, n, alpha=2*alpha, method="beta")[0]
\end{verbatim}

\end{itemize}

\subsection{Prediction}

The randomized algorithm given in pseudocode as \textsc{Predict} leverages the hypothesis test given in  \citet{hung2017rank} for identifying the top category of a multinomial distribution.
\textsc{Predict} has one tunable hyperparameter, $\alpha$.
When $\alpha$ is small, \textsc{Predict} abstains frequently but rarely returns the wrong class.
When $\alpha$ is large, \textsc{Predict} usually makes a prediction, but may often return the wrong class.

We now prove that with high probability, \textsc{Predict} will either return $g(x)$ or abstain.

\textbf{Proposition \ref{thm:predict-correct} (restated).}
\textit{
With probability at least $1 - \alpha$ over the randomness in \textsc{Predict}, \textsc{Predict} will either abstain or return $g(x)$.
(Equivalently: the probability that \textsc{Predict} returns a class other than $g(x)$ is at most $\alpha$.)
}

\begin{proof}
For notational convenience, define $p_c = \mathbb{P}(f(x+\varepsilon) = c)$.  Let $c_A = \max_c p_c$.
Notice that by definition, $g(x) = c_A$.

We can describe the randomized procedure \textsc{Predict} as follows:
\begin{enumerate}
\item Sample a vector of class counts $\{n_c\}_{c \in \mathcal{Y}}$ from  $\text{Multinomial}(\{p_c\}_{c \in \mathcal{Y}}, n)$.
\item Let $\hat{c}_A = \argmax_c n_c$ be the class whose count is largest. Let $n_A$ and $n_B$ be the largest count and the second-largest count, respectively.
\item If the p-value of the two-sided hypothesis test that $n_A$ is drawn from $\text{Binom}\left(n_A + n_B, \frac{1}{2} \right)$ is less than $\alpha$, then return $\hat{c}_A$.
Else, abstain.
\end{enumerate}

The quantities $c_A$ and the $p_c$'s are fixed but unknown, while the quantities $\hat{c}_A$, the $n_c$'s, $n_A$, and $n_B$ are random.

We'd like to prove that the probability that \textsc{Predict} returns a class other than $c_A$ is at most $\alpha$.
 \textsc{Predict} returns a class other than $c_A$ if and only if (1) $\hat{c}_A \neq c_A$ and (2) \textsc{Predict} does not abstain.


We have:
\begin{align*}
\mathbb{P}(\textsc{Predict} \textrm{ returns class } \neq c_A) &=  \mathbb{P}(\hat{c}_A \neq c_A, \textsc{Predict} \textrm{ does not abstain}) \\
&= \mathbb{P}(\hat{c}_A \neq c_A) \; \mathbb{P}(\textsc{Predict} \textrm{ does not abstain} | \hat{c}_A \neq c_A) \\
&\le  \mathbb{P}(\textsc{Predict} \textrm{ does not abstain} | \hat{c}_A \neq c_A)
\end{align*}

Recall that \textsc{Predict} does not abstain if and only if the p-value of the two-sided hypothesis test that $n_A$ is drawn from $\textrm{Binom}(n_A + n_B, \frac{1}{2})$ is less than $\alpha$.
Theorem 1 in \citet{hung2017rank} proves that the conditional probability that this event occurs given that $\hat{c}_A \neq c_A$ is exactly $\alpha$.
That is,
\begin{align*}
 \mathbb{P}(\textsc{Predict} \textrm{ does not abstain} | \hat{c}_A \neq c_A) = \alpha
\end{align*}
Therefore, we have:
\begin{align*}
\mathbb{P}(\textsc{Predict} \textrm{ returns class } \neq c_A) \le \alpha
\end{align*}

\end{proof}

\subsection{Certification}

The certification task is: given some input $x$ and a randomized smoothing classifier described by $(f, \sigma)$, return both (1) the prediction $g(x)$ and (2) a radius $R$ in which this prediction is certified robust.
This task requires identifying the class $c_A$ with maximal weight in $f(x+\varepsilon)$, estimating a lower bound $\underline{p_A}$ on $p_A := \mathbb{P}(f(x+\varepsilon) = c_A)$ and estimating an upper bound $\overline{p_B}$ on $p_B :=\max_{c \neq c_A} \mathbb{P}( f(x+\varepsilon) = c)$ (Figure \ref{page1illustration}).

Suppose for simplicity that we already knew $c_A$ and needed to obtain $\underline{p_A}$.
We could collect $n$ samples of $f(x+\varepsilon)$, count how many times $f(x+\varepsilon) = c_A$, and use a Binomial confidence interval to obtain a lower bound on $p_A$ that holds with probability at least $1 - \alpha$ over the $n$ samples.

However, estimating $\underline{p_A}$ and $\overline{p_B}$ while simultaneously identifying the top class $c_A$ is a little bit tricky, statistically speaking.
We propose a simple two-step procedure.
First, use $n_0$ samples from $f(x+\varepsilon)$ to take a guess $\hat{c}_A$ at the identity of the top class $c_A$.
In practice we observed that $f(x+\varepsilon)$ tends to put most of its weight on the top class, so $n_0$ can be set very small.
Second, use $n$ samples from $f(x+\varepsilon)$ to obtain some $\underline{p_A}$ and $\overline{p_B}$ for which $\underline{p_A} \le p_A$ and $\overline{p_B} \ge p_B$ with probability at least $ 1- \alpha$.
We observed that it is much more typical for the mass of $f(x+\varepsilon)$ not allocated to $c_A$ to be allocated entirely to one runner-up class than to be allocated uniformly over all remaining classes.
Therefore, the quantity $1 - \underline{p_A}$ is a reasonably tight upper bound on $p_B$.
Hence, we simply set $\overline{p_B} = 1 - \underline{p_A}$, so our bound becomes 
\begin{align*}
R &= \frac{\sigma}{2} ( \Phi^{-1}(\underline{p_A}) - \Phi^{-1}(1 - \underline{p_A})) \\
&= \frac{\sigma}{2} ( \Phi^{-1}(\underline{p_A}) + \Phi^{-1}(\underline{p_A})) \\
&= \sigma \Phi^{-1}(\underline{p_A})
\end{align*}

The full procedure is described in pseudocode as \textsc{Certify}.
If $\underline{p_A} < \frac{1}{2}$, we abstain from making a certification; this can occur especially if $\hat{c}_A \neq g(x)$, i.e. if we misidentify the top class using the first $n_0$ samples of $f(x+\varepsilon)$.

\textbf{Proposition \ref{thm:certify-correct} (restated).}
\textit{
With probability at least $1 - \alpha$ over the randomness in \textsc{Certify}, if  \textsc{Certify} returns a class $\hat{c}_A$ and a radius $R$ (i.e. does not abstain), then we have the robustness guarantee
\begin{align*}
g(x + \delta) = \hat{c}_A \quad \text{whenever} \quad \|\delta\|_2 < R
\end{align*}
}
\begin{proof}
From the contract of \textsc{LowerConfBound}, we know that with probability at least $1 - \alpha$ over the sampling of $\varepsilon_1 \hdots \varepsilon_n$, we have $\underline{p_A} \le \mathbb{P}[f(x+\varepsilon) = \hat{c}_A]$.
Notice that \textsc{Certify} returns a class and radius only if $\underline{p_A} > \frac{1}{2}$ (otherwise it abstains).
If $\underline{p_A} \le \mathbb{P}[f(x+\varepsilon) = \hat{c}_A]$ and $\frac{1}{2} < \underline{p_A}$, then we can invoke Theorem \ref{mainbound} with $\overline{p_B} = 1 - \underline{p_A}$ to obtain the desired guarantee.
\end{proof}

\newpage
\section{Estimating the certified test-set accuracy}
\label{section:highprobability}

In this appendix, we show how to convert the ``approximate certified test accuracy'' considered in the main paper into a lower bound on the true certified test accuracy that holds with high probability over the randomness in \textsc{Certify}.

Consider a classifier $g$,  a test set $S=\{(x_1, c_1) \hdots (x_m, c_m)\}$, and a radius $r$.
For each example $i \in [m]$, let $z_i$ indicate whether $g$'s prediction at $x_i$ is both correct and robust at radius $r$, i.e.
\begin{align*}
z_i &= \mathbf{1}[g(x_i + \delta) = c_i  \;\; \forall \|\delta\|_2 <r]
\end{align*}

The certified test set accuracy of $g$ at radius $r$ is defined as $\frac{1}{m} \sum_{i=1}^m z_i$.
If $g$ is a randomized smoothing classifier, we cannot compute this quantity exactly, but we can estimate a lower bound that holds with arbitrarily high probability over the randomness in \textsc{Certify}.
In particular, suppose that we run \textsc{Certify} with failure rate $\alpha$ on each example $x_i$ in the test set.
Let the Bernoulli random variable $Y_i$ denote the event that on example $i$, \textsc{Certify} returns the correct label $c_A = c_i$ and a 
certified radius $R$ which is greater than $r$.
Let $Y = \sum_{i=1}^m Y_i$.
In the main paper, we referred to $Y/m$ as the ``approximate certified accuracy.''
It is ``approximate'' because $Y_i=1$ does not mean that $z_i = 1$.
Rather, from Proposition \ref{thm:certify-correct}, we know the following: if $z_i = 0$, then $\mathbb{P}(Y_i = 1) \le \alpha$.
We now show how to exploit this fact to construct a one-sided confidence interval for the unobserved quantity $\frac{1}{m} \sum_{i=1}^m z_i$ using the observed quantities $Y$ and $m$.

\begin{thm}
For any $\rho > 0$, with probability at least $1 - \rho$ over the randomness in \textsc{Certify},
\begin{align}
\frac{1}{m} \sum_{i=1}^m z_i \ge \frac{1}{1 - \alpha} \left(\frac{Y}{m} - \alpha   - \sqrt{\frac{2 \alpha(1-\alpha) \log(1/\rho)}{m}} - \frac{\log(1/\rho)}{3m} \right)
\end{align}
\end{thm}

\begin{proof}
Let $m_{\text{good}} = \sum_{i=1}^m z_i$ and $m_{\text{bad}} = \sum_{i=1}^m (1 - z_i)$ be the number of test examples on which $z_i = 1$ or $z_i =0$, respectively.
We model $Y_i \sim \text{Bernoulli}(p_i)$, where $p_i$ is in general unknown.
Let $Y_{\text{good}} = \sum_{i: z_i = 1} Y_i$ and $Y_{\text{bad}} = \sum_{i: z_i = 0} Y_i$.
The quantity of interest, the certified accuracy $\frac{1}{m} \sum_{i=1}^m z_i$, is equal to $m_{\text{good}}/m$.
However, we only observe $Y = Y_{\text{good}} + Y_{\text{bad}}$.

Note that if $z_i = 0$, then $p_i \le \alpha$, so we have $\mathbb{E}[Y_i] = p_i \le \alpha$ and assuming $\alpha \le \frac{1}{2}$, we have $\text{Var}[Y_i] = p_i(1-p_i) \le \alpha ( 1 - \alpha)$.

Since $Y_\text{bad}$ is a sum of $m_{\text{bad}}$ independent random variables each bounded between zero and one, with $\mathbb{E}[Y_{\text{bad}}] \le \alpha m_{\text{bad}} $ and $\text{Var}(Y_{\text{bad}}) \le m_{\text{bad}} \alpha (1-\alpha) $, Bernstein's inequality \citep{blanchard2007} guarantees that with probability at least $1 - \rho$ over the randomness in \textsc{Certify},
\begin{align*}
Y_{\text{bad}} \le  \alpha m_{\text{bad}}  +  \sqrt{2 m_{\text{bad}} \alpha(1-\alpha) \log(1/\rho)} + \frac{\log(1/\rho)}{3}
\end{align*}
From now on, we manipulate this inequality --- remember that it holds with probability at least $1 - \rho$.

Since $Y = Y_{\text{good}} + Y_{\text{bad}}$, may write
\begin{align*}
Y_{\text{good}} \ge Y - \alpha m_{\text{bad}} - \sqrt{2 m_{\text{bad}} \alpha(1-\alpha) \log(1/\rho)} - \frac{\log(1/\rho)}{3}
\end{align*}
Since $m_{\text{good}} \ge Y_{\text{good}}$, we may write
\begin{align*}
m_{\text{good}} \ge Y - \alpha m_{\text{bad}} - \sqrt{2 m_{\text{bad}} \alpha(1-\alpha) \log(1/\rho)} - \frac{\log(1/\rho)}{3}
\end{align*}
Since $m_{\text{good}} + m_{\text{bad}} = m$, we may write
\begin{align*}
m_{\text{good}} \ge  \frac{1}{1 - \alpha} \left( Y - \alpha \, m -   \sqrt{2 m_{\text{bad}} \alpha(1-\alpha) \log(1/\rho)} - \frac{\log(1/\rho)}{3} \right)
\end{align*}
Finally, in order to make this confidence interval depend only on observables, we use $m_{\text{bad}} \le m$ to write
\begin{align*}
m_{\text{good}} \ge  \frac{1}{1 - \alpha} \left( Y - \alpha \, m  - \sqrt{2 m \alpha(1-\alpha) \log(1/\rho)} - \frac{\log(1/\rho)}{3} \right)
\end{align*}
Dividing both sides of the inequality by $m$ recovers the theorem statement.

\end{proof}

\newpage

\section{ImageNet and CIFAR-10 Results}
\label{section:experimenttables}

\subsection{Certification}

Tables \ref{table:imagenet_certify_results} and \ref{table:cifar_certify_results} show the approximate certified top-1 test set accuracy of randomized smoothing on ImageNet and CIFAR-10 with various noise levels $\sigma$.
By ``approximate certified accuracy,'' we mean that we ran \textsc{Certify} on a subsample of the test set, and for each $r$ we report the fraction of examples on which \textsc{Certify} (a) did not abstain, (b) returned the correct class, and (c) returned a radius $R$ greater than $r$.
There is some probability (at most $\alpha$) that any example's certification is inaccurate.
We used $\alpha = 0.001$ and $n = 100000$.
On CIFAR-10 our base classifier was a 110-layer residual network and we certified the full test set; on ImageNet our base classifier was a ResNet-50 and we certified a subsample of 500 points.
Note that the certified accuracy at $r=0$ is just the standard accuracy of the smoothed classifier.
See Appendix \ref{section:experimentdetails} for more experimental details.

\begin{table}[h]
\begin{center}
\begin{small}
\begin{sc}
\begin{tabular}{l | c c c c c c c}
& $r = 0.0$& $r = 0.5$& $r = 1.0$& $r = 1.5$& $r = 2.0$& $r = 2.5$& $r = 3.0$\\
\midrule
$\sigma = 0.25$ & \textbf{0.67} & \textbf{0.49} & 0.00 & 0.00 & 0.00 & 0.00 & 0.00\\
$\sigma = 0.50$ & 0.57 & 0.46 & \textbf{0.37} & \textbf{0.29} & 0.00 & 0.00 & 0.00\\
$\sigma = 1.00$ & 0.44 & 0.38 & 0.33 & 0.26 & \textbf{0.19} & \textbf{0.15} & \textbf{0.12}\\
\bottomrule
\end{tabular}
\end{sc}
\end{small}
\end{center}
\caption{Approximate certified test accuracy on ImageNet.  Each row is a setting of the hyperparameter $\sigma$, each column is an $\ell_2$ radius.  The entry of the best $\sigma$ for each radius is bolded. For comparison, random guessing would attain 0.001 accuracy.}
\label{table:imagenet_certify_results}
\end{table}

\begin{table}[h]
\begin{center}
\begin{small}
\begin{sc}
\begin{tabular}{l | c c c c c c c}
& $r = 0.0$& $r = 0.25$& $r = 0.5$& $r = 0.75$& $r = 1.0$& $r = 1.25$& $r = 1.5$\\
\midrule
$\sigma = 0.12$ & \textbf{0.83} & 0.60 & 0.00 & 0.00 & 0.00 & 0.00 & 0.00\\
$\sigma = 0.25$ & 0.77 & \textbf{0.61} & 0.42 & 0.25 & 0.00 & 0.00 & 0.00\\
$\sigma = 0.50$ & 0.66 & 0.55 & \textbf{0.43} & \textbf{0.32} & 0.22 & 0.14 & 0.08\\
$\sigma = 1.00$ & 0.47 & 0.41 & 0.34 & 0.28 & \textbf{0.22} & \textbf{0.17} & \textbf{0.14}\\
\bottomrule
\end{tabular}
\end{sc}
\end{small}
\end{center}
\caption{Approximate certified test accuracy on CIFAR-10.  Each row is a setting of the hyperparameter $\sigma$, each column is an $\ell_2$ radius.  The entry of the best $\sigma$ for each radius is bolded. For comparison, random guessing would attain 0.1 accuracy.}
\label{table:cifar_certify_results}
\end{table}

\subsection{Prediction}

Table \ref{table:imagenet_predict_results} shows the performance of \textsc{Predict} as the number of Monte Carlo samples $n$ is varied between 100 and 10,000.
Suppose that for some test example $(x,c )$, \textsc{Predict} returns the label $\hat{c}_A$.
We say that this prediction was \emph{correct} if $\hat{c}_A = c$ and we say that this prediction was \emph{accurate} if $\hat{c}_A = g(x)$.
For example, a prediction could be correct but inaccurate if $g$ is wrong at $x$, yet \textsc{Predict} accidentally returns the correct class.
Ideally, we'd like \textsc{Predict} to be both correct and accurate.

With $n= $ 100 Monte Carlo samples and a failure rate of $\alpha = 0.001$, \textsc{Predict} is cheap to evaluate (0.15 seconds on our hardware) yet it attains relatively high top-1 accuracy of 65\% on the ImageNet test set, and only abstains 12\% of the time.
When we use $n= $ 10,000 Monte Carlo samples, \textsc{Predict} takes longer to evaluate (15 seconds), yet only abstains 4\% of the time.
Interestingly, we observe from Table \ref{table:imagenet_predict_results} that most of the abstentions when $n=100$ were for examples on which $g$ was wrong, so in practice we would lose little accuracy by taking $n$ to be as small as 100.

\begin{table}[h]
\begin{center}
\begin{small}
\begin{sc}
\begin{tabular}{lrrrrr}
\toprule
{} &  correct, accurate &  correct, inaccurate &  incorrect, accurate &  incorrect, inaccurate &  abstain \\
n     &                    &                      &                      &                        &          \\
\midrule
100   &               0.65 &                 0.00 &                 0.23 &                   0.00 &     0.12 \\
1000  &               0.68 &                 0.00 &                 0.28 &                   0.00 &     0.04 \\
10000 &               0.69 &                 0.00 &                 0.30 &                   0.00 &     0.01 \\
\bottomrule
\end{tabular}
\end{sc}
\end{small}
\end{center}
\caption{Performance of \textsc{Precict} as $n$ is varied.  The dataset was ImageNet and $\sigma = 0.25$, $\alpha = 0.001$.   Each column shows the fraction of test examples which ended up in one of five categories; the prediction at $x$ is ``correct'' if \textsc{Predict} returned the true label, while the prediction is ``accurate'' if \textsc{Predict} returned $g(x)$.  Computing $g(x)$ exactly is not possible, so in order to determine whether \textsc{Predict} was accurate, we took the gold standard to be the top class over $n=$100,000 Monte Carlo samples.}
\label{table:imagenet_predict_results}
\end{table}

\newpage
\section{Training with Noise}
\label{section:training-with-noise}

As mentioned in section \ref{section:train-base-classifier}, in the experiments for this paper, we followed  \citet{lecuyer2018certified} and trained the base classifier by minimizing the cross-entropy loss with Gaussian data augmentation.
We now provide some justification for this idea.

Let $\{(x_1, c_1), \hdots, (x_n, c_n)\}$ be a training dataset.
We assume that the base classifier takes the form $f(x) = \argmax_{c \in \mathcal{Y}} f_c (x)$, where each $f_c$ is the scoring function for class $c$.

Suppose that our goal is to maximize the sum of of the log-probabilities that $f$ will classify each $x_i+\varepsilon$ as $c_i$:
\begin{align}
 \sum_{i=1}^n \log \mathbb{P}_{\varepsilon}(f(x_i+\varepsilon) = c_i) &= \sum_{i=1}^n \log \mathbb{E}_{\varepsilon} \; \mathbf{1} \left[ \argmax_c f_c(x_i + \varepsilon) = c_i \right] \label{eq:train-objective}
\end{align}

Recall that the softmax function can be interpreted as a continuous, differentiable approximation to $\argmax$:
\begin{align*}
\mathbf{1} \left[ \argmax_c f_c(x_i + \varepsilon) = c_i \right] \approx \frac{\exp( f_{c_i}(x_i + \varepsilon) )}{ \sum_{c \in \mathcal{Y}} \exp( f_c(x_i + \varepsilon))} 
\end{align*}
Therefore, our objective is approximately equal to:
\begin{align}
\sum_{i=1}^n \log \mathbb{E}_{\varepsilon} \left[ \frac{\exp( f_{c_i}(x_i + \varepsilon) )}{ \sum_{c \in \mathcal{Y}} \exp( f_c(x_i + \varepsilon))}  \right] \label{eq:approx-train-objective}
\end{align}
By Jensen's inequality and the concavity of $\log$, this quantity is lower-bounded by:
\begin{align*}
\sum_{i=1}^n \mathbb{E}_{\varepsilon} \left[ \log \frac{\exp( f_{c_i}(x_i + \varepsilon) )}{ \sum_{c \in \mathcal{Y}} \exp( f_c(x_i + \varepsilon))}  \right]
\end{align*}
which is the negative of the cross-entropy loss under Gaussian data augmentation.

Therefore, minimizing the cross-entropy loss under Gaussian data augmentation will maximize (\ref{eq:approx-train-objective}), which will approximately maximize (\ref{eq:train-objective}).

\newpage
\section{Noise Level can Scale with Input Resolution}
\label{section:input-resolution}

Since our robustness guarantee \eqref{radius} in Theorem \ref{mainbound} does not explicitly depend on the data dimension $d$, one might worry that randomized smoothing is less effective for images in high resolution --- certifying a fixed $\ell_2$ radius is ``less impressive'' for, say, $224 \times 224$ image than for a $56 \times 56$ image.
However, it turns out that in high resolution, images can be corrupted with larger levels of isotropic Gaussian noise while still preserving their content.
This fact is made clear by Figure \ref{figure:noise_dimension}, which shows an image at high and low resolution corrupted by Gaussian noise with the same variance.full
The class (``hummingbird'') is easy to discern from the high-resolution noisy image, but not from the low-resolution noisy image.
As a consequence, in high resolution one can take $\sigma$ to be larger while still being able to obtain a base classifier that classifies noisy images accurately.
Since our Theorem \ref{mainbound} robustness guarantee scales linearly with $\sigma$, this means that in high resolution one can certify larger radii.

\begin{figure}[h]
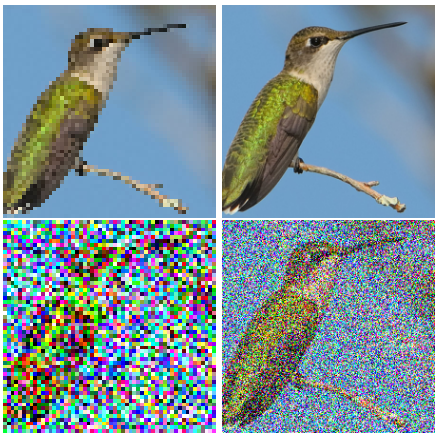

\begin{center}
	\includegraphics[width=80px]{figures/hummingbird_small_clean}
	\includegraphics[width=80px]{figures/hummingbird_big_clean} \\
	\includegraphics[width=80px]{figures/hummingbird_small_noisy}
	\includegraphics[width=80px]{figures/hummingbird_big_noisy}
\end{center}
\caption{\textbf{Top}: An ImageNet image from class  ``hummingbird'' in resolutions 56x56 (left) and 224x224 (right). \textbf{Bottom}: the same images corrupted by isotropic Gaussian noise at $\sigma=0.5$. On noiseless images the class is easy to distinguish no matter the resolution, but on noisy data the class is much easier to distinguish when the resolution is high.}
\label{figure:noise_dimension}
\end{figure}

The argument above can be made rigorous, though we first need to decide what it means for two images to be high- and low-resolution versions of each other.
Here we present one solution:

Let $\mathcal{X}$ denote the space of ``high-resolution'' images in dimension $2k \times 2k \times 3$, and let  $\mathcal{X}'$ denote the space of ``low-resolution'' images in dimension $k \times k \times 3$.
Let $\textsc{AvgPool}: \mathcal{X} \to \mathcal{X}'$ be the function which takes as input an image $x$ in dimension $2k \times 2k \times 3$, averages together every 2x2 square of pixels, and outputs an image in dimension $k \times k \times 3$. 

Equipped with these definitions, we can say that $(x, x') \in \mathcal{X} \times \mathcal{X}' $ are a high/low resolution image pair if $x' = \textsc{AvgPool}(x)$.

\begin{proposition}
Given any smoothing classifier $g': \mathcal{X}' \to \mathcal{Y}$, one can construct a smoothing classifier $g: \mathcal{X} \to \mathcal{Y}$ with the following property:
for any $x \in \mathcal{X}$ and $x' = \textsc{AvgPool}(x)$, $g$ predicts the same class at $x$ that $g'$ predicts at $x'$, but is certifiably robust at twice the radius.
\label{prop:resolution}
\end{proposition}

\begin{proof}
Given some smoothing classifier $g' = (f', \sigma')$ from $\mathcal{X}'$ to $\mathcal{Y}$, define $g$ to be the smoothing classifier $(f, \sigma)$ from $\mathcal{X}$ to $\mathcal{Y}$ with noise level $\sigma = 2 \sigma'$ and base classifier $f(x) = f'(\textsc{AvgPool}(x))$.
Note that the average of four independent copies of $\mathcal{N}(0, (2\sigma)^2)$ is distributed as $\mathcal{N}(0, \sigma^2)$.
Therefore, for any high/low-resolution image pair $x' = \textsc{AvgPool}(x)$, the random variable $\textsc{AvgPool}(x+\varepsilon)$, where $\varepsilon \sim \mathcal{N}(0, (2\sigma)^2 I_{2k \times 2k \times 3})$, is equal in distribution to the random variable $x'+\varepsilon'$, where $\varepsilon' \sim \mathcal{N}(0, \sigma^2 I_{k \times k \times 3})$.
Hence, $f(x+\varepsilon) = f'(\textsc{AvgPool}(x+\varepsilon))$ has the same distribution as $f'(x'+\varepsilon')$.
By the definition of $g$, this means that $g(x) = g'(x')$,
Additionally, by Theorem \ref{mainbound}, since $\sigma=2 \sigma'$, this means that $g$'s prediction at $x$ is certifiably robust at twice the radius as $g'$'s prediction at $x'$.
\end{proof}

\newpage
\section{Additional Experiments}
\label{section:additionalexperiments}

\subsection{Comparisons to baselines}
\label{section:additionalexperiments:comparisons}
Figure \ref{fig:compare-baselines} compares the certified accuracy of a smoothed 20-layer resnet to that of the released models from two recent works on certified $\ell_2$ robustness: the Lipschitz approach from \citet{tsuzuku2018lipschitz} and the approach from \citet{zhang2018efficient}.
Note that in these experiments, the base classifier for smoothing was larger than the networks of competing approaches.
The comparison to \citet{zhang2018efficient} is on CIFAR-10, while the comparison to  \citet{tsuzuku2018lipschitz} is on SVHN.
Note that for each comparison, we preprocessed the dataset to follow the preprocessing used when the baseline was trained; therefore, the radii reported for CIFAR-10 here are not comparable to the radii reported elsewhere in this paper.
Full experimental details are in Appendix \ref{section:experimentdetails}.

\begin{figure}[h]
\centering
\begin{subfigure}{0.49\textwidth}
	\includegraphics[width=\textwidth]{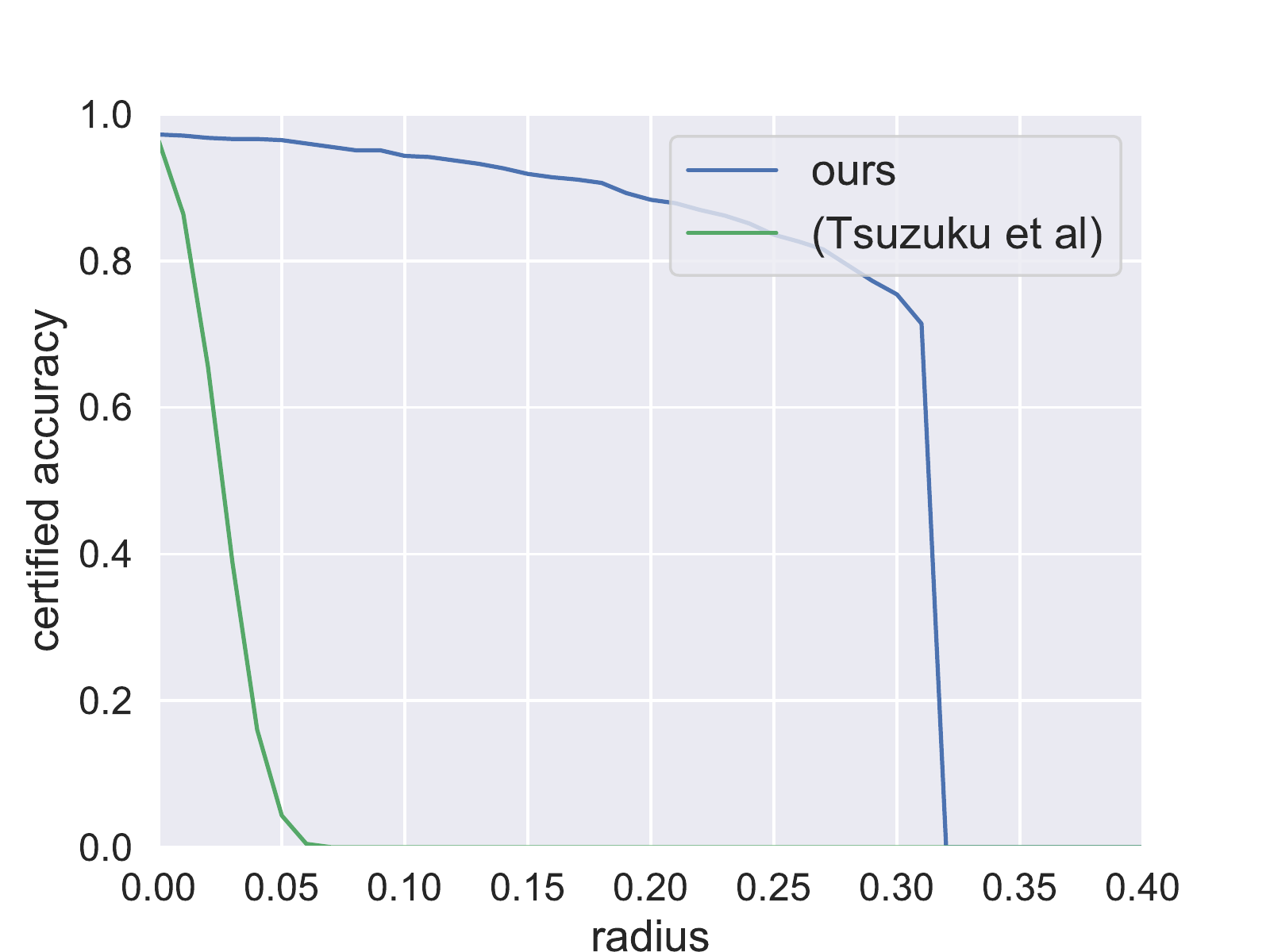}
	\caption{\citet{tsuzuku2018lipschitz}}
	\label{compare_tsuzuku}
\end{subfigure}
\begin{subfigure}{0.49\textwidth}
	\includegraphics[width=\textwidth]{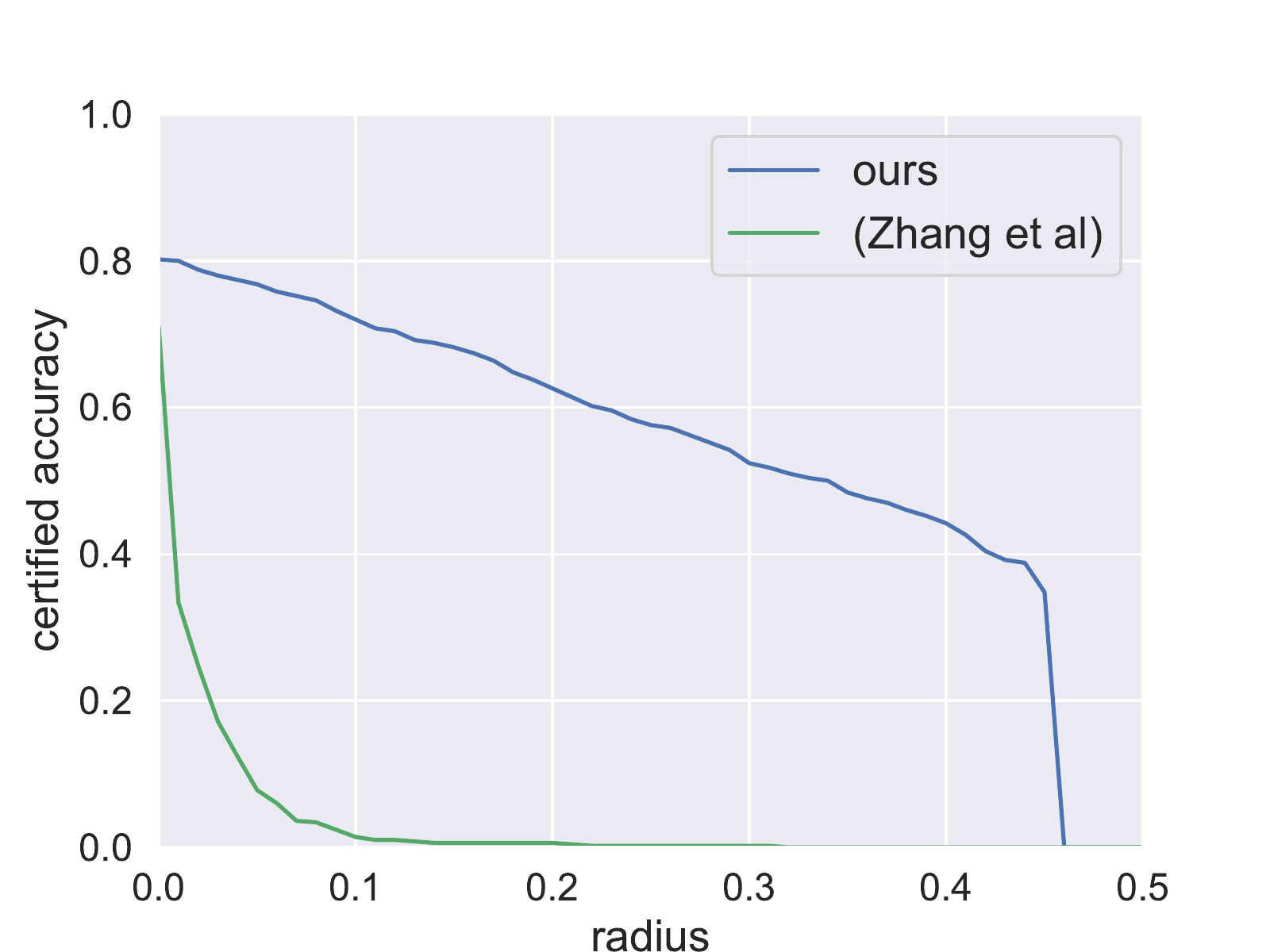}
	\caption{\citet{zhang2018efficient}}
	\label{compare_zhang}
\end{subfigure}
\caption{Randomized smoothing with a 20-layer resnet base classifier attains higher certified accuracy than the released models from two recent works on certified $\ell_2$ robustness.}
\label{fig:compare-baselines}
\end{figure}

\subsection{High-probability guarantees}
\label{section:additionalexperiments:highprobability}

Appendix \ref{section:highprobability} details how to use \textsc{Certify} to obtain a lower bound on the certified test accuracy at radius $r$ of a randomized smoothing classifier that holds with high probability over the randomness in \textsc{Certify}.
In the main paper, we declined to do this and simply reported the approximate certified test accuracy, defined as the fraction of test examples for which \textsc{Certify} gives the correct prediction and certifies it at radius $r$.
Of course, with some probability (guaranteed to be less than $\alpha$), each of these certifications is wrong.

However, we now demonstrate empirically that there is a negligible difference between a proper high-probability lower bound on the certified accuracy and the approximate version that we reported in the paper.
We created a randomized smoothing classifier $g$ on ImageNet with a ResNet-50 base classifier and noise level $
\sigma = 0.25$.
We used \textsc{Certify}  with $\alpha = 0.001$ to certify a subsample of 500 examples from the ImageNet test set.
From this we computed the approximate certified test accuracy at each radius $r$.
Then we used the correction from Appendix \ref{section:highprobability} with $\rho = 0.001$ to obtain a lower bound on the certified test accuracy at $r$ that holds pointwise with probability at least $1 - \rho$ over the randomness in \textsc{Certify}.
Figure \ref{figure:highprobability_lowerbound} plots both quantities as a function of $r$.
Observe that the difference is so negligible that the lines almost overlap.

\begin{figure}[h]
\begin{center}
\includegraphics[width=150px]{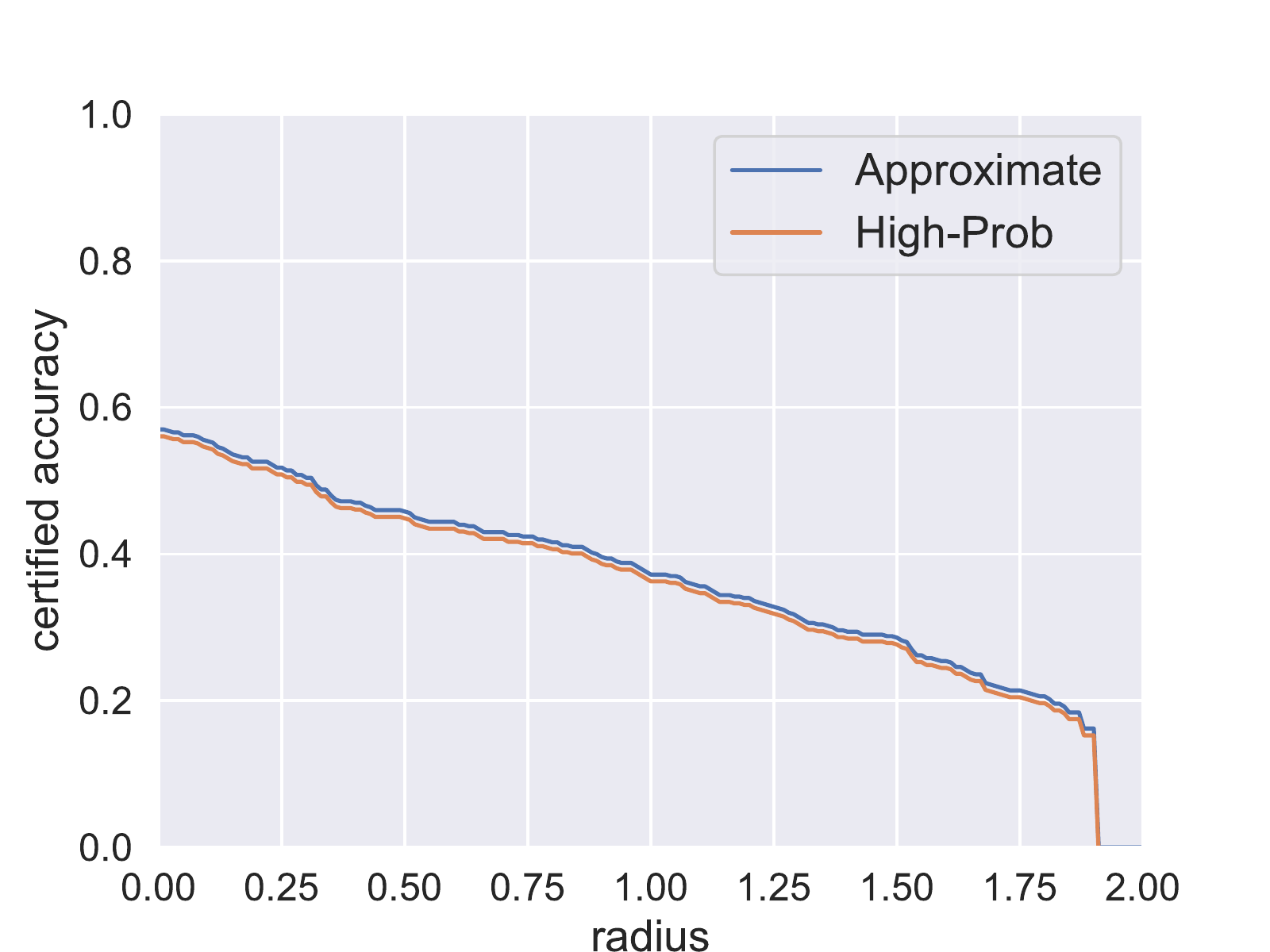}
\caption{The difference between the approximate certified accuracy, and a high-probability lower bound on the certified accuracy, is negligible.}
\label{figure:highprobability_lowerbound}
\end{center}
\end{figure}

\subsection{How much noise to use when training the base classifier?}
\label{section:additionalexperiments:noise}

In the main paper, whenever we created a randomized smoothing classifier $g$ at noise level $\sigma$, we always trained the corresponding base classifier $f$ with Gaussian data augmentation at noise level $\sigma$.
In Figure \ref{varytrainnoise}, we show the effects of training the base classifier with a different level of Gaussian noise.
Observe that $g$ has a lower certified accuracy if $f$ was trained using a different noise level.
It seems to be worse to train with noise $< \sigma$ than to train with noise $> \sigma$.

\begin{figure}[!h]
\begin{center}
\begin{subfigure}{0.40\textwidth}
	\includegraphics[width=\textwidth]{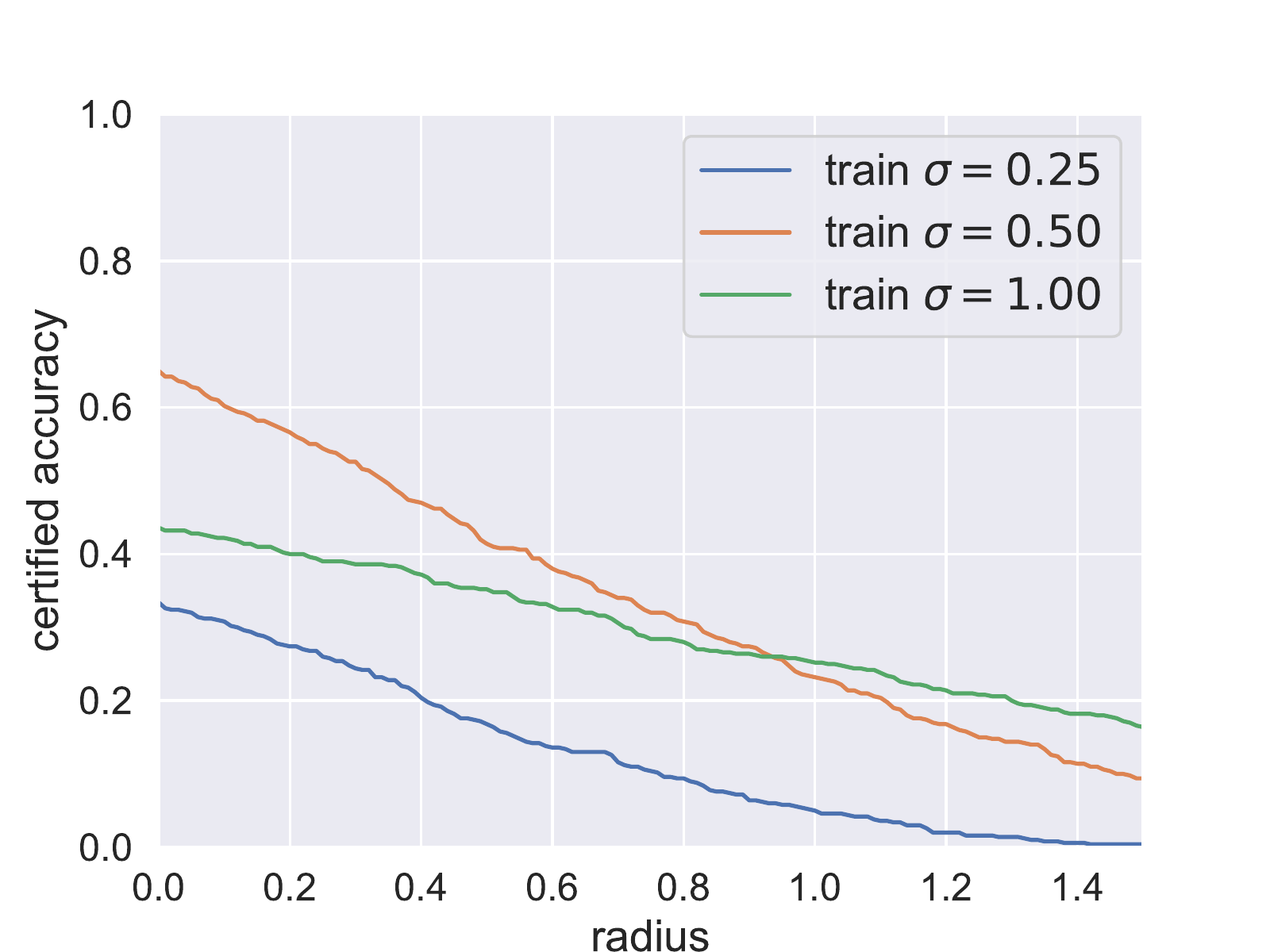}
	\caption{CIFAR-10}
\end{subfigure}
\begin{subfigure}{0.40\textwidth}
	\includegraphics[width=\textwidth]{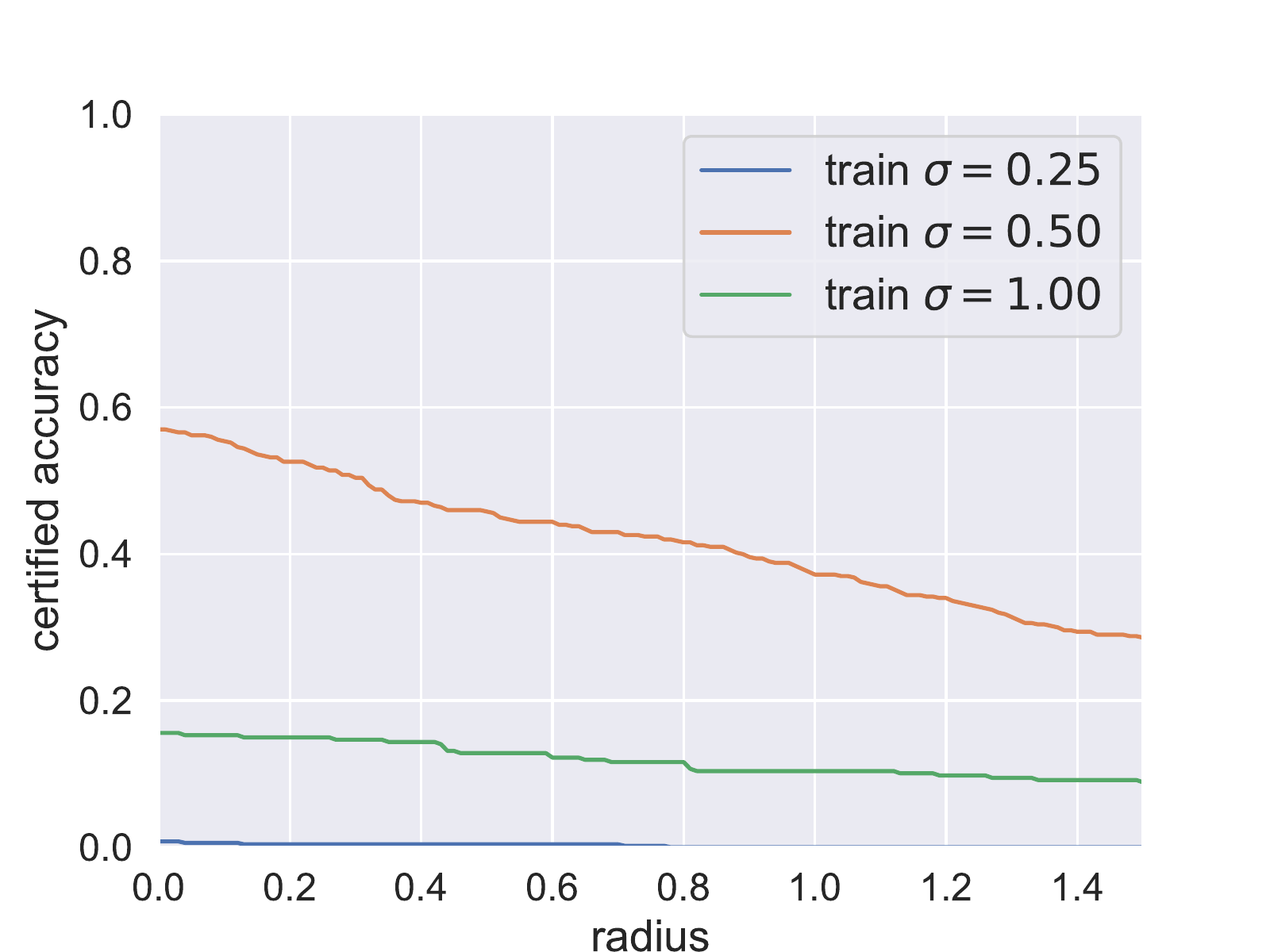}
	\caption{ImageNet}
\end{subfigure}
\caption{Vary training noise while holding prediction noise fixed at $\sigma=0.50$.}
\label{varytrainnoise}
\end{center}
\end{figure}

\newpage

\section{Derivation of Prior Randomized Smoothing Guarantees}
\label{section:derivation}

In this appendix, we derive the randomized smoothing guarantees of \citet{lecuyer2018certified} and \citet{li2018second} using the notation of our paper.
Both guarantees take same general form as ours, except with a different expression for $R$:

\textbf{Theorem (generic guarantee):} \textit{
Let $f: \mathbb{R}^d \to \mathcal{Y}$ be any deterministic or random function, and let $\varepsilon \sim \mathcal{N}(0, \sigma^2 I)$.
Let $g$ be defined as in (\ref{g}).
Suppose $c_A \in \mathcal{Y}$ and $\underline{p_A}, \overline{p_B} \in [0, 1]$ satisfy:
\small
\begin{align}
\label{knownfacts3}
\mathbb{P}(f(x+\varepsilon) = c_A) \ge \underline{p_A} \ge \overline{p_B }\ge \max_{c \neq c_A} \mathbb{P}(f(x+ \varepsilon) = c)
\end{align}
\normalsize
Then $g(x+\delta) = c_A$ for all $\|\delta\|_2 < R$.
}

For convenience, define the notation $X \sim \mathcal{N}(x, \sigma^2 I)$ and $Y \sim \mathcal{N}(x+\delta, \sigma^2 I)$.

\subsection{\citet{lecuyer2018certified}}

\citet{lecuyer2018certified} proved a version of the generic robustness guarantee in which
\begin{align*}
R = \sup_{0 < \beta \le  \min \left(1, \frac{1}{2} \log\frac{\underline{p_A}}{\overline{p_B}} \right) } \frac{\sigma \beta }{\sqrt{2 \log \left( \frac{1.25(1 + \exp(\beta))}{\underline{p_A} - \exp(2 \beta) \overline{p_B}} \right)}}
\end{align*}

\begin{proof}
In order to avoid notation that conflicts with the rest of this paper, we use  $\beta$ and $\gamma$ where \citet{lecuyer2018certified} used $\epsilon$ and $\delta$.

Suppose that we have some $0 < \beta \le 1$ and $\gamma > 0$ such that 
\begin{align}
\sigma^2 = \frac{\|\delta\|^2}{\beta^2} 2 \log \frac{1.25}{\gamma} \label{sigma2}
\end{align}

The ``Gaussian mechanism'' from differential privacy guarantees that:
\begin{align}
\mathbb{P}(f(X) = c_A) &\le \exp(\beta) \mathbb{P}(f(Y) = c_A) + \gamma \label{DP1}
\end{align}
and, symmetrically,
\begin{align}
\mathbb{P}(f(Y) = c_B) &\le \exp(\beta) \mathbb{P}(f(X) = c_B) + \gamma \label{DP2}
\end{align}
See \citet{lecuyer2018certified}, Lemma 2 for how to obtain this form from the standard form of the $(\beta, \gamma)$ DP definition.

Fix a perturbation $\delta$.
To guarantee that $g(x+\delta) = c_A$, we need to show that $\mathbb{P}(f(Y) = c_A) > \mathbb{P}(f(Y) = c_B) $  for each $c_B \neq c_A$.

Together, (\ref{DP1}) and (\ref{DP2}) imply that to guarantee $\mathbb{P}(f(Y) = c_A) > \mathbb{P}(f(Y) = c_B) $ for any $c_B$, it suffices to show that:
\begin{align}
\mathbb{P}(f(X) = c_A) > \exp(2\beta) \mathbb{P}(f(X) = c_B) + \gamma (1+\exp(\beta)) \label{DPbound1}
\end{align}
Therefore, in order to guarantee that$\mathbb{P}(f(Y) = c_A) > \mathbb{P}(f(Y) = c_B) $  for each $c_B \neq c_A$,  by (\ref{knownfacts3}) it suffices to show:
\begin{align}
\underline{p_A} > \exp(2\beta) \overline{p_B} + \gamma (1+\exp(\beta)) \label{DPbound2}
\end{align}
Now, inverting (\ref{sigma2}), we obtain:
\begin{align}
\gamma &= 1.25 \exp \left( - \frac{\sigma^2 \beta^2}{2 \|\delta\|^2} \right) \label{gamma}
\end{align}
Plugging (\ref{gamma}) into (\ref{DPbound2}), we see that to guarantee  $\mathbb{P}(f(Y) = c_A)  \ge \mathbb{P}(f(Y) = c_B)$ it suffices to show that:
\begin{align}
\underline{p_A} > \exp(2 \beta) \overline{p_B} + 1.25 \exp \left( - \frac{\sigma^2 \beta^2}{2 \|\delta\|^2} \right) (1 + \exp(\beta))
\end{align}
which rearranges to:
\begin{align}
\frac{ \underline{p_A} - \exp(2 \beta) \overline{p_B} } {1.25 (1 +\exp(\beta))} > \exp \left(- \frac{\sigma^2 \beta^2}{2 \|\delta\|^2} \right) \label{DPbound3}
\end{align}
Since the RHS is always positive, and the denominator on the LHS is always positive, this condition can only possibly hold if the numerator on the LHS is positive.
Therefore, we need to restrict $\beta$ to
\begin{align}
0 < \beta \le \min \left(1, \frac{1}{2} \log\frac{\underline{p_A}}{\overline{p_B}} \right) \label{validepsilon}
\end{align}
The condition (\ref{DPbound3}) is equivalent to:
\begin{align}
\|\delta\|^2 \log \frac{1.25 (1 +\exp(\beta))} { \underline{p_A} - \exp(2 \beta) \overline{p_B} } < \frac{\sigma^2 \beta^2}{2} 
\end{align}
Since $\underline{p_A} \le 1$ and $\overline{p_B} \ge 0$, the denominator in the LHS is $\le 1$ which is in turn $\le$ the numerator on the LHS.
Therefore, the term inside the log in the LHS is greater than 1, so the log term on the LHS is greater than zero.
Therefore, we may divide both sides of the inequality by the log term on the LHS to obtain:
\begin{align}
\| \delta \|^2 < \frac{\sigma^2 \beta^2}{2 \log \left( \frac{1.25(1 + \exp(\beta))}{\underline{p_A} - \exp(2 \beta) \overline{p_B}} \right)}
\end{align}
Finally, we take the square root and maximize the bound over all valid $\beta$ (\ref{validepsilon}) to yield:
\begin{align}
\| \delta \| < \sup_{0 < \beta \le  \min \left(1, \frac{1}{2} \log\frac{\underline{p_A}}{\overline{p_B}} \right) } \frac{\sigma \beta }{\sqrt{2 \log \left( \frac{1.25(1 + \exp(\beta))}{\underline{p_A} - \exp(2 \beta) \overline{p_B}} \right)}}
\end{align}
\end{proof}

Figure \ref{plot_lecuyerbound} plots this bound at varying settings of the tuning parameter $\beta$, while Figure \ref{optimize_lecuyerbound} plots how the bound varies with $\beta$ for a fixed $\underline{p_A}$ and $\overline{p_B}$.

\subsection{\citet{li2018second}}

\citet{li2018second} proved a version of the generic robustness guarantee in which
\begin{align*}
R = \sup_{\alpha > 0} \sigma \sqrt{- \frac{2}{\alpha} \log \left( 1 - \underline{p_A} - \overline{p_B} + 2 \left( \frac{1}{2} (\underline{p_A}^{1 - \alpha} + \overline{p_B}^{1 - \alpha})^{1 - \alpha} \right) \right)}
\end{align*}

\begin{proof}
A generalization of KL divergence, the $\alpha$-Renyi divergence is an information theoretic measure of distance between two distributions.  It is parameterized by some $\alpha > 0$.  The $\alpha$-Renyi divergence between two discrete distributions $P$ and $Q$ is defined as:
\begin{align}
D_\alpha(P||Q) := \frac{1}{\alpha - 1} \log \left( \sum_{i=1}^k \frac{p_i^\alpha}{q_i^{\alpha - 1 }} \right)
\end{align}
In the continuous case, this sum is replaced with an integral.
The divergence is undefined when $\alpha = 1$ since a division by zero occurs, but the limit of $D_\alpha(P||Q)$ as $\alpha \to 1$ is the KL divergence between $P$ and $Q$.

\citet{li2018second} prove that if $P$ is a discrete distribution for which the highest probability class has probability $\ge \underline{p_A}$ and all other classes have probability $\le \overline{p_B}$, then for any other discrete distribution $Q$ for which
\begin{align}
D_\alpha(P||Q) < - \log \left( 1 - \underline{p_A} - \overline{p_B} + 2 \left( \frac{1}{2} (\underline{p_A}^{1 - \alpha} + \overline{p_B}^{1 - \alpha})^{1 - \alpha} \right) \right) \label{renyibound}
\end{align}
the highest-probability class in $Q$ is guaranteed to be the same as the highest-probability class in $P$.

We now apply this result to the discrete distributions $P=f(X)$ and $Q=f(Y)$.
If $D_\alpha(f(X) || f(Y))$ satisfies (\ref{renyibound}), then it is guaranteed that $g(x) = g(x+\delta)$.

The data processing inequality states that applying a function to two random variables can only decrease the $\alpha$-Renyi divergence between them.  In particular,
\begin{align}
D_\alpha(f(X) || f(Y)) \le D_\alpha(X || Y)
\end{align}
There is a closed-form expression for the $\alpha$-Renyi divergence between two Gaussians:
\begin{align}
D_\alpha(X||Y) = \frac{\alpha \|\delta\|^2}{2 \sigma^2}
\end{align}
Therefore, we can guarantee that $g(x+\delta) = c_A$ so long as
\begin{align}
\frac{\alpha \|\delta\|^2}{2 \sigma^2} < - \log \left( 1 - \underline{p_A} - \overline{p_B} + 2 \left( \frac{1}{2} (\underline{p_A}^{1 - \alpha} + \overline{p_B}^{1 - \alpha})^{1 - \alpha} \right) \right) 
\end{align}
which simplifies to
\begin{align}
\|\delta\| < \sigma \sqrt{- \frac{2}{\alpha} \log \left( 1 - \underline{p_A} - \overline{p_B} + 2 \left( \frac{1}{2} (\underline{p_A}^{1 - \alpha} + \overline{p_B}^{1 - \alpha})^{1 - \alpha} \right) \right)}
\end{align}
Finally, since this result holds for any $\alpha > 0$, we may maximize over $\alpha$ to obtain the largest possible certified radius: 
\begin{align}
\|\delta\| <  \sup_{\alpha > 0} \sigma \sqrt{- \frac{2}{\alpha} \log \left( 1 - \underline{p_A} - \overline{p_B} + 2 \left( \frac{1}{2} (\underline{p_A}^{1 - \alpha} + \overline{p_B}^{1 - \alpha})^{1 - \alpha} \right) \right)}
\end{align}
\end{proof}

Figure \ref{plot_libound} plots this bound at varying settings of the tuning parameter $\alpha$, while figure \ref{optimize_libound} plots how the bound varies with $\alpha$ for a fixed $\underline{p_A}$ and $\overline{p_B}$.

\begin{figure*}[ht]
\begin{center}
\begin{subfigure}{0.49\textwidth}
	\includegraphics[width=\textwidth]{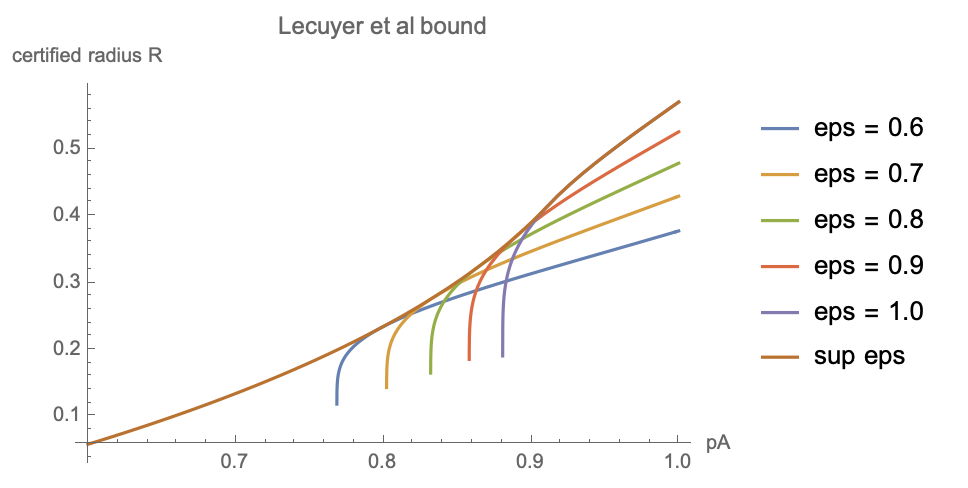}
	\caption{The \citet{lecuyer2018certified} bound over several settings of $\beta$.  The brown line is the pointwise supremum over all eligible $\beta$, computed numerically.}
	\label{plot_lecuyerbound}
\end{subfigure}
\hfill
\begin{subfigure}{0.49\textwidth}
	\includegraphics[width=\textwidth]{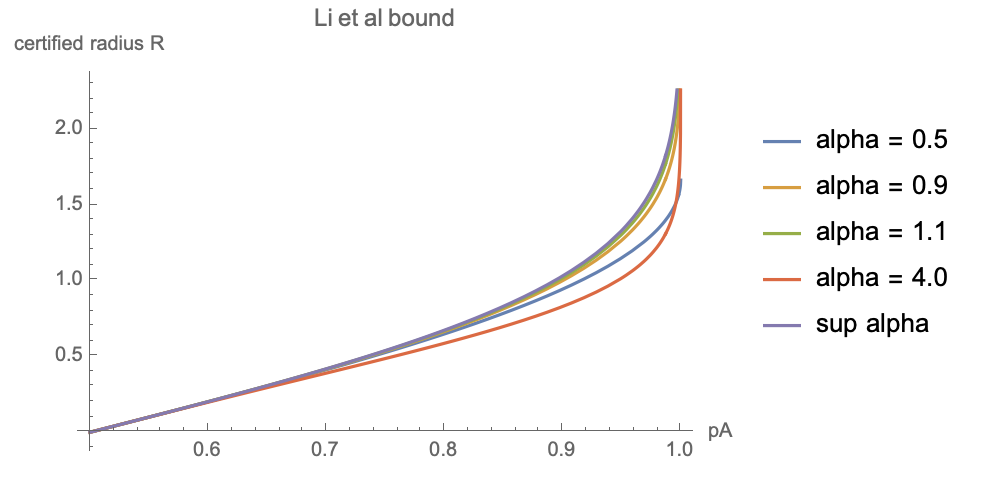}
	\caption{The \citet{li2018second} bound over several settings of $\alpha$.  The purple line is the pointwise supremum over all eligible $\alpha$, computed numerically.}
	\label{plot_libound}
\end{subfigure}
\begin{subfigure}{0.49\textwidth}
	\includegraphics[width=\textwidth]{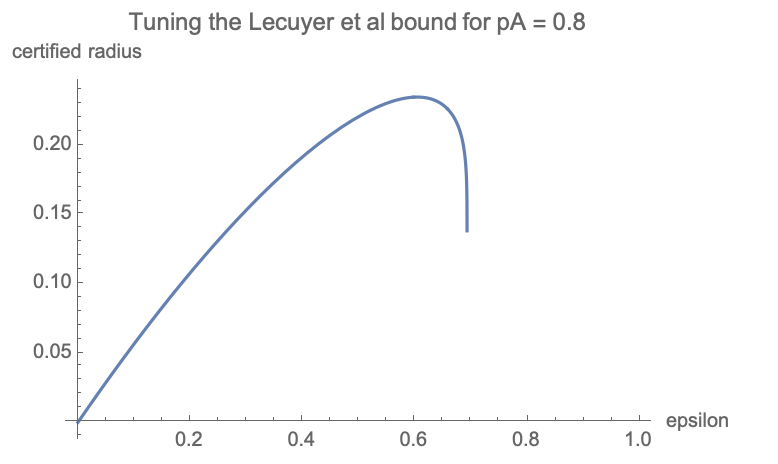}
	\caption{Tuning the \citet{lecuyer2018certified} bound wrt $\beta$ when $\underline{p_A} = 0.8, \overline{p_B} = 0.2$}
	\label{optimize_lecuyerbound}
\end{subfigure}
\begin{subfigure}{0.49\textwidth}
	\includegraphics[width=\textwidth]{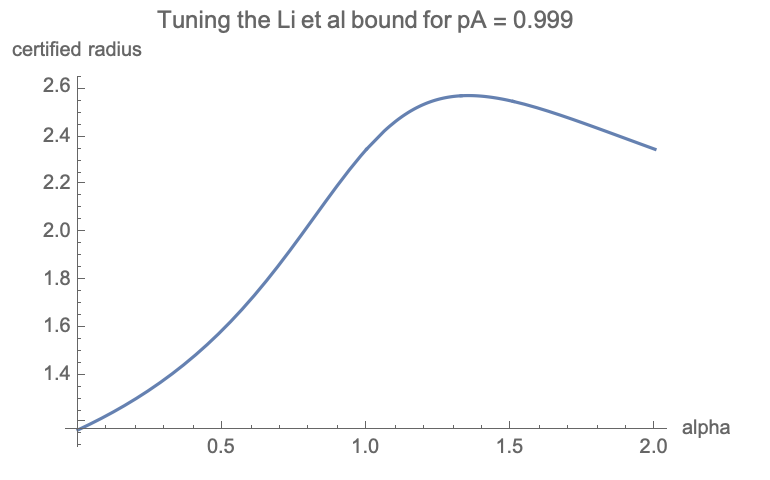}
	\caption{Tuning the \citet{li2018second} bound wrt $\alpha$ when $\underline{p_A} = 0.999, \overline{p_B} = 0.0001$}
	\label{optimize_libound}
\end{subfigure}
\end{center}
\end{figure*}

\newpage

\section{Experiment Details}
\label{section:experimentdetails}

\subsection{Comparison to baselines}

We compared randomized smoothing against three recent approaches for $\ell_2$-robust classification \citep{tsuzuku2018lipschitz, wong2018scaling, zhang2018efficient}.
 \citet{tsuzuku2018lipschitz} and \citet{wong2018scaling} propose both a robust training method and a complementary certification mechanism, while \citet{zhang2018efficient} propose a method to certify generically trained networks.
In all cases we compared against networks provided by the authors.
We compared against \citet{wong2018scaling} and \citet{zhang2018efficient} on CIFAR-10, and we compared against \citet{tsuzuku2018lipschitz} on SVHN.

In image classification it is common practice to preprocess a dataset by subtracting from each channel the mean over the dataset, and dividing each channel by the standard deviation over the dataset.
However, we wanted to report certified radii in the original image coordinates rather than in the standardized coordinates.
Therefore, throughout most of this work we \emph{first} added the Gaussian noise, and \emph{then} standardized the channels, before feeding the image to the base classifier.
(In the practical PyTorch implementation, the first layer of the base classifier was a layer that standardized the input.)
However, all of the baselines we compared against provided pre-trained networks which assumed that the dataset was first preprocessed in a specific way.
Therefore, when comparing against the baselines we also preprocessed the datasets first, so that we could report certified radii that were directly comparable to the radii reported by the baseline methods.

\paragraph{Comparison to \citet{wong2018scaling}}

Following \citet{wong2018scaling}, the CIFAR-10 dataset was preprocessed by subtracting $(0.485, 0.456, 0.406)$ and dividing by $(0.225, 0.225, 0.225)$.

While the body of the \citet{wong2018scaling} paper focuses on $\ell_\infty$ certified robustness, their algorithm naturally extends to $\ell_2$ certified robustness, as developed in the appendix of the paper.
We used three $\ell_2$-trained residual networks publicly released by the authors, each trained with a different setting of their hyperparameter $\epsilon \in \{0.157, 0.628, 2.51\}$.
We used code publicly released by the authors at \url{https://github.com/locuslab/convex_adversarial/blob/master/examples/cifar_evaluate.py} to compute the robustness radius of test images.
The code accepts a radius and returns TRUE (robust) or FALSE (not robust); we incorporated this subroutine into a binary search procedure to find the largest radius for which the code returned TRUE.

For randomized smoothing we used $\sigma = 0.6$ and a 20-layer residual network base classifier.
We ran \textsc{Certify} with $n_0 = 100$, $n=$ 100,000 and $\alpha = 0.001$.

For both methods, we certified the full CIFAR-10 test set.

\paragraph{Comparison to \citet{tsuzuku2018lipschitz}}

Following  \citet{tsuzuku2018lipschitz}, the SVHN dataset was not preprocessed except that pixels were divided by 255 so as to lie within [0, 1].

We compared against a pretrained network provided to us by the authors in which the hyperparameter of their method was set to $c = 0.1$.
The network was a wide residual network with 16 layers and a width factor of 4.
We used the authors' code at  \url{https://github.com/ytsmiling/lmt} to compute the robustness radius of test images.

For randomized smoothing we used $\sigma = 0.1$ and a 20-layer residual network base classifier.
We ran \textsc{Certify} with $n_0 = 100$, $n=$ 100,000 and $\alpha = 0.001$.

For both methods, we certified the whole SVHN test set.

\paragraph{Comparison to \citet{zhang2018efficient}}

Following \citet{zhang2018efficient}, the CIFAR-10 dataset was preprocessed by subtracting 0.5 from each pixel.

We compared against the \texttt{cifar\_7\_1024\_vanilla} network released by the authors, which is a 7-layer MLP.
We used the authors' code at \url{https://github.com/IBM/CROWN-Robustness-Certification} to compute the robustness radius of test images.

For randomized smoothing we used $\sigma = 1.2$ and a 20-layer residual network base classifier.
We ran \textsc{Certify} with $n_0 = 100$, $n=$ 100,000 and $\alpha = 0.001$.

For randomized smoothing, we certified the whole CIFAR-10 test set.
For \citet{zhang2018efficient}, we certified every fourth image in the CIFAR-10 test set.

\subsection{ImageNet and CIFAR-10 Experiments}

Our code is available at \url{http://github.com/locuslab/smoothing}.

In order to report certified radii in the original coordinates, we \emph{first} added Gaussian noise, and \emph{then} standardized the data.
Specifically, in our PyTorch implementation, the first layer of the base classifier was a normalization layer that performed a channel-wise standardization of its input.
For CIFAR-10 we subtracted the dataset mean  $ (0.4914, 0.4822, 0.4465)$ and divided by the dataset standard deviation $(0.2023, 0.1994, 0.2010)$.
For ImageNet we subtracted the dataset mean  $(0.485, 0.456, 0.406)$ and divided by the standard deviation $(0.229, 0.224, 0.225)$.

For both ImageNet and CIFAR-10, we trained the base classifier with random horizontal flips and random crops (in addition to the Gaussian data augmentation discussed explicitly in the paper).  On ImageNet we trained with synchronous SGD on four NVIDIA RTX 2080 Ti GPUs; training took approximately three days.

On ImageNet our base classifier used the ResNet-50 architecture provided in \texttt{torchvision}.
On CIFAR-10 we used a 110-layer residual network from \url{https://github.com/bearpaw/pytorch-classification}.

On ImageNet we certified every 100-th image in the validation set, for 500 images total.
On CIFAR-10 we certified the whole test set.

In Figure \ref{fig:ablations} (\textbf{middle}) we fixed $\sigma = 0.25$ and $\alpha = 0.001$ while varying the number of samples $n$.
We did not actually vary the number of samples $n$ that we simulated: we kept this number fixed at 100,000 but varied the number that we fed the Clopper-Pearson confidence interval.

In Figure \ref{fig:ablations} (\textbf{right}), we fixed $\sigma = 0.25$ and $n =$100,000 while varying $\alpha$.

\subsection{Adversarial Attacks}
\label{section:experimentdetails:attacks}

As discussed in Section \ref{section:experiments}, we subjected smoothed classifiers to a projected gradient descent-style adversarial attack.
We now describe the details of this attack.

Let $f$ be the base classifier and let $\sigma$ be the noise level.
Following \citet{li2018second}, given an example $(x, c) \in \mathbb{R}^d \times \mathcal{Y}$ and a radius $r$, we used a projected gradient descent style adversarial attack to optimize the objective:
\begin{align}
\label{attackobjective}
\argmax_{\delta: \|\delta\|_2 < r} \; \mathbb{E}_{\varepsilon \sim \mathcal{N}(0, \sigma^2 I)} \left[ \ell(f(x+\delta + \varepsilon), c) \right]
\end{align}
where $\ell$ is the softmax loss function.  (Breaking notation with the rest of the paper in which $f$ returns a class, the function $f$ here refers to the function that maps an image in $\mathbb{R}^d$ to a vector of classwise scores.)

At each iteration of the attack, we drew $k$ samples of noise, $\varepsilon_1 \hdots \varepsilon_k \sim \mathcal{N}(0, \sigma^2 I)$, and followed the stochastic gradient 
$g_t = \sum_{i=1}^k \nabla_{\delta_t} \ell(f(x+\delta_t+\varepsilon_k), c)$.

As is typical \citep{kolter2018}, we used a ``steepest ascent'' update rule, which, for the $\ell_2$ norm, means that we normalized the gradient before applying the update.  The overall PGD update is:
$\delta_{t+1} = \text{proj}_r \left (\delta_{t} + \eta \frac{g_t}{\|g_t\|} \right)$
where the function $\text{proj}_r$ that projects its input onto the ball $\{z: \|z\|_2 \le r\}$ is given by $\text{proj}_r(z) = \frac{ r z}{\max(r, \|z\|_2)}$.
We used a constant step size $\eta$ and a fixed number $T$ of PGD iterations.

In practice, our step size was $\eta = 0.1$, we used $T=20$ steps of PGD, and we computed the stochastic gradient using $k=1000$ Monte Carlo samples.

Unfortunately, the objective we optimize (\ref{attackobjective}) is not actually the attack objective of interest.
To force a misclassification, an attacker needs to find some perturbation $\delta$ with $\|\delta\|_2 < r$ and some class $c_B$ for which
\begin{align*}
\mathbb{P}_{\varepsilon \sim \mathcal{N}(0, \sigma^2 I)}(f(x+\delta+\varepsilon) = c_B) \ge \mathbb{P}_{\varepsilon \sim \mathcal{N}(0, \sigma^2 I)}(f(x+\delta+\varepsilon) = c)
\end{align*}
Effective adversarial attacks against randomized smoothing are outside the scope of this paper.

\newpage

\section{Examples of Noisy Images}
\label{section:imageexamples}

We now show examples of CIFAR-10 and ImageNet images corrupted with varying levels of noise. 

\begin{figure}[h]
\captionsetup[subfigure]{labelformat=empty}
\begin{center}
\begin{subfigure}{0.24\textwidth}
	\begin{center}
	\includegraphics[width=32px]{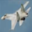}
	\end{center}
\end{subfigure}
\begin{subfigure}{0.24\textwidth}
	\begin{center}
	\includegraphics[width=32px]{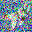}
	\end{center}
\end{subfigure}
\begin{subfigure}{0.24\textwidth}
	\begin{center}
	\includegraphics[width=32px]{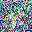}
	\end{center}
\end{subfigure}
\begin{subfigure}{0.24\textwidth}
	\begin{center}
	\includegraphics[width=32px]{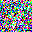}
	\end{center}
\end{subfigure}
\begin{subfigure}{0.24\textwidth}
	\begin{center}
	\includegraphics[width=32px]{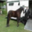}
	\end{center}
\end{subfigure}
\begin{subfigure}{0.24\textwidth}
	\begin{center}
	\includegraphics[width=32px]{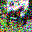}
	\end{center}
\end{subfigure}
\begin{subfigure}{0.24\textwidth}
	\begin{center}
	\includegraphics[width=32px]{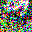}
	\end{center}
\end{subfigure}
\begin{subfigure}{0.24\textwidth}
	\begin{center}
	\includegraphics[width=32px]{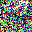}
	\end{center}
\end{subfigure}
\begin{subfigure}{0.24\textwidth}
	\begin{center}
	\includegraphics[width=32px]{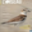}
	\end{center}
\end{subfigure}
\begin{subfigure}{0.24\textwidth}
	\begin{center}
	\includegraphics[width=32px]{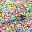}
	\end{center}
\end{subfigure}
\begin{subfigure}{0.24\textwidth}
	\begin{center}
	\includegraphics[width=32px]{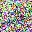}
	\end{center}
\end{subfigure}
\begin{subfigure}{0.24\textwidth}
	\begin{center}
	\includegraphics[width=32px]{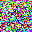}
	\end{center}
\end{subfigure}
\begin{subfigure}{0.24\textwidth}
	\begin{center}
	\includegraphics[width=32px]{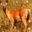}
	\caption{$\sigma = 0.00$}
	\end{center}
\end{subfigure}
\begin{subfigure}{0.24\textwidth}
	\begin{center}
	\includegraphics[width=32px]{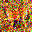}
	\caption{$\sigma = 0.25$}
	\end{center}
\end{subfigure}
\begin{subfigure}{0.24\textwidth}
	\begin{center}
	\includegraphics[width=32px]{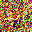}
	\caption{$\sigma = 0.50$}
	\end{center}
\end{subfigure}
\begin{subfigure}{0.24\textwidth}
	\begin{center}
	\includegraphics[width=32px]{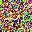}
	\caption{$\sigma = 1.00$}
	\end{center}
\end{subfigure}
\caption{CIFAR-10 images additively corrupted by varying levels of Gaussian noise $\mathcal{N}(0, \sigma^2 I)$.  Pixel values greater than 1.0 (=255) or less than 0.0 (=0) were clipped to 1.0 or 0.0.}
\end{center}
\end{figure}

\begin{figure}[h]
\captionsetup[subfigure]{labelformat=empty}
\begin{center}
\begin{subfigure}{0.24\textwidth}
	\includegraphics[width=\textwidth]{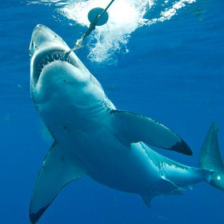}
\end{subfigure}
\begin{subfigure}{0.24\textwidth}
	\includegraphics[width=\textwidth]{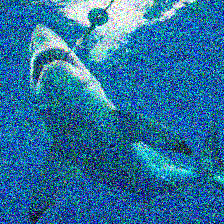}
\end{subfigure}
\begin{subfigure}{0.24\textwidth}
	\includegraphics[width=\textwidth]{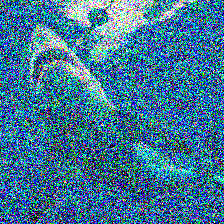}
\end{subfigure}
\begin{subfigure}{0.24\textwidth}
	\includegraphics[width=\textwidth]{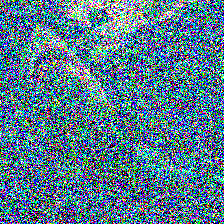}
\end{subfigure}
\begin{subfigure}{0.24\textwidth}
	\includegraphics[width=\textwidth]{figures/example_images/imagenet/19411_0}
\end{subfigure}
\begin{subfigure}{0.24\textwidth}
	\includegraphics[width=\textwidth]{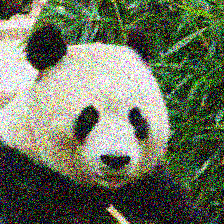}
\end{subfigure}
\begin{subfigure}{0.24\textwidth}
	\includegraphics[width=\textwidth]{figures/example_images/imagenet/19411_50}
\end{subfigure}
\begin{subfigure}{0.24\textwidth}
	\includegraphics[width=\textwidth]{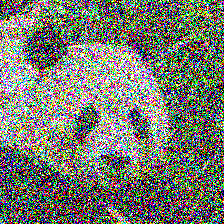}
\end{subfigure}
\begin{subfigure}{0.24\textwidth}
	\includegraphics[width=\textwidth]{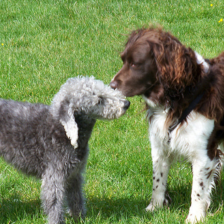}
\end{subfigure}
\begin{subfigure}{0.24\textwidth}
	\includegraphics[width=\textwidth]{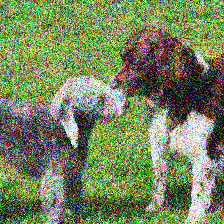}
\end{subfigure}
\begin{subfigure}{0.24\textwidth}
	\includegraphics[width=\textwidth]{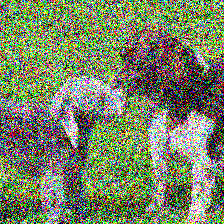}
\end{subfigure}
\begin{subfigure}{0.24\textwidth}
	\includegraphics[width=\textwidth]{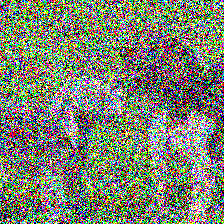}
\end{subfigure}
\begin{subfigure}{0.24\textwidth}
	\includegraphics[width=\textwidth]{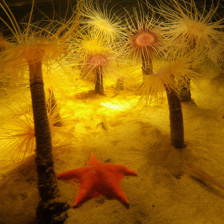}
	\caption{$\sigma = 0.00$}
\end{subfigure}
\begin{subfigure}{0.24\textwidth}
	\includegraphics[width=\textwidth]{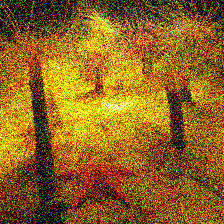}
	\caption{$\sigma = 0.25$}
\end{subfigure}
\begin{subfigure}{0.24\textwidth}
	\includegraphics[width=\textwidth]{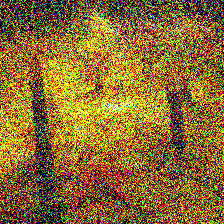}
	\caption{$\sigma = 0.50$}
\end{subfigure}
\begin{subfigure}{0.24\textwidth}
	\includegraphics[width=\textwidth]{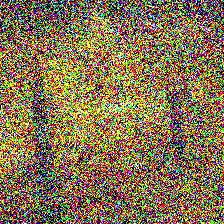}
	\caption{$\sigma = 1.00$}
\end{subfigure}
\caption{ImageNet images additively corrupted by varying levels of Gaussian noise $\mathcal{N}(0, \sigma^2 I)$.  Pixel values greater than 1.0 (=255) or less than 0.0 (=0) were clipped to 1.0 or 0.0.}
\end{center}
\end{figure}

\end{document}